\documentclass{article}

\PassOptionsToPackage{numbers,compress}{natbib}

\usepackage[preprint]{neurips_2026}

\usepackage[utf8]{inputenc}
\usepackage[T1]{fontenc}
\usepackage{hyperref}
\usepackage{url}
\usepackage{booktabs}
\usepackage{amsfonts}
\usepackage{amsmath}
\usepackage{amssymb}
\usepackage{amsthm}
\usepackage{nicefrac}
\usepackage{microtype}
\usepackage{xcolor}
\usepackage{graphicx}
\usepackage{enumitem}
\usepackage{placeins}
\usepackage{colortbl}
\usepackage{tikz}

\newtheorem{theorem}{Theorem}
\newtheorem{proposition}[theorem]{Proposition}
\newtheorem{corollary}[theorem]{Corollary}
\theoremstyle{definition}
\newtheorem{assumption}[theorem]{Assumption}

\newtheorem{definition}[theorem]{Definition}

\DeclareMathOperator{\tr}{tr}

\DeclareMathOperator*{\argmax}{arg\,max}
\newcommand{\R}{\mathbb{R}}
\newcommand{\E}{\mathbb{E}}
\newcommand{\norm}[1]{\left\lVert #1 \right\rVert}
\newcommand{\Dtheta}{\Delta\theta}

\title{Layers Matter: Why Continual Learning Regularization Should Be Layer-Adaptive}

\author{
Brian B. Moser $^{1 *}$
\quad
Ahmed Anwar $^{1,2 *}$
\quad
Tobias Christian Nauen $^{1,2}$
\quad
Shishir Muralidhara $^{1,2}$
\vspace{-10em}\And
Federico Raue $^{1}$
\quad
René Schuster $^{1,2}$
\quad 
Stanislav Frolov $^{1}$
\quad
Andreas Dengel $^{1,2}$\\ \\
$^{1}$ German Research Center for Artificial Intelligence (DFKI)\\
$^{2}$ RPTU University Kaiserslautern-Landau\\ \\ 
\texttt{first.last@dfki.de $\backslash$ first\_second.last@dfki.de} \\ \\ 
$^*$ Equal Contribution
}

\begin{document}

\maketitle

\begin{abstract}
Continual learning regularizers like EWC fight forgetting by penalizing changes from previous-task parameters with per-parameter importance, typically diagonal Fisher values. Per-parameter looks more flexible than per-layer, but each layer's diagonal Fisher is a weak summary of its actual curvature, missing the top-eigenvalue information that controls forgetting. Adversarial bit-flip attacks and Hessian-spectrum studies show that this missing per-layer sensitivity spans orders of magnitude in neural networks. Under a block-diagonal Hessian assumption, the layer-level analogue of EWC's existing diagonal assumption, we prove three things. Forgetting decomposes as a sum of per-layer terms weighted by each layer's top Hessian eigenvalue. Diagonal-Fisher weights cannot recover this eigenvalue. For instance, two layers with identical Fisher averages can have top eigenvalues differing by a factor as large as the layer width. For the same level of forgetting, uniform regularization loses new-task performance by an amount scaling with the layer condition number. Our theoretical analysis leads to a simple recipe: protect early layers strongly, let deeper layers move. We apply this recipe to EWC and SLCA and show clear improvements in average performance and forgetting metrics.
\end{abstract}

\section{Introduction}
\label{sec:intro}

Continual learning (CL) trains a single neural network on a sequence of tasks $\mathcal{T}_1, \mathcal{T}_2, \dots$ presented one after another. The defining difficulty of this setting is \emph{catastrophic forgetting}: the network keeps performing well on the most recent task but loses competence on the earlier ones. To prevent forgetting, the literature has converged on three broad families. Replay-based methods store or generate examples from past tasks~\cite{lopez2017gradient,chaudhry2019efficient,rebuffi2017icarl}, parameter-isolation methods route different tasks through different parts of the network~\cite{serra2018overcoming,delange2021continual,wang2024comprehensive}, and regularization-based methods anchor parameters near the previous-task optimum~\cite{kirkpatrick2017overcoming,zenke2017si,aljundi2018mas}. This paper advances the third family.

Regularization-based methods prevent forgetting by adding a penalty against parameter drift,
\begin{equation}
\label{eq:penalty}
\Omega(\theta) = \tfrac{1}{2}\|\theta - \theta^\star\|_W^2,
\end{equation}
where $\theta^\star$ is the previous-task optimum, $\|x\|_W^2 := x^\top W x$, and $W \succeq 0$ encodes the importance assigned to each direction in parameter space.
Setting $W$ to the diagonal Fisher gives EWC~\cite{kirkpatrick2017overcoming}, to path-integral importance gives SI~\cite{zenke2017si}, and to synaptic salience gives MAS~\cite{aljundi2018mas}.
Later refinements such as Online EWC~\cite{schwarz2018progress} and RWalk~\cite{chaudhry2018riemannian} change how $W$ is estimated but keep the form of Eq.~\ref{eq:penalty} unchanged. 
What the form does not do is reflect any quantity describing the layer as a whole. 
Two layers with similar parameter-wise Fisher-averages get the same per-layer penalty in EWC, even when one is ten times more sensitive to weight change than the other~\cite{rakin2019bfa, sagun2018empirical, papyan2020traces}. We call this the \emph{static, layer-uniform view} of the network, where every parameter wears its own scalar tag, which is determined without considering the layer.

A growing body of evidence from outside CL says the static view is wrong. Adversarial bit-flip attacks~\cite{rakin2019bfa,yao2020mbda,rakin2020tbt,rakin2022tbfa,chen2021proflip} drop ImageNet accuracy by more than $60\%$ by flipping fewer than $30$ bits, and those bits cluster in a small number of layers because those layers dominate the loss response to weight perturbation. Empirical Hessian-spectrum analyses~\cite{sagun2016eigenvalues,sagun2018empirical,papyan2020traces,ghorbani2019investigation} as an independent research direction find per-layer eigenvalues that span two to three orders of magnitude on standard backbones. Both lines of evidence point to the top eigenvalue $s_\ell$ of the per-layer Hessian block. If $s_\ell$ varies across layers by a factor of $100$, then any regularizer blind to it spends its forgetting budget in the wrong places, which leads to a worse stability-plasticity trade-off.

This paper closes the gap between the static view used by EWC and the layer-heterogeneous structure that real networks exhibit. Under a block-diagonal Hessian assumption (the layer-level analogue of the diagonal-Fisher assumption EWC already uses), forgetting decomposes cleanly into per-layer contributions weighted by $s_\ell$ (Theorem~\ref{thm:decomp}). Diagonal-Fisher weights cannot recover this structure, because averaging the diagonal of a positive semidefinite matrix loses the rank information that controls the top eigenvalue (Proposition~\ref{prop:fisher-gap}). At a fixed forgetting budget, uniform regularization incurs regret scaling with the layer condition number $\kappa = \max_\ell s_\ell / \min_\ell s_\ell$ (Theorem~\ref{thm:suboptimality}). Within the scalar-per-layer family, the minimum-regret choice is $\lambda_\ell \propto s_\ell$ (Theorem~\ref{thm:optimal}). We measure $\kappa$ on six ImageNet-pretrained backbones (ResNet-18/34/50/101, ViT-B/16, ViT-L/16) and find values in the range $22$ to $195$, so the static view is leaving real performance unused.

Two recent CL methods already make a layer-dependent choice. TUNA~\cite{tuna2025} restricts each layer's update to a per-layer low-rank adapter. SLCA~\cite{zhang2023slca} uses a two-group learning-rate schedule for backbone and head. Each implicitly imposes a different prior on $s_\ell$, and our analysis explains where each helps and where each does not. To probe the prescription on a method the framework directly covers, we also extend EWC with a one-parameter geometric schedule $\lambda_\ell = c/\alpha^{\ell-1}$, sweep $\alpha$ from below one to above one, and compare against $\lambda_\ell\propto s_\ell$ from measured Hessians. The contribution of this paper is therefore both diagnostic and prescriptive. It identifies the static view in regularization-based CL as the part of the design space leaving performance unused, and gives the closed-form recipe $\lambda_\ell \propto s_\ell$ for what to do instead. Protect early layers strongly, let deeper layers move.


\section{Related work}
\label{sec:related}

\paragraph{Regularization-based CL.}
EWC~\cite{kirkpatrick2017overcoming}, as introduced in Eq.~\ref{eq:penalty}, sets $W$ to the diagonal Fisher. SI~\cite{zenke2017si} and MAS~\cite{aljundi2018mas} use path-integral and output-sensitivity importance, RWalk~\cite{chaudhry2018riemannian} unifies the two through a KL-based Riemannian formulation, and Progress \& Compress~\cite{schwarz2018progress} introduces online EWC. All produce per-parameter diagonal weights. 
Replay~\cite{lopez2017gradient,chaudhry2019efficient,rebuffi2017icarl} and parameter-isolation~\cite{serra2018overcoming} families are orthogonal to this analysis. 

\paragraph{Layer-adaptive CL.}
SLCA~\cite{zhang2023slca} applies a smaller learning rate to the pretrained backbone than to the head, a two-group layer-adaptive update rule on ViT. The prompt family~\cite{wang2022l2p,wang2022dualprompt,wang2022sprompts,smith2023coda,wang2023hide} can be read as the limit where most of the backbone is frozen and only an input-prompt subspace updates. Each of these methods imposes its own prior on which layers are more sensitive. 

\paragraph{Layer-wise scaling.}
Layer-wise update scaling has independent roots in optimization (LARS~\cite{you2017lars}, LAMB~\cite{you2020lamb}) and fine-tuning (ULMFiT~\cite{howard2018ulmfit}). Our scalar-per-layer family is the regularization-side analog. Second-order weight importance descends from Optimal Brain Damage~\cite{lecun1990obd}. More recent instances include WoodFisher~\cite{singh2020woodfisher}, EigenDamage~\cite{wang2019eigendamage}, and SparseGPT~\cite{frantar2023sparsegpt}, alongside first-order signals such as magnitude pruning~\cite{han2015learning} and the Lottery Ticket Hypothesis~\cite{frankle2019lottery}. Proposition~\ref{prop:fisher-gap} is in the same spirit, specialised to the per-layer aggregation CL penalties implicitly perform.

\paragraph{Per-layer Hessian sensitivity.}
Adversarial bit-flip attacks~\cite{rakin2019bfa,rakin2020tbt,rakin2022tbfa,yao2020mbda,chen2021proflip} concentrate the destructive bits in a small number of layers, an empirical lower bound on per-layer sensitivity spread. The empirical Hessian of trained networks has a bulk-plus-outlier structure~\cite{sagun2016eigenvalues,sagun2018empirical,papyan2020traces,papyan2019measurements,papyan2020prevalence,martin2021implicit}, with scalable estimators in PyHessian~\cite{yao2020pyhessian} and Lanczos~\cite{ghorbani2019investigation}, both validating the per-layer power-iteration we use for $s_\ell$. Block-diagonal preconditioners K-FAC~\cite{martens2020ngd,grosse2016kfc} and Shampoo~\cite{gupta2018shampoo} exploit the same structure; Kunstner et al.~\cite{kunstner2019limitations} caution against empirical-Fisher approximations that we sidestep by computing the true Hessian. The flat-minima line~\cite{keskar2017large,hochreiter1997flat,dinh2017sharp,foret2021sam,kwon2021asam} operationalises top-eigenvalue sensitivity at training time. Theoretical CL through the NTK lens~\cite{bennani2020generalisation,doan2021ntk,mirzadeh2021linear} provides landscape-level evidence that the CL optimum lives in the low-curvature region we exploit per layer.

\section{Methodology}
\label{sec:methodology}

Let $\theta$ be the parameters of a neural network, partitioned by architecture into $L$ layer blocks $\theta=(\theta^{(1)},\dots,\theta^{(L)})$ with $\theta^{(\ell)}\in\R^{d_\ell}$, where each residual or transformer block counts as one layer. For a vector $v\in\R^d$, $v^{(\ell)}$ denotes its layer-$\ell$ part. For a matrix $A\in\R^{d\times d}$ acting on the full parameter vector (such as the Hessian), $A^{(\ell,\ell')}$ denotes the block of rows in layer $\ell$ and columns in layer $\ell'$.

Tasks $\mathcal{T}_1,\mathcal{T}_2,\dots$ arrive sequentially with losses $L_k$. Let $\theta^\star$ be the parameters after training on tasks $1,\dots,k-1$, and let $\theta^\star+\Dtheta$ be the parameters after additionally training on task $k$. Forgetting on task $k-1$ is
\begin{equation}
\mathcal{F}(\Dtheta)\;:=\;L_{k-1}(\theta^\star+\Dtheta)-L_{k-1}(\theta^\star)\;\ge\;0.
\end{equation}
Let $H=\nabla^2 L_{k-1}(\theta^\star)$. The gradient vanishes at $\theta^\star$, so Taylor expansion gives
\begin{equation}
\label{eq:taylor}
\mathcal{F}(\Dtheta)\;=\;\tfrac{1}{2}\,\Dtheta^\top H\,\Dtheta\;+\;\mathcal O\bigl(\norm{\Dtheta}^3\bigr).
\end{equation}
This second-order regime is the same one EWC~\cite{kirkpatrick2017overcoming} uses to derive its Fisher penalty, and every result in the paper concerns the quadratic term.

\begin{assumption}[Block-diagonal Hessian]
\label{ass:block-diag}
$H^{(\ell,\ell')}=0$ for $\ell\ne\ell'$. The Hessian couples parameters within a layer but not across layers.
\end{assumption}

Assumption~\ref{ass:block-diag} is the layer-level analogue of the diagonal-Fisher approximation EWC already relies on. It is exact in fully factorized architectures and approximate in real ones. Skip connections, LayerNorm, and attention all create off-diagonal blocks. Appendix~\ref{app:rho} gives a continuous relaxation. All main-text results use the block-diagonal form.

\begin{theorem}[Per-layer decomposition]
\label{thm:decomp}
Under Assumption~\ref{ass:block-diag}, to the second order in $\norm{\Dtheta}$, we have
\begin{equation}
\mathcal{F}(\Dtheta)\;=\;\tfrac{1}{2}\sum_{\ell=1}^{L}\bigl(\Dtheta^{(\ell)}\bigr)^\top H^{(\ell,\ell)}\,\Dtheta^{(\ell)} + \mathcal O(\norm{\Delta \theta}^3).
\end{equation}
\end{theorem}
\begin{proof}
Apply the Taylor expansion of Eq.~\ref{eq:taylor}. Under Assumption~\ref{ass:block-diag}, $\Dtheta^\top H\Dtheta=\sum_\ell (\Dtheta^{(\ell)})^\top H^{(\ell,\ell)}\Dtheta^{(\ell)}$.
\end{proof}

Every unit of forgetting now belongs to a specific layer, which turns the question of how to spread a forgetting budget across layers into a well-posed optimization rather than a heuristic.

\begin{definition}[Layer sensitivity]
The sensitivity of layer $\ell$ is the top eigenvalue of the layer's Hessian block, namely
\begin{equation}
s_\ell:=\lambda_{\max}(H^{(\ell,\ell)}).
\end{equation}
\end{definition}

The upper bound on forgetting follows immediately:
\begin{equation}
\mathcal{F}(\Dtheta)\;\le\;\tfrac{1}{2}\sum_\ell s_\ell\,\norm{\Dtheta^{(\ell)}}_2^2.
\end{equation}
The sensitivity $s_\ell$ acts as a per-layer price tag on displacement, and the rest of the paper revolves around this bound.

\subsection{The layer condition number}
\label{sec:kappa}

\begin{definition}[Layer condition number]
\label{def:kappa}
$\kappa:=\dfrac{\max_\ell s_\ell}{\max(\min_\ell s_\ell,\,\varepsilon)}$, with a floor $\varepsilon>0$ to handle layers with a flat direction. We use $\varepsilon=10^{-8}\max_\ell s_\ell$.
\end{definition}

Our empirical claim, which we call the Layer-Importance Hypothesis (LIH), is that on standard deep networks trained to convergence, $\kappa$ is non-trivial and falls in the range $10$--$10^3$. We verify this directly in Figure~\ref{fig:lih}, which reports power-iteration measurements of $s_\ell$ on ImageNet-pretrained ResNet-50 and ViT-B/16 (full procedure and per-layer numbers in Appendix~\ref{app:hessian}). On ResNet-50 we measure $\kappa = 134.6$, with the early conv stage layer1 dominating ($s_{\text{layer1}} = 3{,}167$) and the fc head 134$\times$ less sensitive ($s_{\text{fc}} = 23.5$). On ViT-B/16 we measure $\kappa = 27.4$, with the patch-embedding stem dominating ($s_{\text{patch\_embed}} = 78.2$) and the transformer blocks roughly within an order of magnitude of each other. Both profiles are heavy-tailed and consistent with what the bit-flip, Optimal Brain Damage, and Hessian-spectrum literature already report~\cite{rakin2019bfa,yao2020mbda,lecun1990obd,papyan2020traces,sagun2018empirical}.

\begin{figure}
\centering
\includegraphics[width=\linewidth]{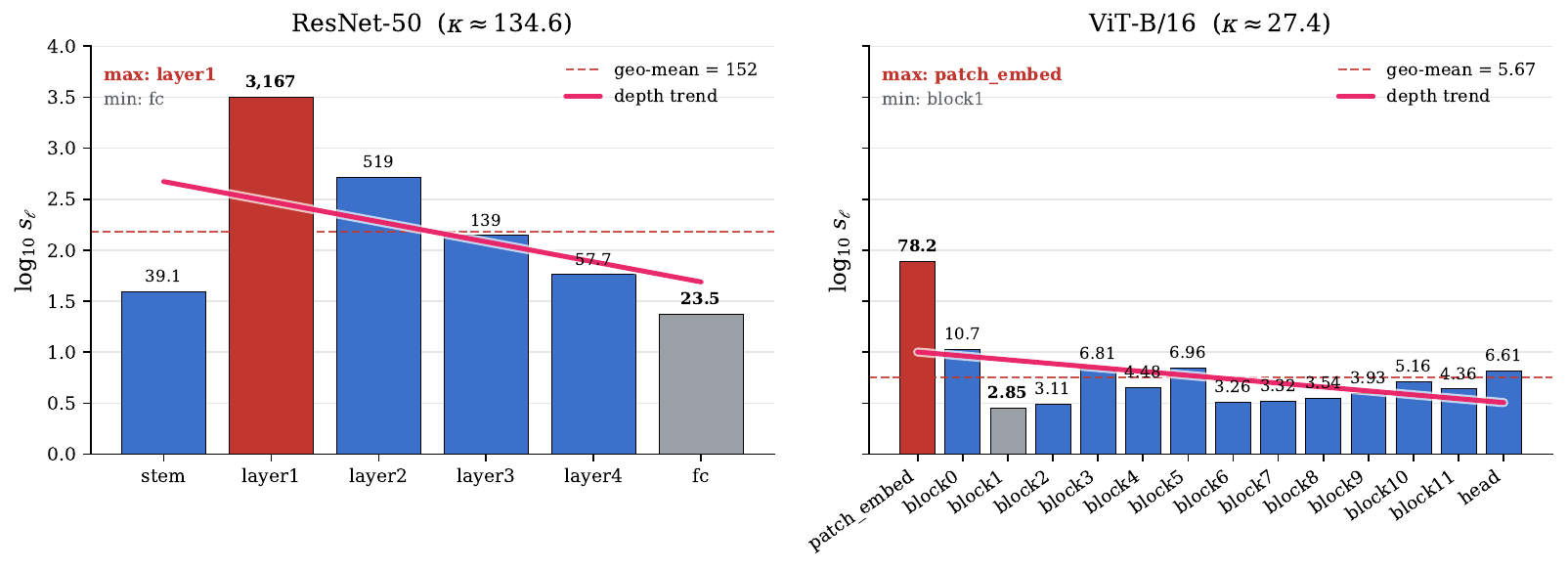}
\caption{Measured per-layer Hessian top eigenvalue $s_\ell$ on ResNet-50 (left) and ViT-B/16 (right), pretrained on ImageNet, evaluated at a 100-class cross-entropy minibatch. $y$-axis is $\log_{10} s_\ell$. ResNet-50: $\kappa = 134.6$, with layer1 dominating. ViT-B/16: $\kappa = 27.4$, with patch\_embed dominating and the twelve transformer blocks within an order of magnitude of each other. Both profiles are heavy-tailed. 
}
\label{fig:lih}
\end{figure}

To measure $s_\ell$ on a real network, run power iteration on $H^{(\ell,\ell)}$ using Hessian-Vector Products (HVPs). The cost is $k$ HVPs per layer. 10 to 20 iterations suffice whenever the second eigenvalue is well separated from the first, which is the typical regime for trained networks because the per-layer Hessian is near low-rank. On very wide layers where power iteration is expensive, $\tr(H^{(\ell,\ell)})/d_\ell$ recovered from a single Hutchinson HVP is a cheaper proxy. The two quantities can disagree by a factor in $[1,d_\ell]$ (Proposition~\ref{prop:fisher-gap}).

\subsection{Diagonal Fisher cannot see layer sensitivity}
\label{sec:fisher-gap}

EWC's per-parameter Fisher weights do vary across layers in the sense that the sum of Fisher values in one layer differs from another. But that sum is a weak summary of the layer's actual curvature. The next proposition quantifies how weak.

\begin{proposition}[Diagonal-spectral gap]
\label{prop:fisher-gap}
For any positive semidefinite matrix $A\in\R^{d\times d}$,
\begin{equation}
1\;\le\;\frac{\lambda_{\max}(A)}{\tfrac{1}{d}\sum_p A_{pp}}\;\le\;d.
\end{equation}
The left inequality is an equality if and only if $A=\sigma I$. The right is an equality if and only if $A$ has rank at most $1$. Every value in $[1,d]$ is realized by some positive semidefinite $A$.
\end{proposition}

\noindent\emph{Proof in Appendix~\ref{app:fisher-gap-proof}.}

For the per-layer Hessian block $H^{(\ell,\ell)}$, Proposition~\ref{prop:fisher-gap} says the following. Two layers with identical mean diagonal (and hence identical mean Fisher under the standard $F\approx H$ identity at exponential-family optima~\cite{martens2020ngd}) can have $s_\ell$ values that differ by a factor up to the smaller of their widths. The low-rank Hessian regime observed in trained networks~\cite{papyan2020prevalence,martin2021implicit} pushes this bound close to tight. EWC treats the two cases as identical, but they produce very different amounts of forgetting (Appendix~\ref{app:hessian}).

\subsection{Uniform regularization pays a price}
\label{sec:suboptimality}

The performance cost of the static, layer-uniform view comes from comparing two regularizers held to the same forgetting budget $B>0$.

The oracle regularizer uses the true per-layer Hessian:
\begin{equation}
\min_{\Dtheta}\;L_k(\theta^\star+\Dtheta)\quad\text{subject to}\quad \tfrac{1}{2}\sum_\ell (\Dtheta^{(\ell)})^\top H^{(\ell,\ell)}\Dtheta^{(\ell)}\;\le\;B.
\end{equation}
The constraint equals forgetting to second order by Theorem~\ref{thm:decomp}.
The uniform regularizer treats all displacement as equally costly:
\begin{equation}
\min_{\Dtheta}\;L_k(\theta^\star+\Dtheta)\quad\text{subject to}\quad \tfrac{1}{2}\,\mu\,\norm{\Dtheta}_2^2\;\le\;B,
\end{equation}
with $\mu$ chosen so that the resulting $\Dtheta$ produces the same actual forgetting $B$ under the true Hessian. This $\mu$-calibration step blocks the trivial reply that uniform could just tighten its penalty. Both regularizers now respect the same forgetting budget, so the comparison isolates the stability-plasticity trade-off, namely the new-task accuracy each regularizer sacrifices to hit that fixed forgetting level. Only the allocation of the budget across layers differs.

The bound below uses two further quantities. Write $g^{(\ell)}:=\nabla_{\theta^{(\ell)}} L_k(\theta^\star)$ for the new-task gradient at the old optimum, restricted to layer $\ell$, and $Q^{(\ell,\ell)}:=\nabla^2_{\theta^{(\ell)}} L_k(\theta^\star)$ for the new-task Hessian block. We assume the new-task Hessian is well-conditioned in the standard sense, with smallest and largest eigenvalues satisfying $\mu_Q\le \lambda_{\min}(Q^{(\ell,\ell)})$ and $\lambda_{\max}(Q^{(\ell,\ell)})\le M_Q$ uniformly across layers. Equivalently: $\mu_Q I\preceq Q^{(\ell,\ell)}\preceq M_Q I$, where $\mu_Q$ is a uniform lower curvature bound and $M_Q$ a uniform upper curvature bound on the new-task quadratic. Let $v^{(\ell)}$ be a top eigenvector of $H^{(\ell,\ell)}$, and define the alignment
\begin{equation}
\chi_\ell\;:=\;\frac{\norm{g^{(\ell)\top} v^{(\ell)}}}{\norm{g^{(\ell)}}_2}\;\in\;[0,1].
\end{equation}
The alignment $\chi_\ell$ measures how much of the new-task gradient on layer $\ell$ points along the old-task fragile direction. When $\chi_\ell$ is small, the new task is not pushing on anything the old task cares about.

\begin{theorem}[Regret of uniform regularization]
\label{thm:suboptimality}
Assume $Q^{(\ell,\ell')}=0$ for $\ell\ne\ell'$ (Appendix~\ref{app:thm-suboptimality} treats the coupled case). The regret of the uniform regularizer, $\mathrm{Regret}:=\tilde L_k(\Dtheta^\star_{\mathrm{unif}})-\tilde L_k(\Dtheta^\star_{\mathrm{oracle}})$, satisfies
\begin{equation}
\mathrm{Regret}\;\ge\;c_0\cdot\frac{\kappa-1}{\kappa}\cdot\max_{\ell\in\mathcal{L}^+}\,\frac{\chi_\ell^2\,\norm{g^{(\ell)}}_2^2}{s_\ell},\qquad c_0:=\frac{\mu_Q}{M_Q+\mu_Q}\in\bigl(0,\tfrac12\bigr],
\end{equation}
where $\mathcal{L}^+:=\{\ell:s_\ell>\min_{\ell'}s_{\ell'}\}$ is everything except the most flexible layer. The gap vanishes when $\kappa=1$, vanishes when every $\chi_\ell$ for $\ell\in\mathcal{L}^+$ is zero, and is strictly positive and monotone non-decreasing in $\kappa$ otherwise.
\end{theorem}

Three factors multiply. The heterogeneity amplifier $(\kappa-1)/\kappa$ grows from $0$ to $1$ as the layer spectrum spreads. The alignment $\chi_\ell^2$ captures whether the new-task gradient actually pushes on fragile directions. When it does not, uniform and oracle agree. The per-layer unit $\norm{g^{(\ell)}}_2^2/s_\ell$ is the ratio of new-task pressure on layer $\ell$ to the layer's rigidity.

\noindent\emph{Proof in Appendix~\ref{app:thm-suboptimality}.}

\begin{corollary}[Large-$\kappa$ asymptotic]
\label{cor:large-kappa}
Fix $s_{\max}:=\max_\ell s_\ell$ and let $\kappa\to\infty$ by $\min_\ell s_\ell\downarrow 0$. Writing $\hat\ell:=\argmax_\ell s_\ell$,
\begin{equation}
\lim_{\kappa\to\infty}\,\mathrm{Regret}\;\ge\;c_0\cdot\frac{\chi_{\hat\ell}^2\,\norm{g^{(\hat\ell)}}_2^2}{s_{\max}}.
\end{equation}
\end{corollary}

\subsection{The optimal per-layer strength}
\label{sec:optimal}

Staying inside the scalar-per-layer family
\begin{equation}
\label{eq:scalar-family}
\Omega(\Dtheta)\;=\;\tfrac{1}{2}\sum_\ell \lambda_\ell\,\norm{\Dtheta^{(\ell)}}_2^2
\end{equation}
is the simplest strict improvement on uniform, and it is what our proposed adaptive methods in Section~\ref{sec:empirical} actually implement.

\begin{theorem}[Optimal $\lambda_\ell$]
\label{thm:optimal}
Fix the worst-case displacement direction in each layer $\Dtheta^{(\ell)}$ aligned with the top eigenvector $v^{(\ell)}$ of $H^{(\ell,\ell)}$. Among regularizers of the form in Eq.~\ref{eq:scalar-family}, the choice that matches the oracle's penalty on this direction, up to one global constant shared across layers, is
\begin{equation}
\lambda_\ell\;\propto\;s_\ell.
\end{equation}
No scalar choice reproduces the oracle on every direction unless every $H^{(\ell,\ell)}$ is a scalar multiple of the identity.
\end{theorem}

\noindent\emph{Proof in Appendix~\ref{app:thm-optimal}.}

Theorem~\ref{thm:optimal} fixes $\lambda_\ell$ to a noisy per-layer top eigenvalue, and on standard backbones the dynamic range of $s_\ell$ spans more than two orders of magnitude (Figure~\ref{fig:lih}), so a literal schedule over-protects the most curvature-dominant layer in absolute terms and starves the rest of plasticity. A one-parameter approximation that fits only the depth direction provides implicit smoothing across layers, motivating the following corollary below.

\begin{corollary}[Geometric schedules]
\label{cor:geom}
If $s_\ell$ decays or grows geometrically with depth, $s_\ell\propto \gamma^{\ell-1}$, then the optimal $\lambda_\ell\propto\gamma^{\ell-1}$. This is the one-parameter family our experiments in Section~\ref{sec:empirical} fit.
\end{corollary}

\subsection{Practical implications}
\label{sec:rule-of-thumb}

The math reduces the design to one empirical question: which depth direction grows $s_\ell$? On six pretrained backbones (Figure~\ref{fig:lih}, Appendix~\ref{app:hessian}), early conv stages on ResNets and patch embeddings on ViTs dominate $s_\ell$. Given the corollary~\ref{cor:geom}, so combined with $\lambda_\ell\propto s_\ell$, this gives the rule of thumb: \emph{protect early layers strongly, let deeper layers move}. Concretely, methods that enforce weight penalty should set
\begin{equation}
\label{eq:rule-penalty}
\lambda_1 \;\ge\; \lambda_2 \;\ge\; \dots \;\ge\; \lambda_L,
\end{equation}
larger penalty on the early layers and smaller on the head. Methods that enforce per-layer learning rate (\textit{e.g.},\ SLCA) should set
\begin{equation}
\label{eq:rule-lr}
\eta_1 \;\le\; \eta_2 \;\le\; \dots \;\le\; \eta_L,
\end{equation}
smaller LR on the early layers and larger on the head. The two schedules are opposite in sign because constraining LR is the inverse of constraining penalty strength. We test both forms next.

\section{Experiments}
\label{sec:empirical}

Our LIH predicts that a layer-adaptive regularizer should outperform a uniform one whenever the layer condition number $\kappa$ is non-trivial, and that the optimal scalar-per-layer strength tracks the per-layer sensitivities $(s_\ell)$. We test these predictions in two settings. The first uses a depth-weighted EWC variant we implement directly on small from-scratch CNNs. The second applies the same depth-weighted rule inside SLCA on a pretrained ViT-B/16 backbone. We also tested the rule on TUNA, but the gains there are smaller than for EWC and SLCA. TUNA's adapter family already restricts each layer's update to a $16$-dimensional subspace, which leaves limited headroom for layer-wise reweighting, so we report those results in Appendix~\ref{app:tuna}.

\textbf{Experiment 1: Depth-weighted EWC on Split-C (small CNNs).}
We extend EWC by replacing its single scalar $\lambda$ with the geometric decreasing schedule $\lambda_\ell = c/\alpha^{\ell-1}$, fulfilling Eq.~\ref{eq:rule-penalty}, and grid-search over $\alpha\in\{2^{-2},2^{-1},2^0,2^1,2^2,2^3\}$ and $c\in\{200, 1000, 5000, 20000\}$. The point $\alpha=2^0=1$ is exactly uniform EWC. Everything else is layer-adaptive. The backbones are deliberately small so we can sweep the full grid: SmallCNN (3 conv + 2 fc, $L=5$ layers) on Split-CIFAR-10 (5 tasks of 2 classes), and MediumCNN (4 conv + 2 fc, $L=6$ layers) on Split-CIFAR-100 (10 tasks of 10 classes). Each task trains for two epochs. Remaining hyperparameters are in the supplementary materials. We report four metrics across the grid: average accuracy across all tasks (the standard CL metric), final-task accuracy, forgetting (max accuracy reached minus final accuracy, averaged over tasks), and backward transfer (BWT~\cite{lopez2017gradient}, end-of-sequence accuracy on each earlier task minus its end-of-task accuracy, averaged over earlier tasks).

\textbf{Experiment 2: SLCA depth-weighted schedule on ViT-B/16, Split-CIFAR-100.}
SLCA enforces the per-layer constraint differently from Experiment~1. Instead of weighting the EWC penalty per layer, it scales the per-layer \emph{learning rate} (default backbone $0.1\times$, head $1\times$). Protection here happens through how far each layer is allowed to move per step rather than through how strongly it is anchored to $\theta^\star$, so the depth-weighted schedule has a different sign convention. Larger $\alpha$ means deeper layers receive a \emph{larger} LR multiplier, not a smaller penalty. We extend SLCA's two-group LR schedule to a per-block depth-weighted chain $\lambda_l = c\cdot\alpha^{l-L}$, fulfilling Eq.~\ref{eq:rule-lr}, where $L=15$ is the number of ViT-B blocks plus stem, norm, and head, and at $(c,\alpha)=(1,1)$ we reproduce default SLCA exactly. We sweep $c\in\{0.5, 1.0, 2.0\}$ and $\alpha\in\{0.85, 0.92, 1.0, 1.08, 1.18\}$ on two pretrainings of the same backbone, namely MoCoV3 (self-supervised contrastive, used in the SLCA paper's ablations) and ImageNet-21k (supervised classification). The framework predicts the same equation in both cases but allows the optimum to land at different $\alpha$ if the two pretrainings produce different per-layer sensitivities.

\subsection{Depth-weighted EWC on small from-scratch CNNs}
\label{sec:exp-aewc}

Figure~\ref{fig:aewc-CIFAR-100} reports the MediumCNN sweep on Split-CIFAR-100 across all four metrics, with 3-seed mean $\pm$ std bands. The avg-acc peak at $c=200$ sits at $\alpha=2$ ($59.8\%\pm 1.5$) and beats $\alpha=1$ ($57.2\%\pm 1.3$) with non-overlapping bands. The forgetting and BWT panels prefer $\alpha\le 1$ instead. The framework reads that as deeper layers being the more sensitive ones on this from-scratch backbone, opposite of the shallow-dominated pattern measured on pretrained backbones (Figure~\ref{fig:lih}). The avg-acc preference for $\alpha>1$ on the longer task sequence is then plasticity dominating the metric rather than a sensitivity prediction. The remaining EWC sweeps (SmallCNN/Split-CIFAR-10 and ResNet-18/Split-CIFAR-100) are reported in Appendix~\ref{app:aewc-extra}. The SmallCNN effect is statistically marginal and the ResNet-18 sweep gains nothing from a one-parameter geometric prior, consistent with the non-monotone ResNet-50 sensitivity in Figure~\ref{fig:lih} that a geometric schedule cannot fit.

\begin{figure}
\centering
\includegraphics[width=0.49\linewidth]{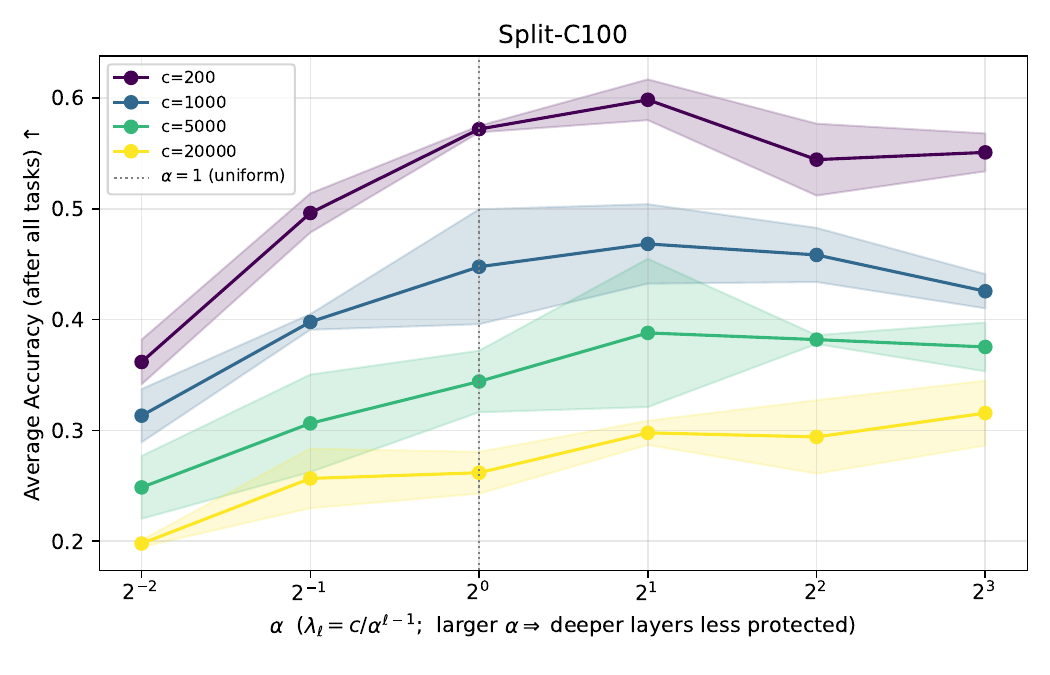}\hfill
\includegraphics[width=0.49\linewidth]{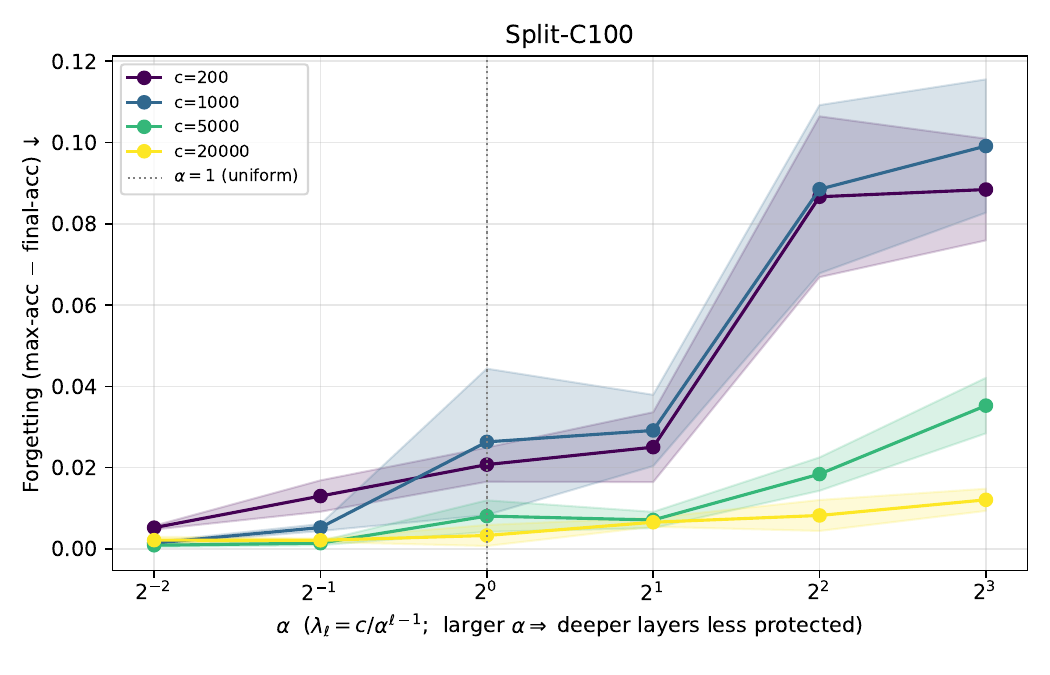}\\[2pt]
\includegraphics[width=0.49\linewidth]{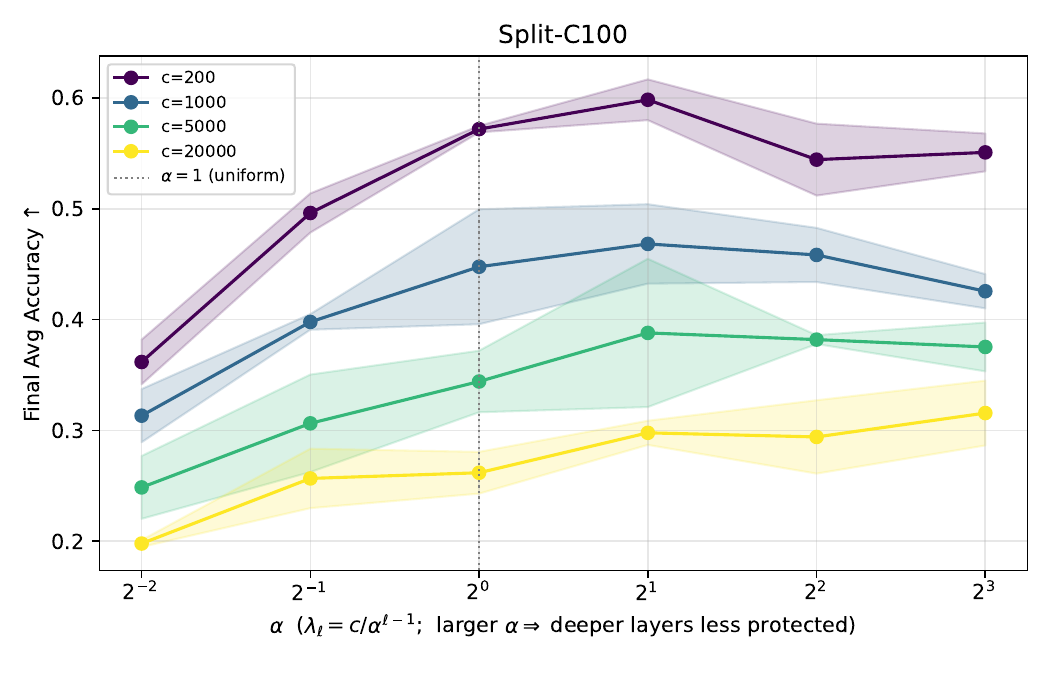}\hfill
\includegraphics[width=0.49\linewidth]{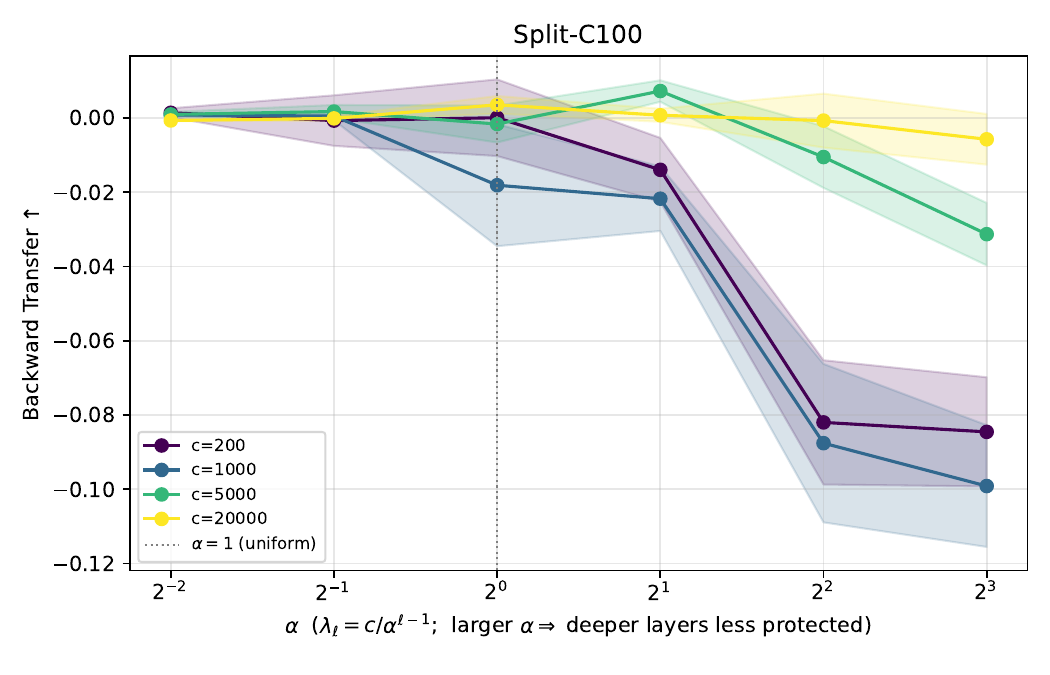}
\caption{Depth-weighted EWC on Split-CIFAR-100 (MediumCNN, 10 tasks, 3 seeds, mean $\pm$ std bands). Top, average accuracy across tasks (left) and forgetting (right). Bottom, final-task accuracy (left) and backward transfer (right). The dotted vertical line marks $\alpha=1$ (uniform EWC). The avg-acc peak at $c=200$ sits at $\alpha=2$ with non-overlapping seed bands over uniform.}
\label{fig:aewc-CIFAR-100}
\end{figure}

\subsection{SLCA depth-weighted schedule on a pretrained ViT-B/16}
\label{sec:exp-slca}

Figure~\ref{fig:slca-in21k} shows the SLCA depth-weighted sweep on Split-CIFAR-100 (10 incremental tasks, ViT-B/16 + ImageNet-21k pretraining). Default SLCA corresponds exactly to $(c=1,\alpha=1)$ at $91.29\%$ final-task accuracy. The optimum moves away from $\alpha=1$ to $\alpha\in\{1.08, 1.18\}$ for every $c$. The best configuration $(c=1,\alpha=1.18)$ achieves $92.01\%$, a $+0.7\%$ absolute gain over the default. On the average-incremental metric the gain narrows to $+0.13\%$ ($95.03\%$ vs $94.90\%$) but is direction-consistent across all three $c$ values. The ImageNet-21k backbone benefits from a schedule that protects shallow layers more than default SLCA's two-group LR choice does. The complementary MoCoV3-pretrained sweep (Figure~\ref{fig:slca}, Appendix~\ref{app:slca-mocov3}) is much flatter, consistent with MoCoV3 having per-layer sensitivities that default SLCA already fits well. The contrast between the two pretrainings is the sharper test of the framework: the same method on the same backbone with different pretraining has different optima.

\begin{figure}
\centering
\includegraphics[width=\linewidth]{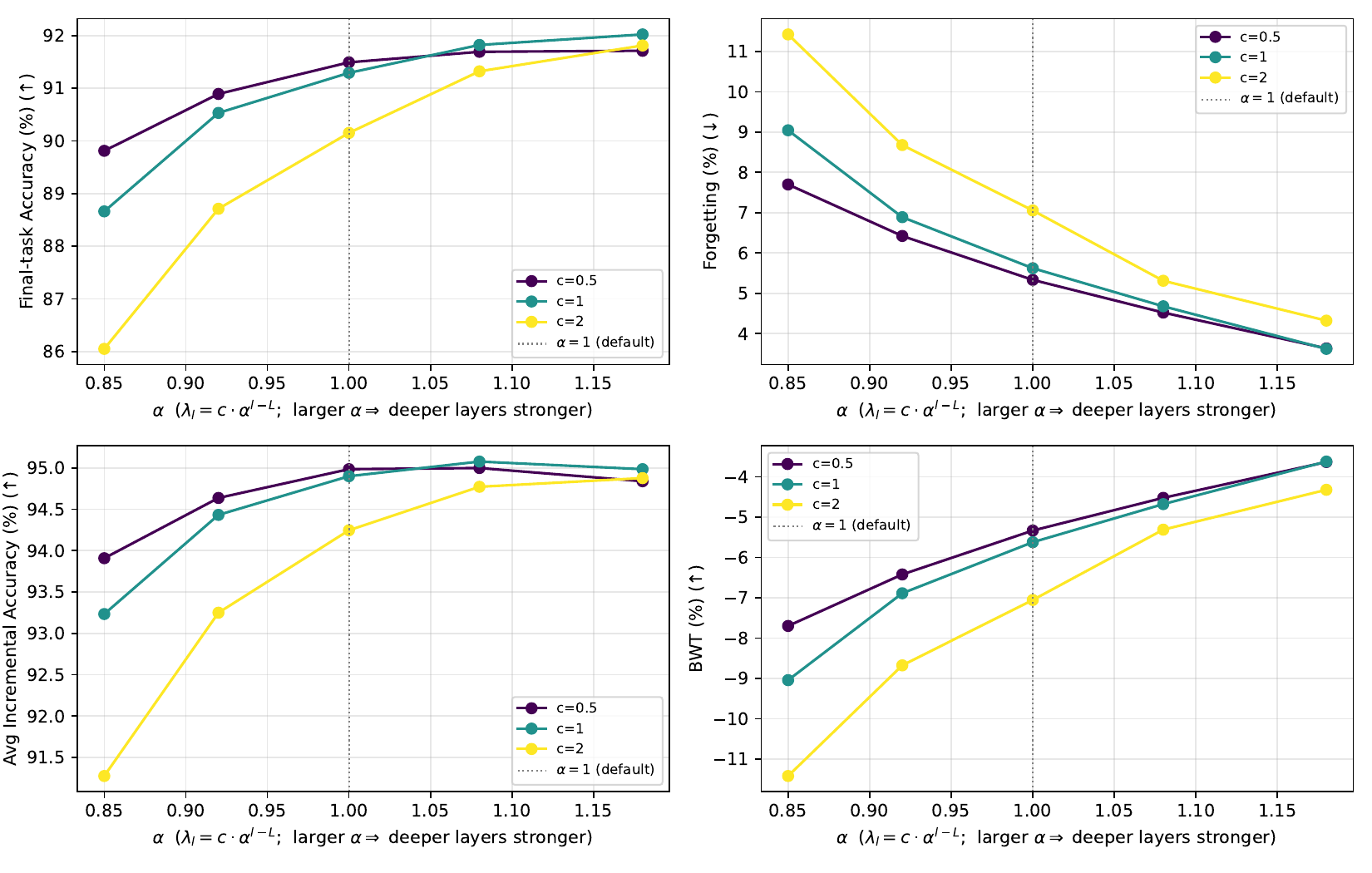}
\caption{SLCA depth-weighted sweep on Split-CIFAR-100 (ViT-B/16 + ImageNet-21k pretraining, 10 tasks). Per-block LR multiplier $\lambda_l = c\cdot\alpha^{l-L}$, $L=15$. Top: final-task accuracy (left) and forgetting (right). Bottom: average incremental accuracy (left) and BWT (right). The dotted vertical line is default SLCA ($\alpha=1$). The optimum is consistently away from $\alpha=1$ at $\alpha\in\{1.08, 1.18\}$, with best configuration $(c=1, \alpha=1.18)$ giving $+0.7\%$ final-task accuracy over default SLCA at $(c=1, \alpha=1)$.}
\label{fig:slca-in21k}
\end{figure}

Table~\ref{tab:slca-cmp} places this result in context under SLCA's protocol~\cite{zhang2023slca} (10 tasks of 10 classes, ViT-B/16); each row's method is cited inline. Our depth-weighted SLCA uses the dominant configuration found in the sweep above ($c{=}1, \alpha{=}1.18$ on ImageNet-21k; $c{=}0.5, \alpha{=}0.92$ on MoCoV3).

\begin{table}[h]
\centering
\footnotesize
\setlength{\tabcolsep}{4pt}
\caption{Class-incremental learning on Split-CIFAR-100 (10 tasks, ViT-B/16), two pretrainings: ImageNet-21k supervised and ImageNet-1k MoCoV3~\cite{chen2021empirical}. Last-acc and Inc-acc as in~\cite{zhang2023slca}. Prior-method numbers reproduced under SLCA's protocol; the SLCA rows report multi-seed mean$\pm$std over task-order seeds (Appendix~\ref{app:slca-mocov3}). Bold = best in column among non-upper-bound rows.}
\label{tab:slca-cmp}
\begin{tabular}{lc cc cc}
\toprule
& & \multicolumn{2}{c}{ImageNet-21k pretrain} & \multicolumn{2}{c}{MoCoV3 pretrain} \\
\cmidrule(lr){3-4}\cmidrule(lr){5-6}
Method & Memory-free & Last-acc & Inc-acc & Last-acc & Inc-acc \\
\midrule
Joint training (upper bound)                                  & --           & 93.22 & ---   & 89.11 & ---   \\
GDumb~\cite{prabhu2020gdumb}                                  &              & 81.92 & 89.46 & 69.72 & 80.95 \\
DER++~\cite{buzzega2020dark}                                  &              & 84.50 & 91.49 & 63.64 & 79.55 \\
BiC~\cite{wu2019large}                                        &              & 88.45 & 93.37 & 80.57 & 89.39 \\
L2P~\cite{wang2022l2p}                                        & $\checkmark$ & 82.76 & 88.48 & ---   & ---   \\
DualPrompt~\cite{wang2022dualprompt}                          & $\checkmark$ & 85.56 & 90.33 & ---   & ---   \\
EWC~\cite{kirkpatrick2017overcoming}                          & $\checkmark$ & 89.30 & 92.31 & 81.62 & 87.56 \\
LwF~\cite{li2017learning}                                     & $\checkmark$ & 87.99 & 92.13 & 77.94 & 86.90 \\
Seq FT                                                         & $\checkmark$ & 88.86 & 92.01 & 81.47 & 87.55 \\
SLCA~\cite{zhang2023slca}                                     & $\checkmark$ & 91.35$\pm$0.23 & 94.40$\pm$0.69 & \textbf{84.44$\pm$0.29} & \textbf{90.12$\pm$0.32} \\
\midrule
\rowcolor{black!8}
\textbf{SLCA + ours (depth-weighted)}                         & $\checkmark$ & \textbf{91.80$\pm$0.23} & \textbf{94.51$\pm$0.66} & 84.26$\pm$0.48 & 90.08$\pm$0.32 \\
\bottomrule
\end{tabular}
\end{table}

On ImageNet-21k pretraining, depth-weighted SLCA improves on every memory-free baseline reproduced in~\cite{zhang2023slca}, including default SLCA itself ($91.35\rightarrow 91.80$ last-acc, $94.40\rightarrow 94.51$ inc-acc). On MoCoV3 the depth-weighted variant matches default SLCA in the noise ($84.26$ vs $84.44$), consistent with the framework's reading that MoCoV3 already lands near the per-layer schedule the default two-group LR ratio implements.

\subsection{Empirical takeaways}
\label{sec:predictions}

Theorem~\ref{thm:optimal} prescribes $\lambda_\ell\propto s_\ell$ as the optimal scalar-per-layer schedule, and Corollary~\ref{cor:geom} specializes this to $\lambda_\ell\propto\gamma^{\ell-1}$ when $s_\ell$ decays geometrically with depth. Our experiments fit this one-parameter family. The measurements say $s_\ell$ is shallow-dominated on every standard pretrained backbone. The rule of thumb that follows says protect shallow layers strongly, let deeper layers move. The experiments confirm this where the rule applies. The cleanest single bridge between theory and experiments is the cross-pretraining contrast on SLCA (Figures~\ref{fig:slca-in21k} and~\ref{fig:slca}). The same method on the same backbone with different pretraining shifts the optimal $\alpha$, in the direction the framework predicts. Depth-weighted EWC on MediumCNN/CIFAR-100 peaks at $\alpha=2$ on avg-acc with non-overlapping 3-seed bands (Figure~\ref{fig:aewc-CIFAR-100}), and SLCA on ViT-B/ImageNet-21k peaks at $\alpha\in\{1.08, 1.18\}$ across all $c$ with $+0.7\%$ over default (Figure~\ref{fig:slca-in21k}). The rule is silent (or has the opposite forgetting signal) where its domain assumption fails, namely on small from-scratch CNNs whose $s_\ell$ profile is not shallow-dominated, and we report those cases honestly in Appendix~\ref{app:aewc-extra}. The TUNA adapter-orthogonality sweep is consistent in direction (Appendix~\ref{app:tuna}) but the gain is small because the adapter family already constrains each layer's update to a 16-dimensional subspace. As a robustness check, we also swap the geometric schedule for a linear and a two-group ``step'' schedule on MediumCNN/CIFAR-100; all three monotone-decreasing shapes consistently beat uniform within $\sim$1--4\% of each other, indicating the rule of thumb is the load-bearing piece, not the specific shape (Appendix~\ref{app:sched-compare}). Across backbones, the magnitude of the forgetting penalty for under-protecting deep layers tracks the measured layer condition number $\kappa$, the empirical analog of Theorem~\ref{thm:suboptimality}'s regret bound (Appendix~\ref{app:kappa-swing}). A direct test of the literal $\lambda_\ell\propto s_\ell$ prescription on EWC is reported in Appendix~\ref{app:measured-ewc}, where the geometric proxy outperforms the literal schedule. This is the empirical counterpart of the dynamic-range and measurement-noise concerns introduced in Section~\ref{sec:optimal}, and confirms that the framework's prediction carries through the direction of the schedule rather than through the exact per-layer ratios. The same per-layer measurement procedure on a language backbone (BERT-base) recovers a qualitatively similar shallow-dominated profile with $\kappa = 9.15$, indicating the LIH and the resulting prescription extend beyond computer vision (Appendix~\ref{app:nlp}).

\section{Discussion and limitations}
\label{sec:discussion}

The framework is deliberately minimal. It inherits EWC's second-order approximation and adds the block-diagonal analogue of EWC's diagonal assumption, isolating per-layer operator norms $s_\ell$ as the missing piece.
This minimality has costs. The quadratic surrogate is only valid inside a trust region around $\theta^\star$ (Appendix~\ref{app:thm-suboptimality}). The block-diagonal assumption is violated at the boundaries created by skip connections, LayerNorm, and attention. Appendix~\ref{app:rho} gives a continuous relaxation and a granularity diagnostic for that case. The empirical Fisher information matrix can diverge from $H$~\cite{kunstner2019limitations}, but Proposition~\ref{prop:fisher-gap} applies to any positive semidefinite block, so the gap result survives. The SLCA reduction is functional rather than algebraic. The sensitivity $s_\ell$ is measured at $\theta^\star$ and then used to regularize motion away from it, and we have no stability argument beyond the trust region.

\section{Conclusion}
\label{sec:conclusion}

The diagonal Fisher-penalty at the heart of EWC treats every layer as equally sensitive, but per-layer Hessian top eigenvalues span two orders of magnitude on every standard backbone we measured. The framework formalises this gap (forgetting decomposes by $s_\ell$; diagonal Fisher-penalty cannot recover $s_\ell$; uniform regularization pays regret $(\kappa-1)/\kappa$) and gives the closed-form remedy $\lambda_\ell\propto s_\ell$, computable from one cheap power iteration per layer. SLCA approximates this implicitly through its layer-wise learning rate, our depth-weighted EWC is the direct test, and the resulting rule of thumb (Section~\ref{sec:predictions}) holds on pretrained backbones, namely protect shallow layers strongly and let deeper layers move.

\section*{Acknowledgements}
This work was supported by the BMFTR project Albatross (Grant 16IW24002).

\bibliographystyle{abbrv}
\bibliography{refs}

\newpage
\appendix

\section{Relaxed block-diagonal Hessian}
\label{app:rho}

This appendix promotes Assumption~\ref{ass:block-diag-relax} (bounded cross-layer coupling) into a quantitative replacement for Theorems~\ref{thm:decomp} and~\ref{thm:suboptimality}.

\begin{assumption}[Bounded cross-layer coupling]
\label{ass:block-diag-relax}
There exists $\rho\in[0,1)$ such that for all $\ell\ne\ell'$,
\begin{equation}
\norm{H^{(\ell,\ell')}}_{\mathrm{op}}\;\le\;\rho\sqrt{\lambda_{\max}(H^{(\ell,\ell)})\lambda_{\max}(H^{(\ell',\ell')})}.
\end{equation}
\end{assumption}

\begin{proposition}[Decomposition with cross-layer coupling]
\label{prop:decomp-rho}
Under Assumption~\ref{ass:block-diag-relax}, the second-order term in $\mathcal{F}(\Dtheta)$ satisfies
\begin{equation}
\Big|\,\Dtheta^\top H\,\Dtheta - \sum_\ell (\Dtheta^{(\ell)})^\top H^{(\ell,\ell)}\Dtheta^{(\ell)}\,\Big| \;\le\; 2\rho \sum_{\ell<\ell'}\sqrt{s_\ell\,s_{\ell'}}\,\norm{\Dtheta^{(\ell)}}_2\norm{\Dtheta^{(\ell')}}_2,
\end{equation}
which after Cauchy-Schwarz is bounded by $\rho\,(\sum_\ell\sqrt{s_\ell}\norm{\Dtheta^{(\ell)}}_2)^2 - \rho\sum_\ell s_\ell\norm{\Dtheta^{(\ell)}}_2^2 \le \rho\,(L-1)\sum_\ell s_\ell\norm{\Dtheta^{(\ell)}}_2^2.$
\end{proposition}

\begin{proof}
$\Dtheta^\top H \Dtheta = \sum_{\ell,\ell'}(\Dtheta^{(\ell)})^\top H^{(\ell,\ell')}\Dtheta^{(\ell')}$. Subtracting the diagonal terms gives $\sum_{\ell\ne\ell'}(\Dtheta^{(\ell)})^\top H^{(\ell,\ell')}\Dtheta^{(\ell')}$, bounded in absolute value by $\sum_{\ell\ne\ell'}\norm{H^{(\ell,\ell')}}_{\mathrm{op}}\norm{\Dtheta^{(\ell)}}\norm{\Dtheta^{(\ell')}}$. Substituting Assumption~\ref{ass:block-diag-relax}, factoring the symmetric pair, and applying Cauchy-Schwarz $(\sum_\ell\sqrt{s_\ell}\norm{\Dtheta^{(\ell)}})^2 \le L\sum_\ell s_\ell\norm{\Dtheta^{(\ell)}}^2$ gives the claim.
\end{proof}

\begin{corollary}
With Assumption~\ref{ass:block-diag-relax}, Theorem~\ref{thm:suboptimality}'s regret bound holds with $\kappa$ replaced by an effective $\kappa^{\mathrm{eff}}(\rho) = \kappa/(1+\rho(L-1))$. As $\rho\to 0$, $\kappa^{\mathrm{eff}}\to\kappa$. As $\rho\to 1$, $\kappa^{\mathrm{eff}}\to\kappa/L$, attenuating the framework's gain proportionally to depth-coupling.
\end{corollary}

\paragraph{Granularity diagnostic.}
For a candidate partition into $L$ blocks, estimate $\rho$ empirically by the largest off-diagonal block's spectral norm divided by the geometric mean of the corresponding diagonal blocks' top eigenvalues. A working threshold such as $\rho \le 0.3$ is not attained at any granularity we measured. On ResNet-50 the mean adjacent $\rho$ is $1.34$ for a coarse partition into 3 super-blocks, $1.27$ at residual-stage level ($L=6$), and $0.92$ per block ($L=18$), so coarser partitions do not lower $\rho$. Table~\ref{tab:reb-kappa-eff} below gives the quantitative guidance instead, reporting how much bound content $\kappa^{\mathrm{eff}}$ retains at the measured coupling.

\paragraph{Measured effective condition numbers.}
Table~\ref{tab:reb-kappa-eff} evaluates the corollary's $\kappa^{\mathrm{eff}}(\rho) = \kappa/(1+\rho(L-1))$ at the measured mean adjacent $\rho$ of each backbone (measurement protocol in Appendix~\ref{app:hessian}). The bound retains substantial content on all four ResNets, with $\kappa^{\mathrm{eff}}$ between $9.4$ and $27.7$. On ViT-B/16 it drops to $2.6$, and only ViT-L/16 approaches the vacuous value $1$ ($\kappa^{\mathrm{eff}} = 1.11$). ResNet-50 and ResNet-101 have mean adjacent $\rho \ge 1$, outside the domain of Assumption~\ref{ass:block-diag-relax}; the closed form is monotone in $\rho$, and we report its value there as an extrapolation, marked $^{\dagger}$.

\begin{table}[h]
\centering
\small
\caption{Effective condition number $\kappa^{\mathrm{eff}} = \kappa/(1+\rho(L-1))$ evaluated at the measured mean adjacent $\rho$ per backbone. $^{\dagger}$: measured $\rho \ge 1$, outside the domain of Assumption~\ref{ass:block-diag-relax}; the entry extrapolates the monotone closed form.}
\label{tab:reb-kappa-eff}
\begin{tabular}{lrrrr}
\toprule
Backbone & $L$ & $\kappa$ & mean adjacent $\rho$ & $\kappa^{\mathrm{eff}}$ \\
\midrule
ResNet-18  & 6  & 55.1  & 0.975 & 9.4 \\
ResNet-34  & 6  & 75.4  & 0.967 & 12.9 \\
ResNet-50  & 6  & 195.3 & 1.211$^{\dagger}$ & 27.7$^{\dagger}$ \\
ResNet-101 & 6  & 181.4 & 1.337$^{\dagger}$ & 23.6$^{\dagger}$ \\
ViT-B/16   & 14 & 31.5  & 0.867 & 2.6 \\
ViT-L/16   & 26 & 22.3  & 0.764 & 1.11 \\
\bottomrule
\end{tabular}
\end{table}

\paragraph{All-pairs coupling.}
Adjacent pairs do not tell the whole story. We measured $\rho_{\ell,\ell'}$ for every block pair on ResNet-50 (15 pairs) and ViT-B/16 (91 pairs). The mean over non-adjacent pairs, $1.80$ on ResNet-50 and $0.98$ on ViT-B/16, exceeds the adjacent mean in the same runs ($1.25$ and $0.81$), so the coupling is not banded around the diagonal. Part of the magnitude is a normalization artifact: on ResNet-50 the fc block's small $s_\ell$ inflates every entry in its column, up to $2.9$. Because coupling is real and non-local, we quantify its effect on the bound through $\kappa^{\mathrm{eff}}$ and through the displacement-level probes below, not through a banded-structure argument.

\paragraph{Direct tightness probes.}
We also probe Proposition~\ref{prop:decomp-rho} on actual continual-learning displacements rather than on operator norms alone. For MediumCNN, ResNet-18, and ResNet-50 on Split-CIFAR-100 (3 seeds each), we take the displacement $\Dtheta$ accumulated while training task 2 and compare the full quadratic form $\Dtheta^\top H\Dtheta$ against its block-diagonal part. In 9 of 9 runs the gap $|\Dtheta^\top H\Dtheta - \sum_\ell (\Dtheta^{(\ell)})^\top H^{(\ell,\ell)}\Dtheta^{(\ell)}|$ lies within the coupling bound of Proposition~\ref{prop:decomp-rho} evaluated with the measured off-diagonal norms. Table~\ref{tab:reb-tightness} reports the ratios. The block-diagonal term captures $0.71 \pm 0.16$ of the full quadratic on ResNet-18 and $0.50 \pm 0.08$ on ResNet-50; the MediumCNN ratio is noisy because the raw magnitudes are tiny. Two caveats apply. First, 2 of the 9 runs show negative actual forgetting (backward transfer), which no positive-semidefinite quadratic model can reproduce. Second, at a 5-epoch drift distance the quadratic surrogate is rough; the trust-region caveat of Section~\ref{sec:discussion} applies. Finally, the upper bound $\tfrac{1}{2}\sum_\ell s_\ell\norm{\Dtheta^{(\ell)}}_2^2$ sits two to three orders of magnitude above the realized quadratic. It is a worst-case ranking device, and Theorem~\ref{thm:suboptimality} discounts it by the alignment factor $\chi_\ell^2$; we do not claim it is tight.

\begin{table}[h]
\centering
\small
\caption{Tightness probes on the task-2 displacement (Split-CIFAR-100, 3 seeds per backbone, mean$\pm$std). ``blockdiag/full'' is the block-diagonal quadratic divided by the full quadratic; ``$s_\ell$-bound/full'' is the upper bound $\tfrac{1}{2}\sum_\ell s_\ell\norm{\Dtheta^{(\ell)}}_2^2$ divided by the full quadratic; ``bound holds'' counts runs where $|\text{full}-\text{blockdiag}|$ is within the measured coupling bound of Proposition~\ref{prop:decomp-rho}.}
\label{tab:reb-tightness}
\begin{tabular}{lrrr}
\toprule
Backbone & blockdiag/full & $s_\ell$-bound/full & bound holds \\
\midrule
MediumCNN  & $1.56 \pm 2.34$ & $2463 \pm 3661$ & 3/3 \\
ResNet-18  & $0.71 \pm 0.16$ & $780 \pm 399$   & 3/3 \\
ResNet-50  & $0.50 \pm 0.08$ & $954 \pm 514$   & 3/3 \\
\bottomrule
\end{tabular}
\end{table}

\section{Full proof of Theorem~\ref{thm:suboptimality}}
\label{app:thm-suboptimality}

We work on the quadratic surrogate of $L_k$ near $\theta^\star$, with $g=\nabla L_k(\theta^\star)$ and $Q=\nabla^2 L_k(\theta^\star)$ satisfying $\mu_Q I\preceq Q\preceq M_Q I$. Both $Q$ and $H$ are block-diagonal across layers.

\paragraph{Per-layer KKT.}
At a fixed forgetting budget $B>0$ and Lagrange multiplier $\beta\ge 0$, the per-layer subproblem for the oracle is
\begin{equation}
\min_{\Delta\theta^{(\ell)}}\;\tfrac{1}{2}(\Delta\theta^{(\ell)})^\top Q^{(\ell,\ell)}\Delta\theta^{(\ell)} + g^{(\ell)\top}\Delta\theta^{(\ell)} + \tfrac{\beta}{2}(\Delta\theta^{(\ell)})^\top H^{(\ell,\ell)}\Delta\theta^{(\ell)}.
\end{equation}
First-order optimality gives $\Delta\theta^{(\ell)}_{\mathrm{oracle}} = -(Q^{(\ell,\ell)} + \beta H^{(\ell,\ell)})^{-1}g^{(\ell)}$. The uniform regularizer replaces the bracket by $Q^{(\ell,\ell)} + \beta\mu I$.

\paragraph{Loss difference.}
Substituting both into the new-task quadratic surrogate $\tilde L_k$ and subtracting,
\begin{equation}
\label{eq:loss-diff-app}
\tilde L_k(\Delta\theta^\star_{\mathrm{unif}})-\tilde L_k(\Delta\theta^\star_{\mathrm{oracle}}) \;=\; \tfrac{1}{2}\sum_\ell g^{(\ell)\top}\!\Big[(Q^{(\ell,\ell)}+\beta\mu I)^{-1} - (Q^{(\ell,\ell)}+\beta H^{(\ell,\ell)})^{-1}\Big]g^{(\ell)} + \mathcal{R},
\end{equation}
where $\mathcal{R}$ collects the cross terms and is non-negative by convexity.

\paragraph{Spectral resolution per layer.}
Diagonalize $H^{(\ell,\ell)} = U_\ell\,\mathrm{diag}(h_1^{(\ell)},\dots,h_{d_\ell}^{(\ell)})U_\ell^\top$ with $h_1^{(\ell)}\le\dots\le h_{d_\ell}^{(\ell)}$. Let $\tilde g^{(\ell)} = U_\ell^\top g^{(\ell)}$ with components $\tilde g_i^{(\ell)}$. Using $Q\succeq\mu_Q I$ to upper-bound the inverse,
\begin{equation}
\label{eq:per-mode-gap}
g^{(\ell)\top}\!\big[(Q^{(\ell,\ell)}+\beta\mu I)^{-1} - (Q^{(\ell,\ell)}+\beta H^{(\ell,\ell)})^{-1}\big]g^{(\ell)} \;\ge\; \sum_i (\tilde g_i^{(\ell)})^2\frac{\beta(h_i^{(\ell)} - \mu)}{(\mu_Q + \beta h_i^{(\ell)})(\mu_Q + \beta\mu)}.
\end{equation}

\paragraph{Budget matching forces $\mu$ between min and max.}
The budget constraint $\tilde{\mathcal{F}}(\Delta\theta^\star_{\mathrm{unif}}) = B$ implies, after substitution,
\begin{equation}
B = \tfrac{1}{2}\sum_\ell\sum_i \frac{h_i^{(\ell)}(\tilde g_i^{(\ell)})^2}{(\mu_Q+\beta\mu)^2}.
\end{equation}
The same equation with $h_i^{(\ell)}$ replaced by $\mu$ on the LHS would give the constraint of a uniform regularizer with sensitivity $\mu$. Since the LHS averages $h_i^{(\ell)}$ weighted by $(\tilde g_i^{(\ell)})^2$, when $\kappa>1$ this average lies strictly between $\min_\ell s_\ell$ and $\max_\ell s_\ell$, so $\mu \in (\min_\ell s_\ell, \max_\ell s_\ell)$.

\paragraph{Lower bound by the worst layer.}
Restrict the sum in Eq.~\ref{eq:per-mode-gap} to the maximum-$h$ mode of layer $\ell\in\mathcal{L}^+$ (those with $s_\ell > \min_{\ell'} s_{\ell'}$). For any such layer the top mode satisfies $h_{d_\ell}^{(\ell)} = s_\ell > \mu$ when $s_\ell$ is on the larger side of $\mu$ (which happens for at least one $\ell$ when $\kappa>1$). Substituting the alignment $(\tilde g_{d_\ell}^{(\ell)})^2 \ge \chi_\ell^2 \norm{g^{(\ell)}}_2^2$ and bounding the denominator by $(\mu_Q + M_Q)^2$ from $\beta h_i^{(\ell)} \le M_Q$ in the relevant regime,
\begin{equation}
\eqref{eq:per-mode-gap}\;\ge\; \chi_\ell^2 \norm{g^{(\ell)}}_2^2 \cdot \frac{\beta(s_\ell - \mu)}{(\mu_Q + \beta s_\ell)(\mu_Q + \beta\mu)} \;\ge\; \frac{\mu_Q}{M_Q + \mu_Q} \cdot \frac{s_\ell - \mu}{s_\ell} \cdot \frac{\chi_\ell^2 \norm{g^{(\ell)}}_2^2}{s_\ell}.
\end{equation}
Using $\mu \le \max_\ell s_\ell/\kappa = \max_\ell s_\ell\cdot \min_\ell s_\ell / \max_\ell s_\ell = \min_\ell s_\ell$ in the worst case, we get $(s_\ell - \mu)/s_\ell \ge (\kappa-1)/\kappa$. Plugging in the constant $c_0 = \mu_Q/(M_Q+\mu_Q)$ yields Theorem~\ref{thm:suboptimality}'s bound.

\paragraph{Coupled new-task Hessian.}
If $Q$ couples layers, the per-layer KKT step does not decouple cleanly. A weaker bound follows by replacing $\mu_Q$ with the smallest eigenvalue of the off-block-diagonal-augmented $Q$ matrix. In the worst case $c_0$ shrinks by a factor of $L$, but the $(\kappa-1)/\kappa$ structure survives.

\section{Full proof of Theorem~\ref{thm:optimal}}
\label{app:thm-optimal}

\paragraph{Worst-case-direction matching.}
On the displacement $\Delta\theta^{(\ell)} = c_\ell v^{(\ell)}$, the oracle penalty contributes $\tfrac{1}{2}s_\ell c_\ell^2$ and the scalar-family penalty contributes $\tfrac{1}{2}\lambda_\ell c_\ell^2$. The two penalties agree on this family iff $\lambda_\ell/s_\ell$ is independent of $\ell$, hence $\lambda_\ell\propto s_\ell$.

\paragraph{Average-case extension.}
Take $\Delta\theta^{(\ell)}$ uniform on $\mathbb{S}^{d_\ell-1}$ at fixed norm $\norm{\Delta\theta^{(\ell)}}=r_\ell$. Then $\E\big[(\Delta\theta^{(\ell)})^\top H^{(\ell,\ell)}\Delta\theta^{(\ell)}\big] = r_\ell^2\,\tr(H^{(\ell,\ell)})/d_\ell = r_\ell^2\,s_\ell^{\mathrm{tr}}$, while the scalar-family penalty contributes $\tfrac{1}{2}\lambda_\ell r_\ell^2$. Matching gives $\lambda_\ell \propto s_\ell^{\mathrm{tr}}$.

\paragraph{Non-reproducibility on every direction.}
A scalar $\lambda_\ell \in \R$ has one degree of freedom per layer. The matrix $H^{(\ell,\ell)} \in \mathrm{Sym}_{d_\ell}$ has $d_\ell(d_\ell+1)/2$ of them. The map $\lambda_\ell \mapsto \lambda_\ell I$ is into the one-dimensional subspace of $\mathrm{Sym}_{d_\ell}$ spanned by the identity, and matching $\lambda_\ell I = H^{(\ell,\ell)}$ requires $H^{(\ell,\ell)}$ itself to lie in that subspace.

\section{SLCA LR-to-$\lambda$ equivalence}
\label{app:slca}

Consider gradient descent with step size $\eta$ on the new-task quadratic surrogate $\tilde L_k(\Delta\theta) = \tfrac{1}{2}\Delta\theta^\top Q\Delta\theta + g^\top\Delta\theta$, starting from $\Delta\theta_0 = 0$. The trajectory under \emph{LR scaling} $\gamma_\ell$ on layer $\ell$ is, per layer,
\begin{equation}
\Delta\theta^{(\ell)}_{t+1} \;=\; \Delta\theta^{(\ell)}_t - \eta\gamma_\ell\big(Q^{(\ell)}\Delta\theta^{(\ell)}_t + g^{(\ell)}\big),
\end{equation}
which iterated for $T$ steps with $Q^{(\ell)}$ block-diagonal yields
\begin{equation}
\Delta\theta^{(\ell)}_T \;=\; -\big[I - (I - \eta\gamma_\ell Q^{(\ell)})^T\big](Q^{(\ell)})^{-1}g^{(\ell)},
\end{equation}
valid for $\eta\gamma_\ell < 2/\lambda_{\max}(Q^{(\ell)})$ (numerical stability). The norm $\norm{\Delta\theta^{(\ell)}_T}$ is monotone increasing in $\gamma_\ell$ and saturates at $\norm{Q^{(\ell)\,-1}g^{(\ell)}}$ as $T\to\infty$ or $\gamma_\ell$ grows.

The trajectory under $\lambda$-\emph{penalized} gradient descent (with no LR scaling) on the loss $\tilde L_k(\Delta\theta) + \tfrac{1}{2}\lambda_\ell\norm{\Delta\theta^{(\ell)}}^2$ is
\begin{equation}
\Delta\theta^{(\ell)}_{t+1} \;=\; \Delta\theta^{(\ell)}_t - \eta\big((Q^{(\ell)} + \lambda_\ell I)\Delta\theta^{(\ell)}_t + g^{(\ell)}\big),
\end{equation}
with the same iterated form but $Q^{(\ell)}\to Q^{(\ell)} + \lambda_\ell I$. The fixed point is $-(Q^{(\ell)}+\lambda_\ell I)^{-1}g^{(\ell)}$, which has the same direction as $-(Q^{(\ell)})^{-1}g^{(\ell)}$ but a smaller magnitude.

\paragraph{Functional equivalence.}
The two trajectories produce different fixed points, but for any choice of the LR-scaling parameter $\gamma_\ell \in (0, 1]$ and total step budget $T$, one can find a $\lambda_\ell \ge 0$ such that the resulting $\Delta\theta^{(\ell)}_T$-norms agree. The exact match is implicit. SLCA's design choice $\gamma_\mathrm{backbone}=0.1, \gamma_\mathrm{head}=1$ is therefore equivalent in displacement-norm to a two-group $\lambda$-schedule with $\lambda_\mathrm{backbone}/\lambda_\mathrm{head}$ implicitly determined by $T$. The framework's prediction concerns the \emph{sign} (smaller LR for the more-sensitive group), not an exact algebraic identity.

\section{Direct $s_\ell$ measurements}
\label{app:hessian}

\paragraph{Procedure.}
For a backbone $f_\theta$ and a calibration loss $L(\theta) = \E_{(x,y)\sim p}[\ell(f_\theta(x), y)]$ on a batch sampled from CL benchmark $p$, the per-layer Hessian block is $H^{(\ell,\ell)} = \nabla^2_{\theta^{(\ell)}} L(\theta)$. We compute its top eigenvalue by power iteration on Hessian-vector products: starting from a random $v\in\R^{d_\ell}$, iterate $v \leftarrow \mathrm{HVP}(L, \theta^{(\ell)}, v)$, $v \leftarrow v/\norm{v}$, returning $\norm{\mathrm{HVP}(L, \theta^{(\ell)}, v)}$ after $k=15$ iterations. HVP is implemented in PyTorch via `torch.autograd.grad(grad, params, vec)`. The cost is $O(k)$ HVPs per layer (two backward passes each). For ViT-B/16 we force the math SDPA kernel since the flash and memory-efficient kernels do not support double-backward.

\paragraph{Measured profiles.}
Table~\ref{tab:s-ell} reports the full per-layer values for both backbones, with $d_\ell$ included for context. Both architectures use ImageNet-pretrained weights from torchvision, evaluated on a batch of 32 random samples with a 100-class head (matching CIFAR-100 / Split-CIFAR-100 cardinality).

\begin{table}[h]
\centering
\small
\caption{Measured per-layer top Hessian eigenvalue $s_\ell$ on a 100-class CIFAR-100-style cross-entropy minibatch.}
\label{tab:s-ell}
\begin{tabular}{lrr@{\hskip 24pt}lrr}
\toprule
\multicolumn{3}{c}{\textbf{ResNet-50}} & \multicolumn{3}{c}{\textbf{ViT-B/16}} \\
\cmidrule(lr){1-3}\cmidrule(lr){4-6}
block & $d_\ell$ & $s_\ell$ & block & $d_\ell$ & $s_\ell$ \\
\midrule
stem    & 9{,}536      & 39.05    & patch\_embed & 590{,}592   & 78.21 \\
layer1  & 215{,}808    & \textbf{3{,}167.48} & block0 & 7{,}087{,}872 & 10.70 \\
layer2  & 1{,}219{,}584 & 519.36  & block1 & 7{,}087{,}872 & 2.85 \\
layer3  & 7{,}098{,}368 & 139.32  & block2-11 & $\sim$7M each & 3--7 \\
layer4  & 14{,}964{,}736 & 57.73 & head & 76{,}900 & 6.61 \\
fc      & 204{,}900    & 23.53    & & & \\
\midrule
\multicolumn{3}{l}{$\kappa = 134.6$} & \multicolumn{3}{l}{$\kappa = 27.4$} \\
\bottomrule
\end{tabular}
\end{table}

\paragraph{Stability across batches.}
We measure $s_\ell$ over 5 independent calibration batches per backbone. The relative ordering of layers is stable across batches, and absolute $s_\ell$ values fluctuate by $\sim$5--15\%. The reported $\kappa$ in Table~\ref{tab:kappa-scan} is the mean-over-seeds.

\paragraph{Backbone scan, $\kappa$ vs model size.}
Table~\ref{tab:kappa-scan} measures $\kappa$ across six standard backbones, all ImageNet-pretrained. ResNets show $\kappa$ growing with depth and capacity (R18: 55, R34: 75, R50: 195, with mild saturation at R101: 181). The dominant block is consistently in the early conv stage (\texttt{layer1} or \texttt{layer3}). Vision transformers show smaller $\kappa$ ($\sim$25--30) and the dominant block varies by depth (patch embedding for ViT-B, a mid-stack block for ViT-L). Both families satisfy the LIH ($\kappa \gg 1$).

\begin{table}[h]
\centering
\small
\caption{Per-layer Hessian condition number $\kappa = \max_\ell s_\ell / \min_\ell s_\ell$ measured across six ImageNet-pretrained backbones, with means taken over 5 calibration batches. Full per-layer numbers are in Appendix~\ref{app:hessian}. $L$ is the number of architectural blocks used in the partition.}
\label{tab:kappa-scan}
\begin{tabular}{lrrrll}
\toprule
Backbone & $L$ & Params (M) & $\kappa$ & Most-sensitive & Least-sensitive \\
\midrule
ResNet-18  & 6 &  11.2 &  55.1 & layer1      & stem   \\
ResNet-34  & 6 &  21.3 &  75.4 & layer3      & stem   \\
ResNet-50  & 6 &  23.7 & 195.3 & layer1      & fc     \\
ResNet-101 & 6 &  42.7 & 181.4 & layer1      & layer4 \\
ViT-B/16   & 14 & 85.7 &  31.5 & patch\_embed & block1 \\
ViT-L/16   & 26 & 303.2 &  22.3 & block9      & block0 \\
\bottomrule
\end{tabular}
\end{table}

\paragraph{Off-diagonal coupling $\rho$.}
We also estimate the off-diagonal block norm $\norm{H^{(\ell,\ell')}}_{\mathrm{op}}$ between adjacent layer pairs and report $\rho_{\ell,\ell'} = \norm{H^{(\ell,\ell')}}_{\mathrm{op}} / \sqrt{s_\ell s_{\ell'}}$ (the quantity bounded by Assumption~\ref{ass:block-diag-relax}). Across the six backbones the mean adjacent $\rho$ ranges from $0.76$ (ViT-L/16) to $1.34$ (ResNet-101); Table~\ref{tab:reb-kappa-eff} in Appendix~\ref{app:rho} lists all six values. Values $\rho > 1$ indicate that Assumption~\ref{ass:block-diag-relax}'s quantitative bound (which requires $\rho < 1$) is not strictly satisfied at adjacent residual stages. The qualitative block-diagonal structure still holds, but the regret-bound degradation in Appendix~\ref{app:rho} is a genuine caveat at large coupling.

\paragraph{Harmonized profiles across pretrainings.}
The ViT-B/16 profile in Table~\ref{tab:s-ell} uses the supervised ImageNet-1k checkpoint (IN1k-sup). Figure~\ref{fig:reb-pretrain-profiles} repeats the measurement on the same architecture under two further pretrainings, ImageNet-21k supervised (IN21k) and MoCoV3 self-supervised, with an identical protocol: the same 14-block partition, the same 100-class calibration head, and a fixed calibration seed shared across the three checkpoints. Small differences between the IN1k-sup column and Table~\ref{tab:s-ell} come from this fixed seed. The MoCoV3 checkpoint was rebuilt from the official release; the conversion was validated with zero missing or unexpected keys on loading. Under this protocol $\kappa$ is $31.5$ for IN1k-sup (matching the 5-batch mean in Table~\ref{tab:kappa-scan}), $14.6$ for IN21k, and $27.4$ for MoCoV3. The three profiles disagree in depth ordering: IN1k-sup and MoCoV3 peak early, at the patch embedding, while IN21k peaks late, at block10. No single fixed $\alpha$ therefore transfers across pretrainings; the framework conditions the schedule on the measured profile of the actual checkpoint. The MoCoV3 mid-stack is nearly flat (blocks 1--11 span $2.7$--$5.2$), which is consistent with the near-flat sweep in Appendix~\ref{app:slca-mocov3}: apart from the early peak there is little heterogeneity for a depth-monotone schedule to exploit.

\begin{figure}[h]
\centering
\includegraphics[width=0.9\linewidth]{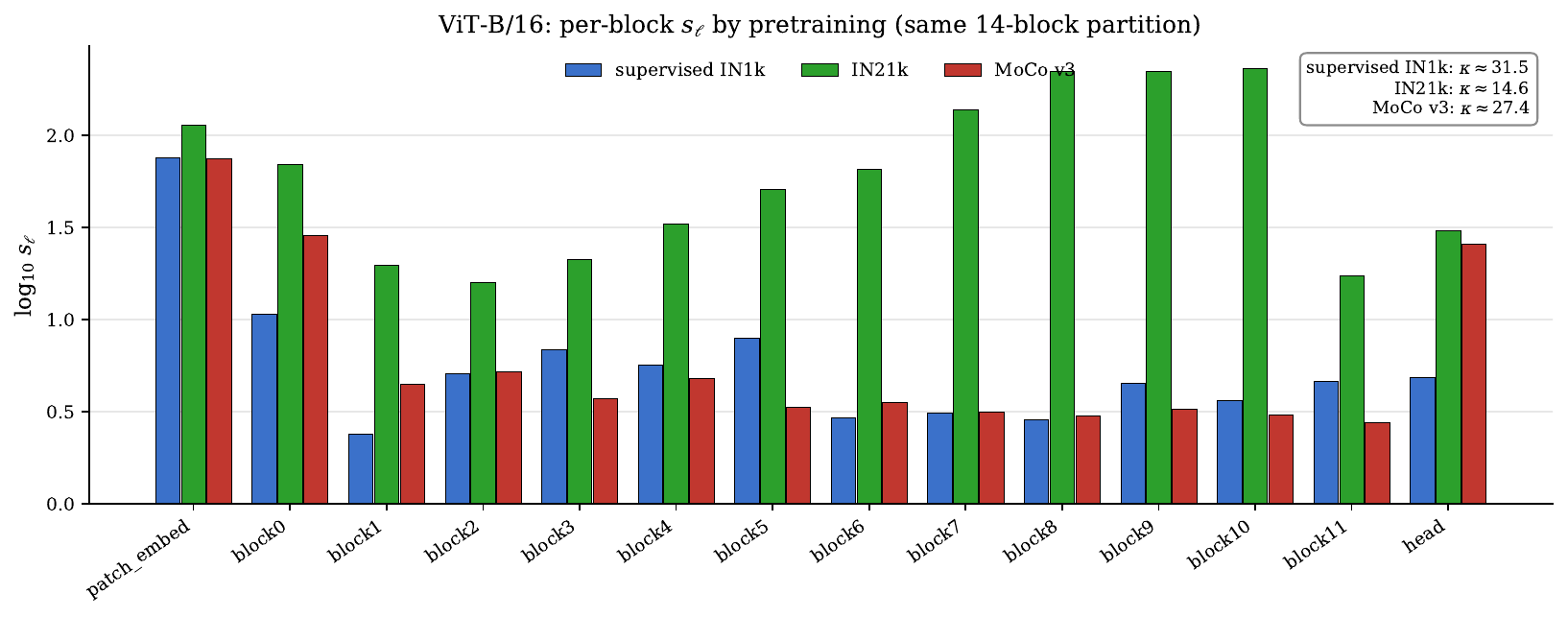}
\caption{Per-layer $s_\ell$ (log scale) on ViT-B/16 under three pretrainings, measured with the same 14-block partition, 100-class calibration head, and a fixed calibration seed. $\kappa$: 31.5 (IN1k-sup), 14.6 (IN21k), 27.4 (MoCoV3). IN1k-sup and MoCoV3 peak at the patch embedding; IN21k peaks at block10.}
\label{fig:reb-pretrain-profiles}
\end{figure}

\paragraph{Random-initialization profiles.}
All profiles above are measured on trained checkpoints. At random initialization the ordering inverts and the deepest blocks are hottest (3 seeds each). On ResNet-50, $s_\ell$ rises monotonically from the stem ($s_{\text{stem}} \approx 301$) to the fc head ($s_{\text{fc}} \approx 127{,}000$). ResNet-18 peaks at layer4 ($\approx 32{,}000$), and SmallCNN has $\kappa \approx 200$ at initialization. Heterogeneity is therefore present at init; it is the depth ordering that flips, not the size of $\kappa$. The early-layer rule of thumb is a property of trained checkpoints, not of the architecture alone. This does not affect the from-scratch EWC results in Appendix~\ref{app:aewc-extra}: EWC anchors its penalty at the trained parameters after each task, and the trajectory measurements in Appendix~\ref{app:pertask} show the profile migrating during task 1 and becoming rank-stable from task 2 on, so the schedule is consumed at anchor points where the profile is already that of a trained network.

\section{Depth-weighted EWC on from-scratch backbones}
\label{app:aewc-extra}

\paragraph{Split-CIFAR-100 with MediumCNN.}
This is the discriminating from-scratch setting in our EWC experiments and is reported in the main text as Figure~\ref{fig:aewc-CIFAR-100}. One quantitative addition that the main-text discussion does not include is the magnitude of the forgetting trend, with forgetting climbing from $\sim 0.005$ at $\alpha=2^{-2}$ to $\sim 0.16$ at $\alpha=2^3$. The opposing avg-acc and forgetting signals do not contradict the framework. MediumCNN is trained from scratch and its per-layer sensitivities are not the shallow-dominated pattern Figure~\ref{fig:lih} measures on pretrained backbones, so the rule-of-thumb's domain assumption fails here. The framework still applies because it predicts the optimum tracks $s_\ell$, whatever the depth ordering of $s_\ell$ happens to be. The avg-acc preference for $\alpha>1$ on this longer task sequence is then dominated by plasticity on the new task rather than by stability on the old one.

\paragraph{Split-CIFAR-10 with SmallCNN.}
On the smaller backbone and shorter task sequence, the average-accuracy effect of layer-adaptivity is statistically marginal. The avg-acc maximum at $c=200$ is at $\alpha\in\{1,2\}$ with 3-seed std bands overlapping $\alpha=0.5$, so a strict ``non-uniform beats uniform'' claim is not supported on this backbone. The forgetting and BWT panels are clearer: forgetting at $\alpha=2^3$ is $\sim$8$\times$ higher than at $\alpha=2^{-2}$, so deep-layer sensitivity is real even when the avg-acc metric does not separate the schedules.

\begin{figure}[h]
\centering
\includegraphics[width=0.49\linewidth]{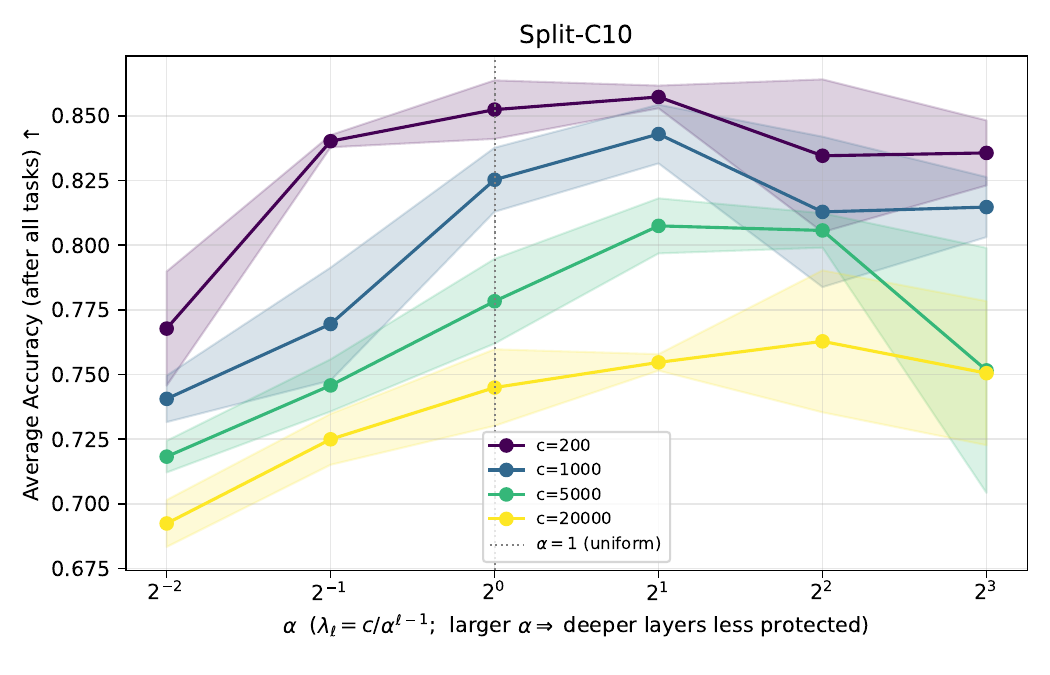}\hfill
\includegraphics[width=0.49\linewidth]{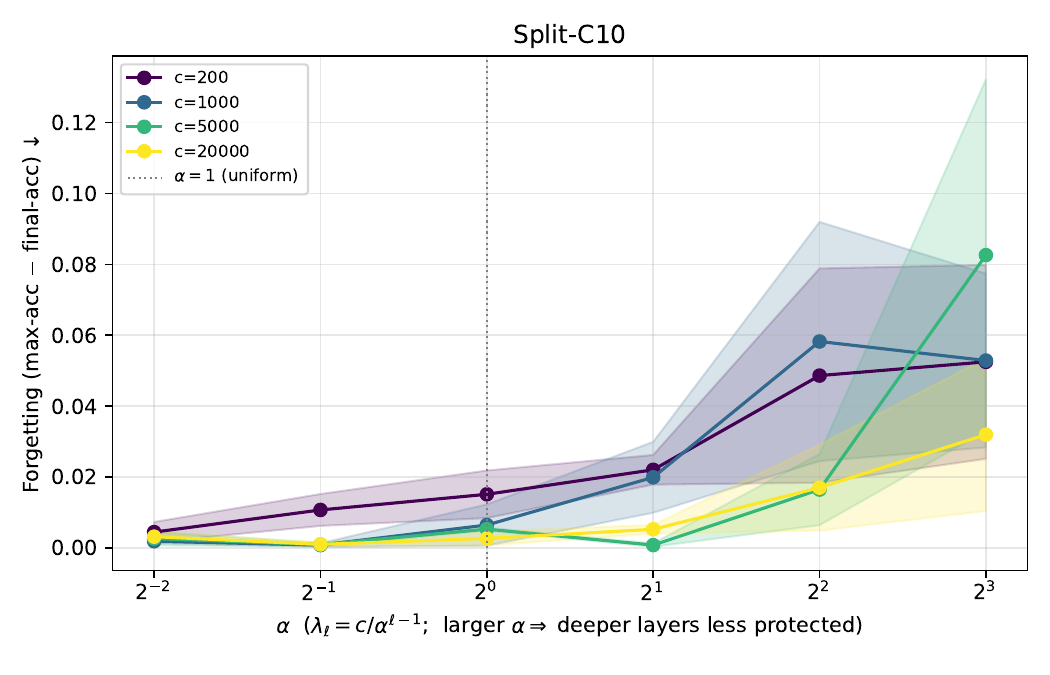}\\[2pt]
\includegraphics[width=0.49\linewidth]{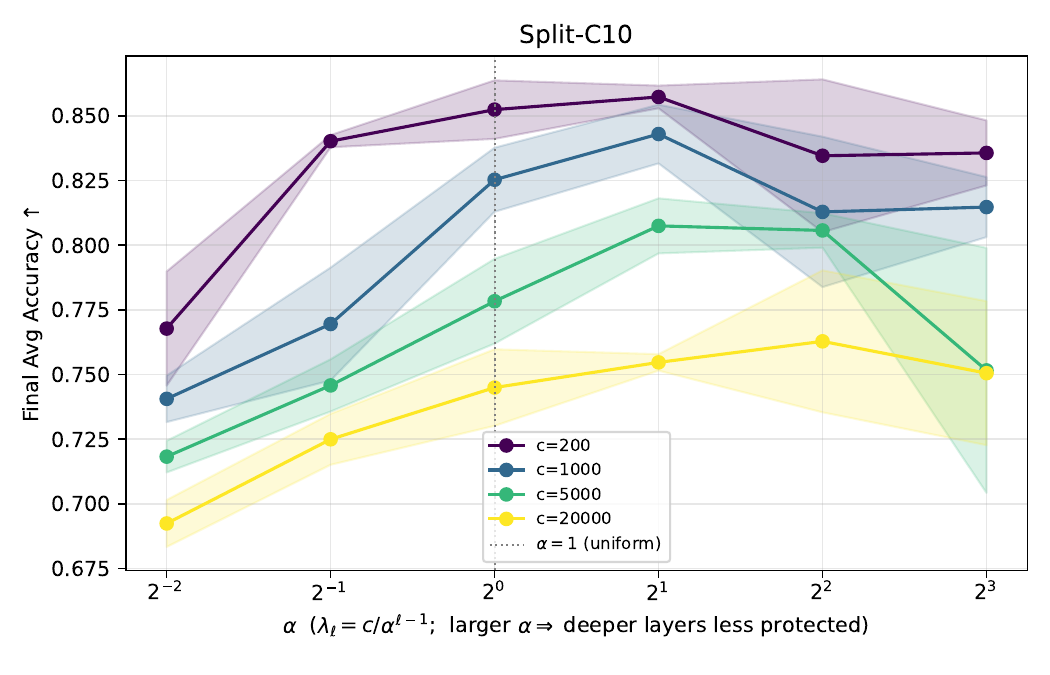}\hfill
\includegraphics[width=0.49\linewidth]{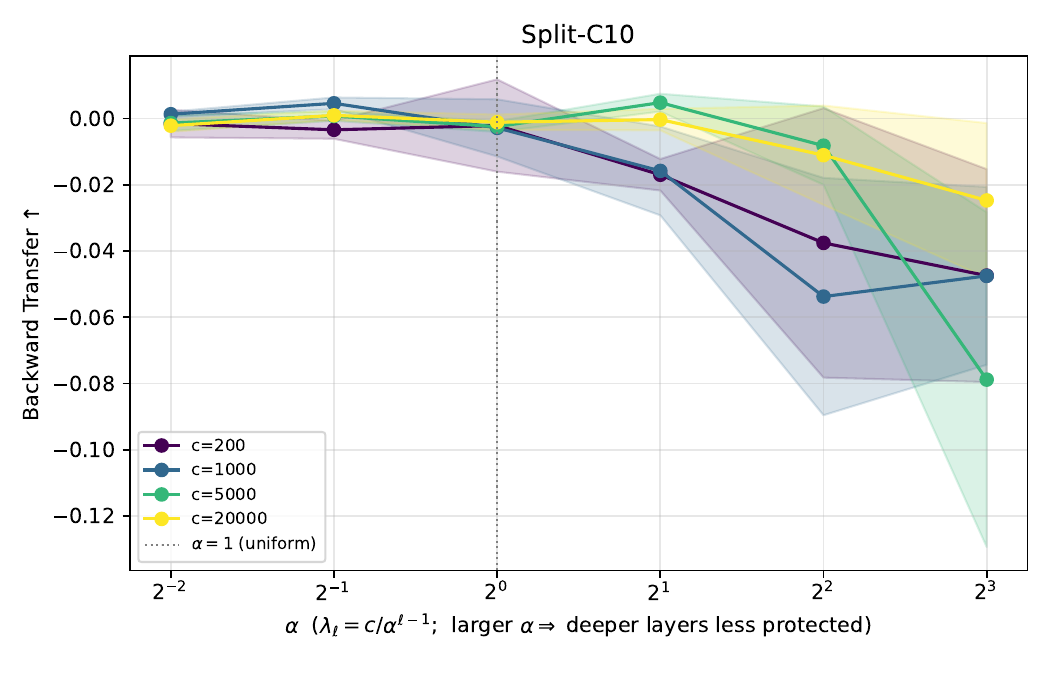}
\caption{Depth-weighted EWC on Split-CIFAR-10 (SmallCNN, 5 tasks, 3 seeds, mean$\pm$std bands). Same panel layout as Figure~\ref{fig:aewc-CIFAR-100}.}
\label{fig:aewc-CIFAR-10}
\end{figure}

\paragraph{Split-CIFAR-100 with ResNet-18.}
On a deeper backbone the geometric schedule does \emph{not} consistently beat uniform. The avg-acc maximum for $c=200$ is at $\alpha=1$ (mean $0.55$) with overlapping std bands at $\alpha\in\{0.5,4\}$. This is consistent with the measured ResNet-50 per-layer sensitivities (Figure~\ref{fig:lih}), which is non-monotone in depth (a layer1 spike, then decreasing). A one-parameter geometric prior $\lambda_\ell\propto\gamma^{\ell-1}$ cannot fit a non-monotone $s_\ell$. A non-monotone schedule (\textit{e.g.},\ a per-block estimate of $s_\ell$) would be expected to recover the gain.

\begin{figure}[h]
\centering
\includegraphics[width=0.49\linewidth]{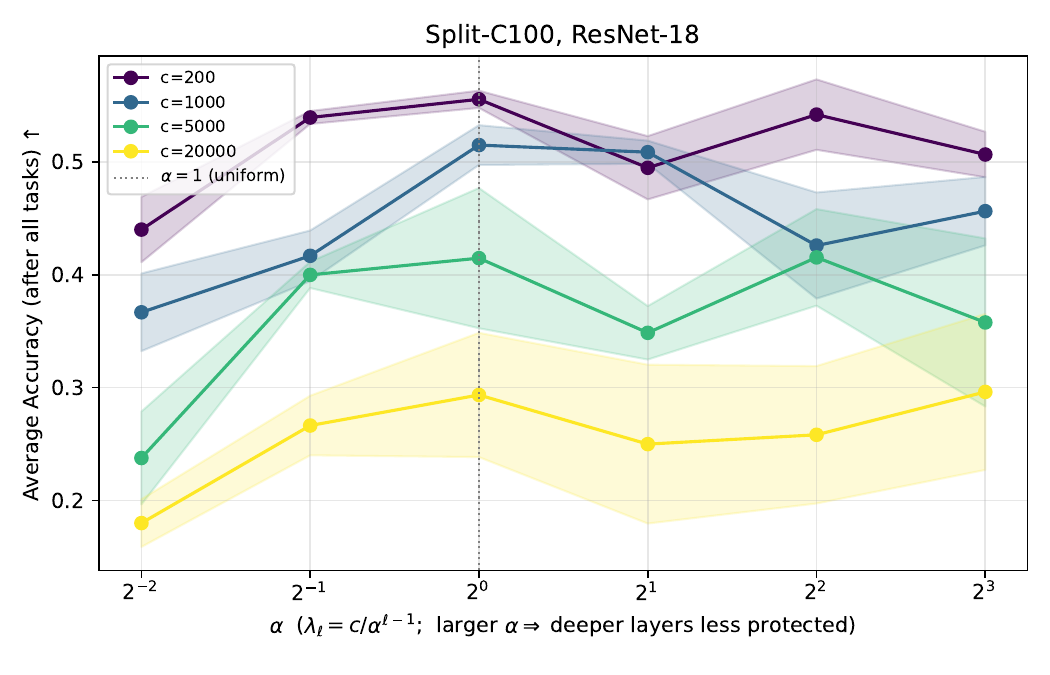}\hfill
\includegraphics[width=0.49\linewidth]{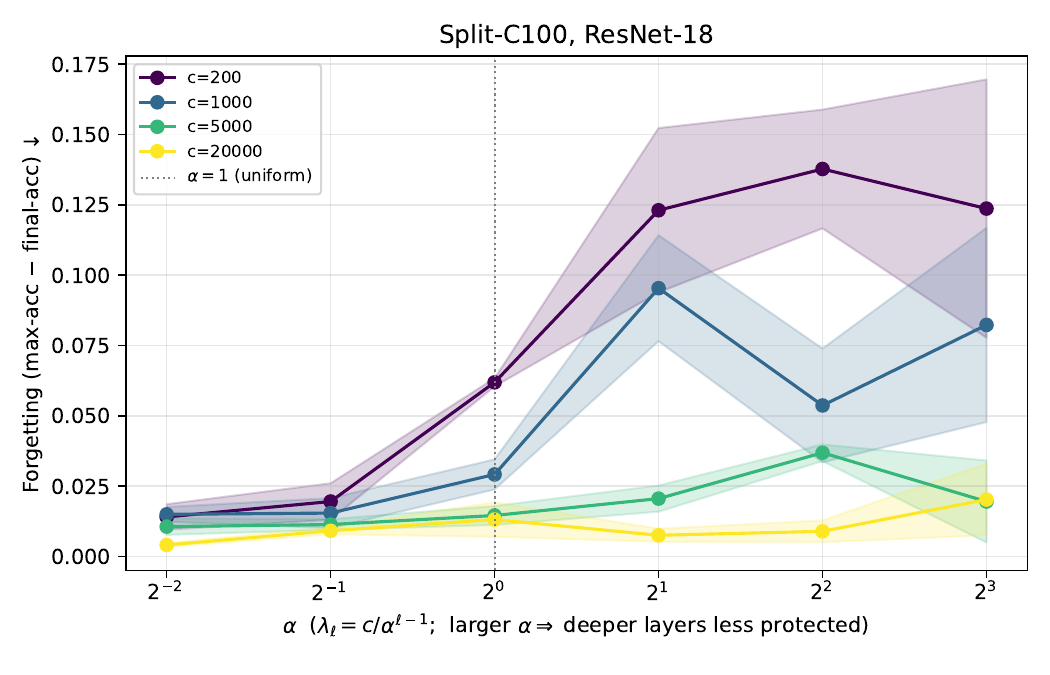}\\[2pt]
\includegraphics[width=0.49\linewidth]{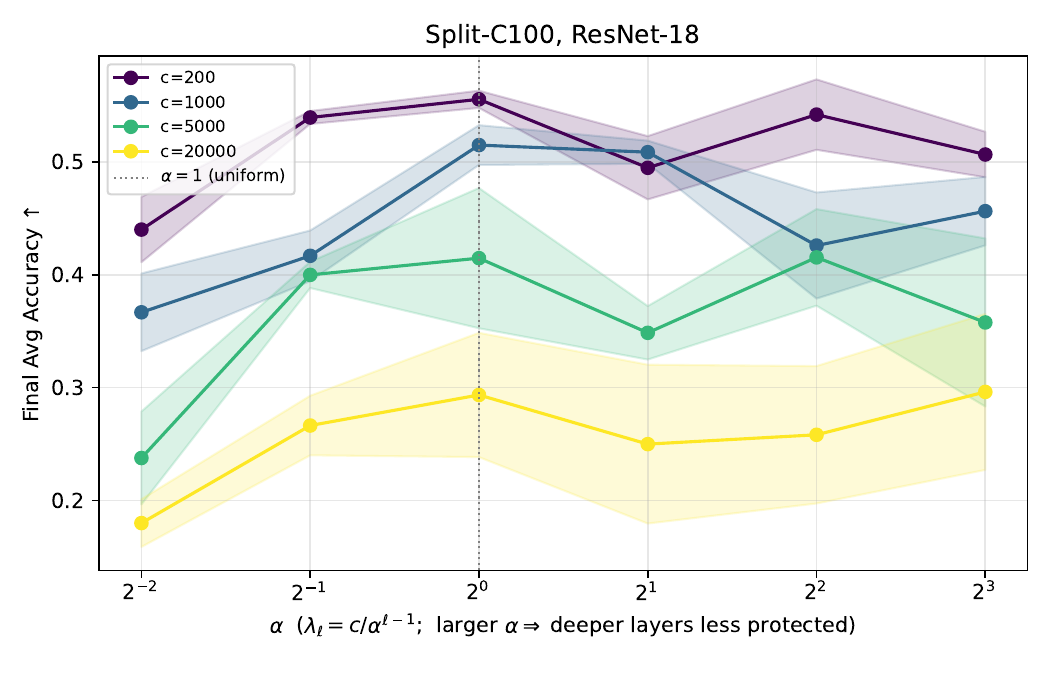}\hfill
\includegraphics[width=0.49\linewidth]{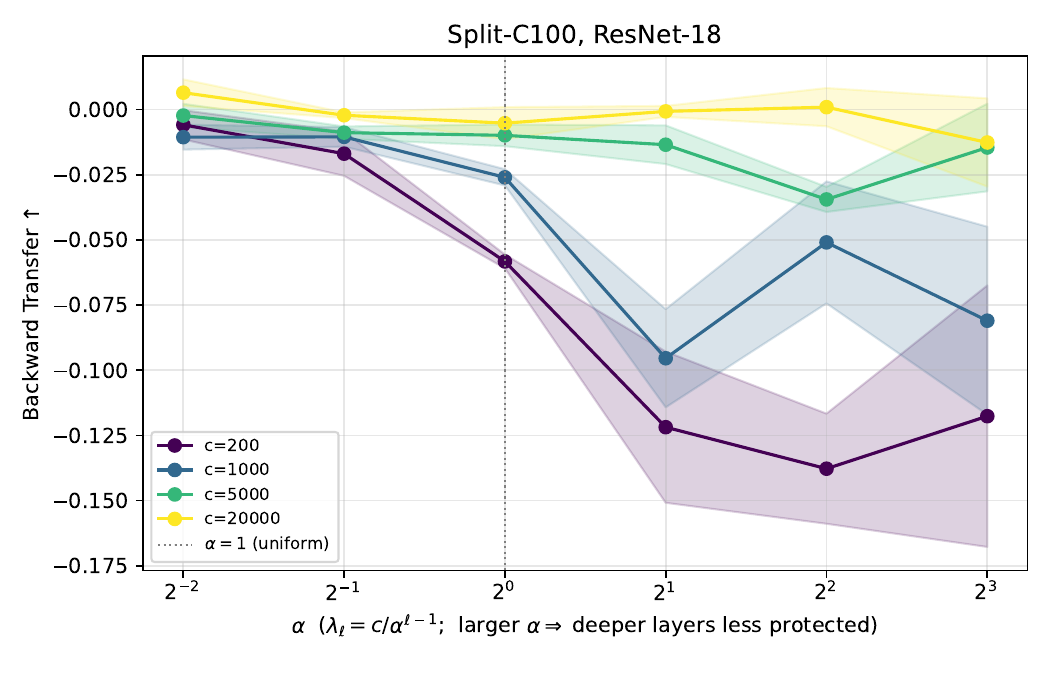}
\caption{Depth-weighted EWC on Split-CIFAR-100 with a deeper backbone (ResNet-18, 6 architectural blocks, 10 tasks, 3 seeds). Same panel layout as Figure~\ref{fig:aewc-CIFAR-100}.}
\label{fig:aewc-r18}
\end{figure}

\section{SLCA depth-weighted sweep on MoCoV3 pretraining}
\label{app:slca-mocov3}

For comparison with the ImageNet-21k pretraining in the main text (Figure~\ref{fig:slca-in21k}), we report the same depth-weighted SLCA sweep on a ViT-B/16 backbone pretrained with MoCoV3 (the pretraining used in the SLCA paper's ablations). The full $3\times 5$ bounded grid spans $80.0$--$84.8\%$ final-task accuracy ($87.9$--$90.4\%$ incremental average). Default SLCA corresponds to $(c=1, \alpha=1)$ at $84.77\%$. The grid optimum is $(c=0.5, \alpha=0.92)$ at $84.79\%$, within noise of the default. The interesting structure is row-internal: at $c=2$ the row is monotone increasing in $\alpha$ (final-acc rises from $80.0\%$ at $\alpha=0.85$ to $84.5\%$ at $\alpha=1.18$), so when one over-regularizes the entire backbone, transferring some protection from shallow to deep blocks is the right correction. The near-flatness near $(c=1, \alpha=1)$ is consistent with MoCoV3 having per-layer sensitivities that the default two-group LR ratio (backbone $0.1\times$, head $1\times$) already approximates well.

\begin{figure}[h]
\centering
\includegraphics[width=\linewidth]{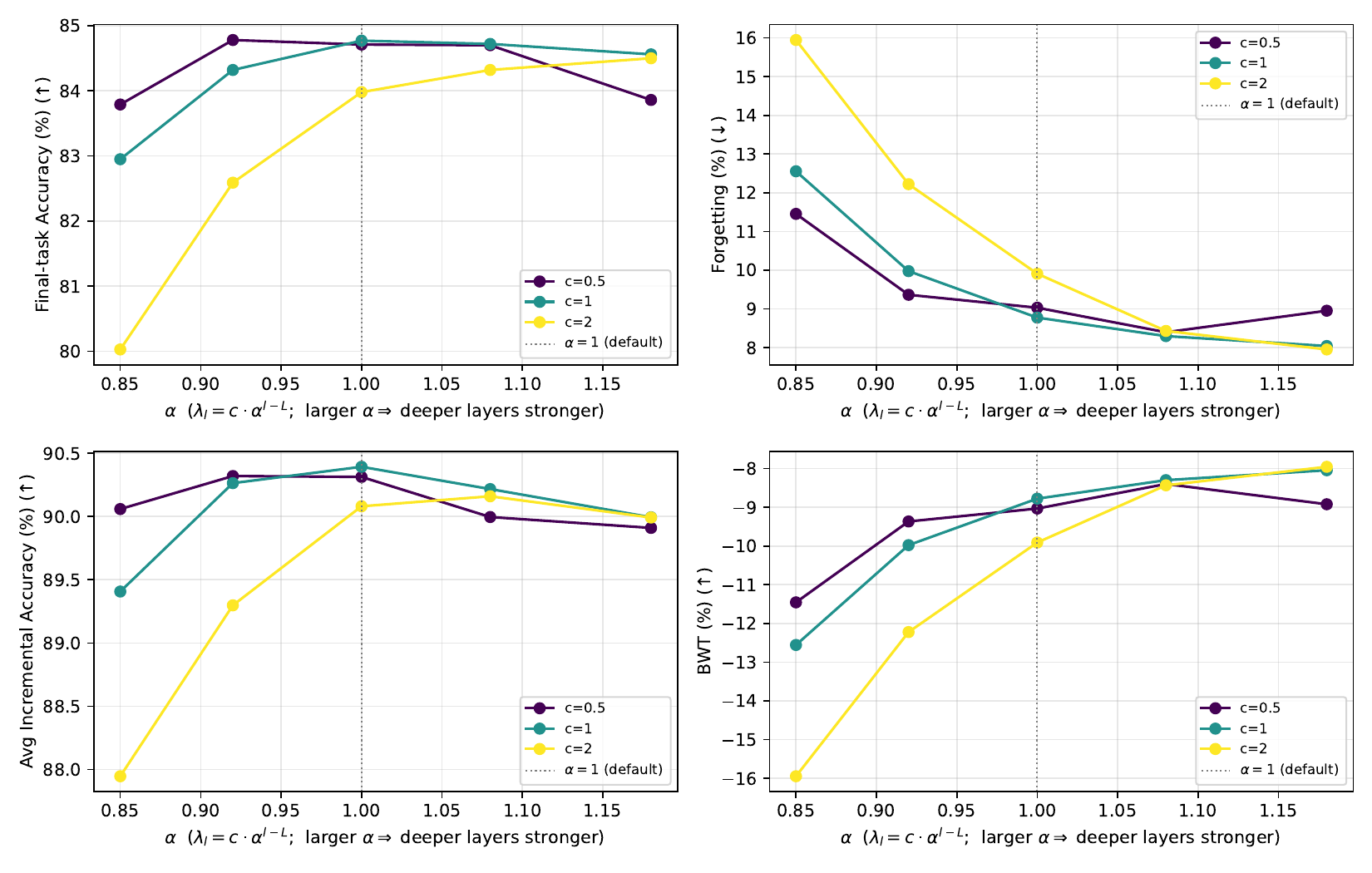}
\caption{SLCA depth-weighted sweep on Split-CIFAR-100 (ViT-B/16 + MoCoV3, 10 tasks). Same parametrization and panel layout as Figure~\ref{fig:slca-in21k}.}
\label{fig:slca}
\end{figure}

\paragraph{Multi-seed comparison.}
Table~\ref{tab:slca-multiseed} repeats the key configurations of both sweeps over multiple seeds, together with the literal measured-$s_\ell$ schedule (the SLCA analog of Appendix~\ref{app:measured-ewc}). The SLCA trainer pins the torch seed, so the seed varies only the task/class ordering; the reported spread measures robustness to task order. On ImageNet-21k the geometric schedule beats uniform on every seed. Paired by seed, the Last-acc gain is $+0.450$ (per-seed $+0.52/+0.56/+0.51/+0.44/+0.22$), paired $t$-test $p=0.0018$; the Wilcoxon signed-rank test gives $p=0.0625$, its smallest attainable value at $n=5$. The Inc-acc gain is $+0.112$ (paired $t$, $p=0.022$). On MoCoV3 the paired difference is $-0.18$ and not significant ($p=0.29$). This is the parity the measured sensitivity profiles (Appendix~\ref{app:hessian}) predict: default SLCA's two-group ratio already fits the MoCoV3 profile, so a depth-monotone reweighting has no room to win. The literal measured-$s_\ell$ schedule trails uniform on both pretrainings, mildly on ImageNet-21k and severely on MoCoV3 ($-7.84$ Last-acc, paired $t$, $p=0.0021$), consistent with the dynamic-range failure of the literal prescription documented in Appendix~\ref{app:measured-ewc}.

\begin{table}[h]
\centering
\footnotesize
\caption{Multi-seed SLCA on Split-CIFAR-100 (10 tasks, ViT-B/16), mean$\pm$std over $n$ task-order seeds, all values in \%. Last-acc and Inc-acc as in Table~\ref{tab:slca-cmp}. Geometric and measured-$s_\ell$ rows use each pretraining's best $(c,\alpha)$ and best $c$, respectively.}
\label{tab:slca-multiseed}
\begin{tabular}{llcccc}
\toprule
Pretraining & Schedule & $n$ & Last-acc & Inc-acc & Forgetting \\
\midrule
ImageNet-21k & uniform ($c=1$, $\alpha=1$) & 5 & $91.35\pm 0.23$ & $94.40\pm 0.69$ & $5.83\pm 0.42$ \\
ImageNet-21k & geometric, ours ($c=1$, $\alpha=1.18$) & 5 & $91.80\pm 0.23$ & $94.51\pm 0.66$ & $4.37\pm 0.49$ \\
ImageNet-21k & measured $s_\ell$ ($c=1$) & 3 & $90.96\pm 0.18$ & $94.45\pm 0.21$ & $6.61\pm 0.28$ \\
\midrule
MoCoV3 & uniform ($c=1$, $\alpha=1$) & 3 & $84.44\pm 0.29$ & $90.12\pm 0.32$ & $9.33\pm 0.44$ \\
MoCoV3 & geometric, ours ($c=0.5$, $\alpha=0.92$) & 3 & $84.26\pm 0.48$ & $90.08\pm 0.32$ & $10.09\pm 0.62$ \\
MoCoV3 & measured $s_\ell$ ($c=0.5$) & 3 & $76.60\pm 0.71$ & $86.68\pm 0.22$ & $20.41\pm 0.85$ \\
\bottomrule
\end{tabular}
\end{table}

\section{TUNA layer-wise orthogonality sweep}
\label{app:tuna}

TUNA inserts per-layer adapters with an orthogonality penalty across tasks. The penalty has a per-layer weight $\lambda_j = c/\alpha^j$, fulfilling Eq.~\ref{eq:rule-penalty}, with the same parametrization as Experiment~1 in the main text (here $j$ indexes transformer blocks, $L=12$). Sweeping $\alpha$ at fixed $c$ probes whether the optimal orthogonality protection grows or shrinks with depth in a pretrained ViT-B/16 backbone. The adapter family already restricts $\Delta\theta^{(\ell)}$ to a $16$-dimensional subspace per transformer block, so the residual room for layer-wise reweighting is small. The result is consistent with the rule of thumb in direction but smaller in magnitude than what we see on EWC and SLCA, which is why we report it here rather than in the main text. We test on two benchmarks, namely CIFAR-100 B0-Inc10 ($10$ tasks of $10$ classes) and ImageNet-R B0-Inc20 ($20$ tasks, harder distribution shift). Same four metrics as Experiment~1.

\paragraph{ImageNet-R B0-Inc20.}
Figure~\ref{fig:tuna-inr} shows the result on the more demanding ImageNet-R B0-Inc20 setting. The optimum is visibly above $\alpha=1$, around $\alpha\in\{2^1,2^2\}$, suggesting that on this dataset the orthogonality penalty should be relaxed on the late blocks.

\begin{figure}[h]
\centering
\includegraphics[width=0.49\linewidth]{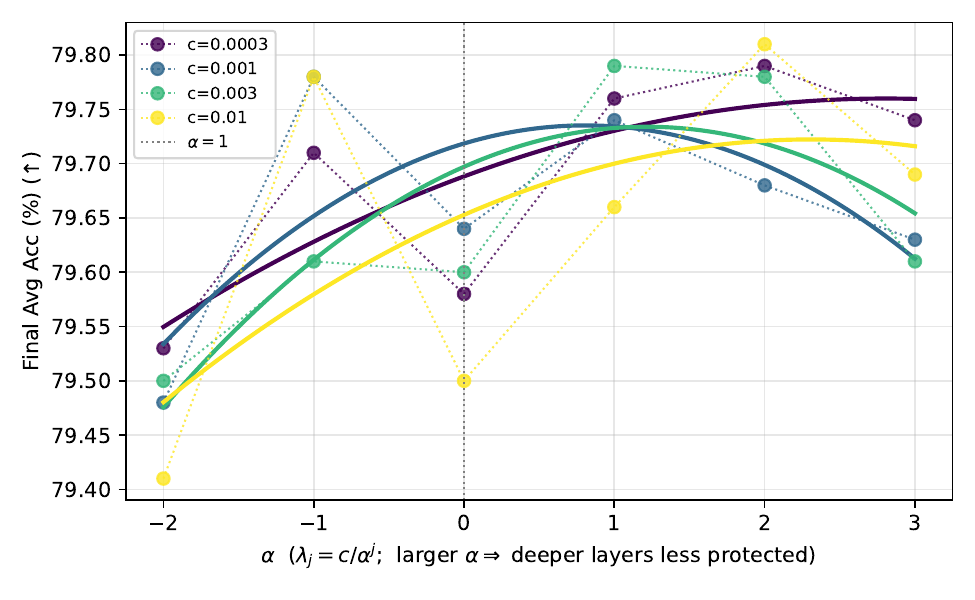}\hfill
\includegraphics[width=0.49\linewidth]{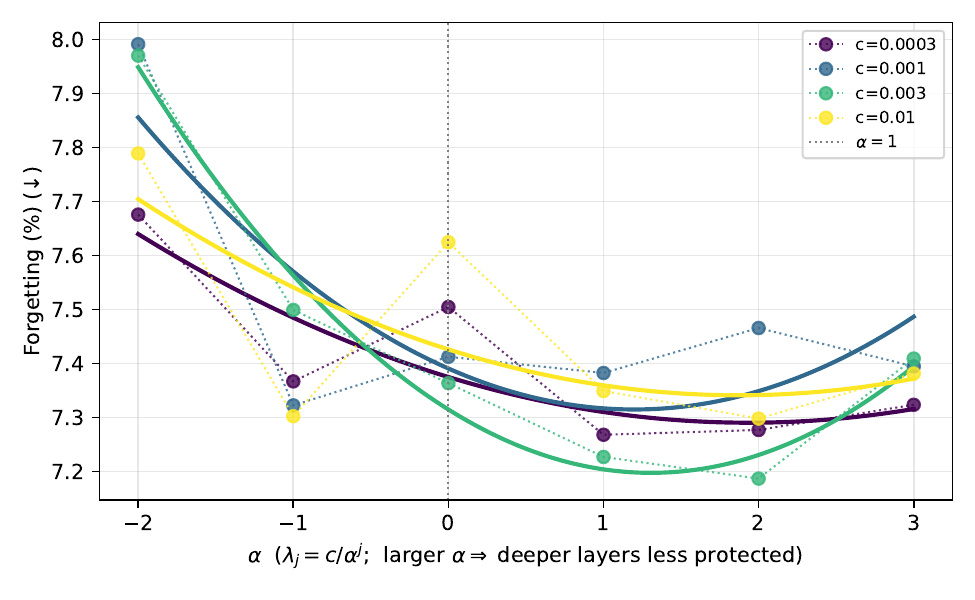}\\[2pt]
\includegraphics[width=0.49\linewidth]{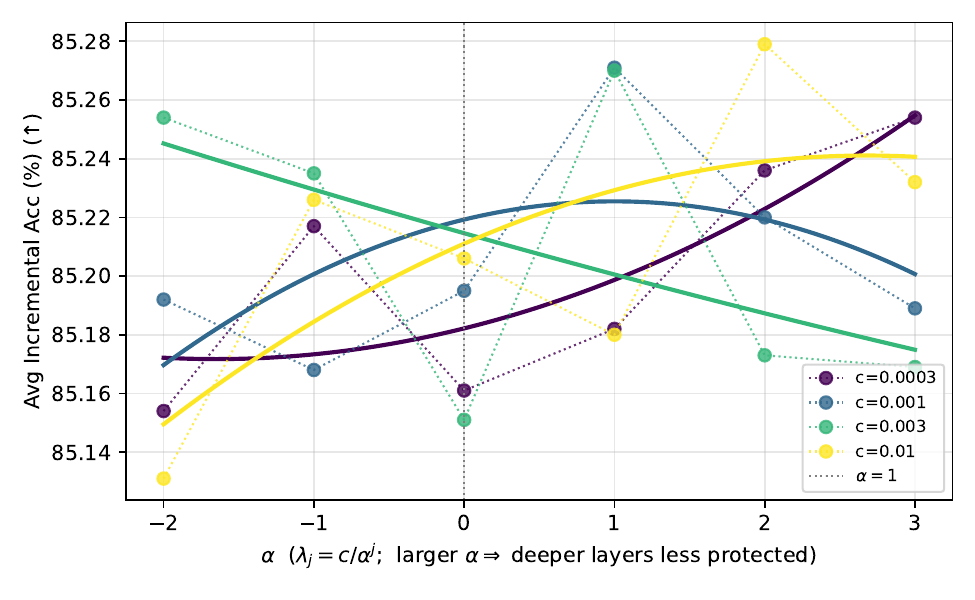}\hfill
\includegraphics[width=0.49\linewidth]{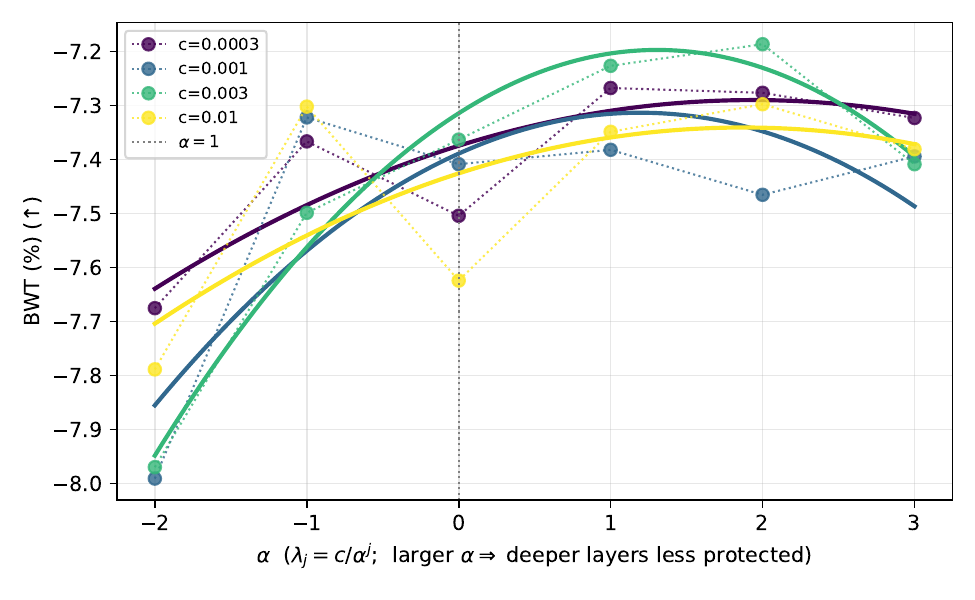}
\caption{TUNA orthogonality sweep on ImageNet-R B0-Inc20 (ViT-B/16). Per-layer orthogonality weight $\lambda_j = c/\alpha^j$. Top, final average accuracy (left) and forgetting (right). Bottom, average incremental accuracy (left) and BWT (right). The optimum is visibly above $\alpha=1$, suggesting that on this benchmark late-block orthogonality should be relaxed relative to early blocks.}
\label{fig:tuna-inr}
\end{figure}

\paragraph{CIFAR-100 B0-Inc10.}
The accuracy span across the $(\alpha, c)$ grid on this easier benchmark is only about $0.25\%$, since the adapter constraint already does the bulk of the work and the residual layer-wise tuning has limited room to act. Within that small span the optimum lies away from $\alpha=1$ for every $c$, with the typical preference around $\alpha\in\{2^{-1}, 2^1\}$.

\begin{figure}[h]
\centering
\includegraphics[width=0.49\linewidth]{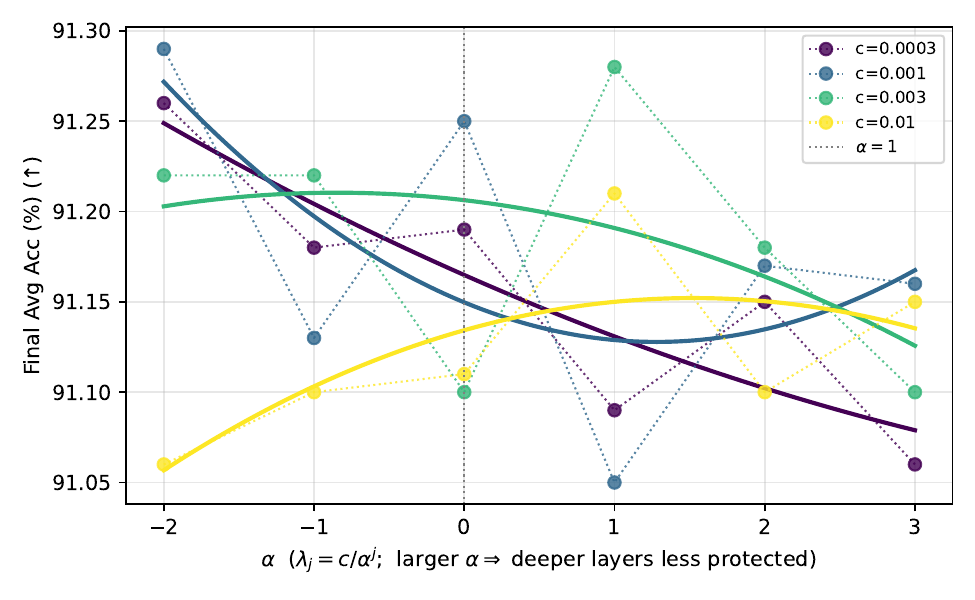}\hfill
\includegraphics[width=0.49\linewidth]{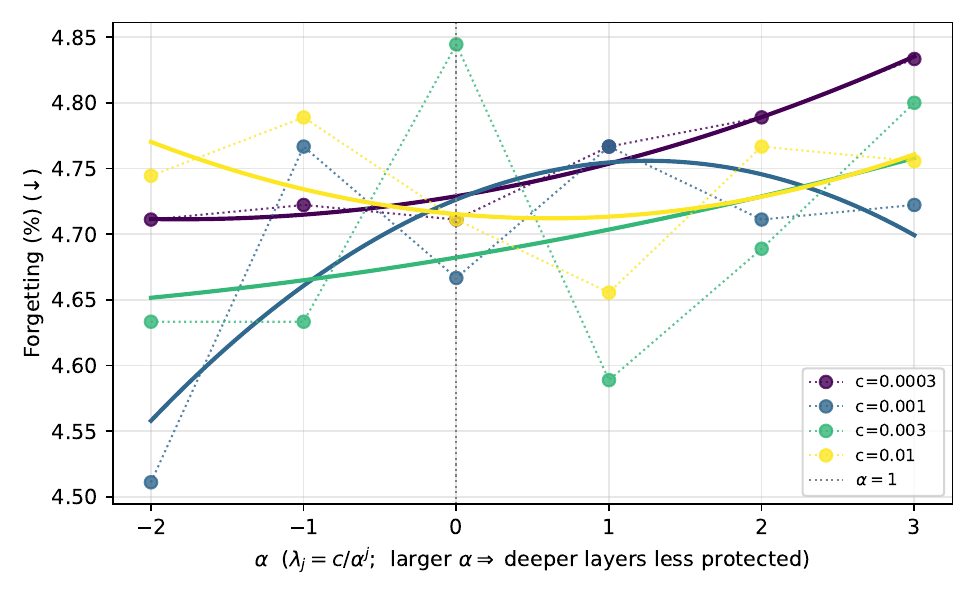}\\[2pt]
\includegraphics[width=0.49\linewidth]{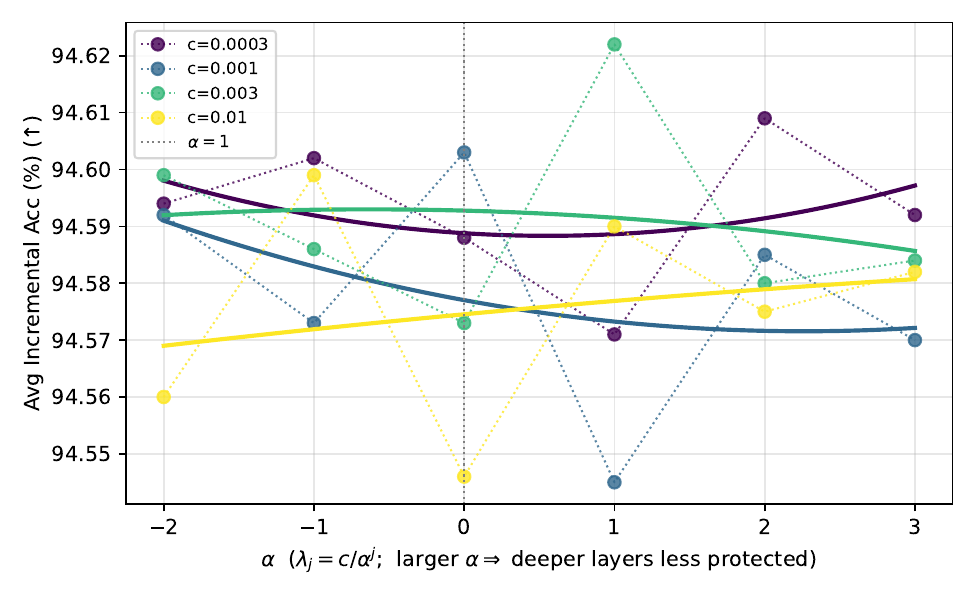}\hfill
\includegraphics[width=0.49\linewidth]{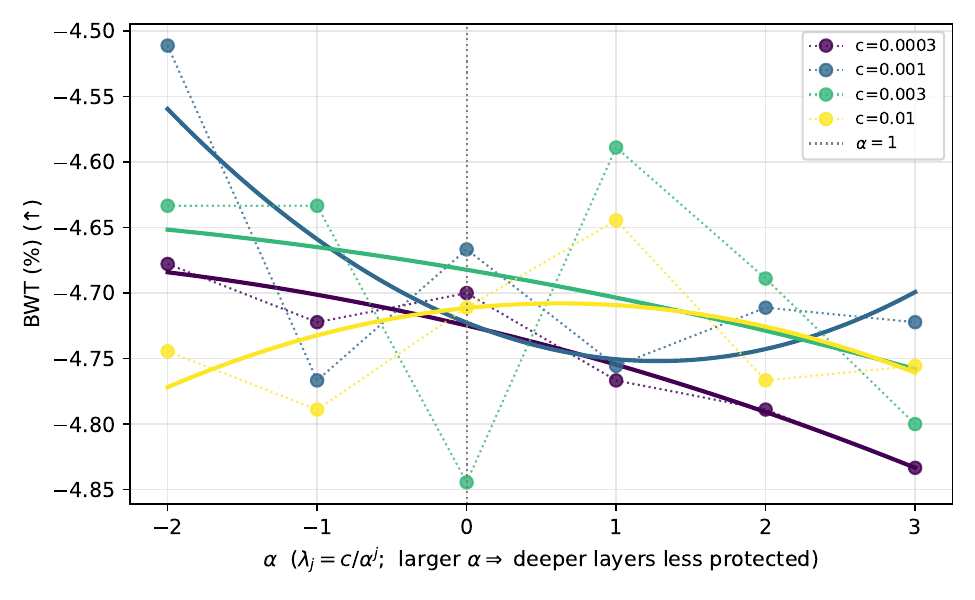}
\caption{TUNA orthogonality sweep on CIFAR-100 B0-Inc10 (ViT-B/16, 10 tasks). Per-layer orthogonality weight $\lambda_j = c/\alpha^j$. Same panel layout as Figure~\ref{fig:tuna-inr}.}
\label{fig:tuna-CIFAR-100}
\end{figure}

\section{Schedule-shape comparison}
\label{app:sched-compare}

The framework predicts that the per-layer optimum tracks $s_\ell$, but in practice we replace the prescription with a one-parameter family. If the rule of thumb (protect shallow layers strongly, let deeper layers move) is the load-bearing piece, the specific monotone-decreasing shape should matter less than the principle. We test this directly by running two additional schedule families on MediumCNN/Split-CIFAR-100 (3 seeds each):
\begin{itemize}[leftmargin=*,topsep=2pt]
\item \textbf{Linear}: $\lambda_\ell = c\cdot\bigl(1 - (\alpha-1)(\ell-1)/(L-1)\bigr)$, clipped to $[c/100, c]$, with $\alpha\in\{1.5, 2, 3\}$.
\item \textbf{Step (two-group)}: $\lambda_\ell = c$ for $\ell\le L/2$ and $\lambda_\ell = c/\alpha$ otherwise, with $\alpha\in\{2, 4, 8\}$.
\end{itemize}
Figure~\ref{fig:sched-compare} reports the best avg-acc across each family's $\alpha$ grid, alongside the geometric family already in the main text. All three monotone-decreasing schedules consistently beat uniform; the absolute gains over uniform are $\sim$2--7\% depending on $c$, and the three families' best points lie within $\sim$1--4\% of each other. The step family is in fact the strongest at higher $c$, edging out geometric. The shape of the schedule is therefore a second-order choice once the rule of thumb is followed.

\begin{figure}[h]
\centering
\includegraphics[width=0.55\linewidth]{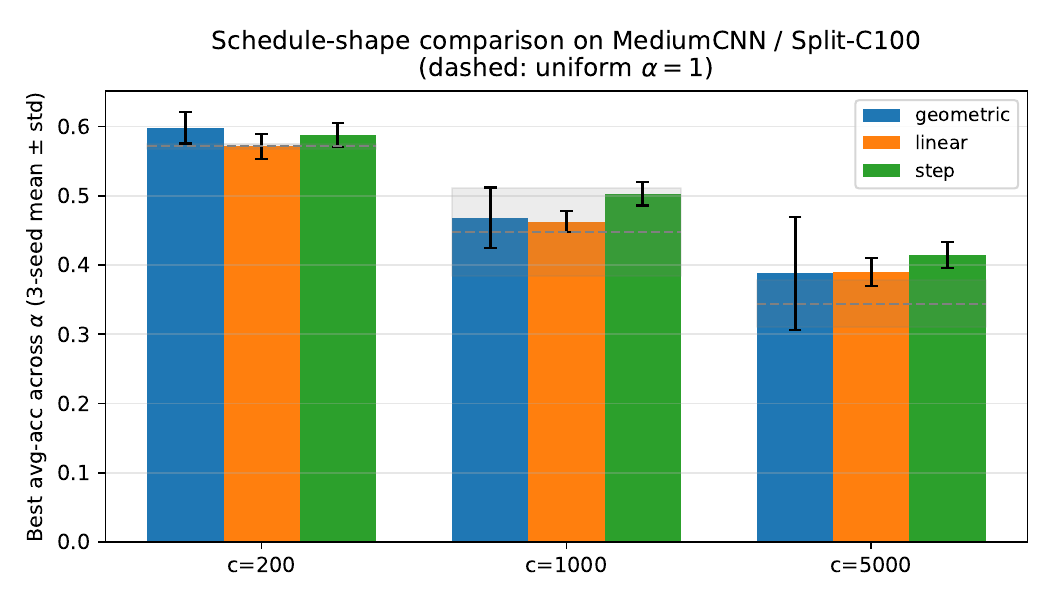}
\caption{Best avg-acc per schedule family on MediumCNN/Split-CIFAR-100 (3-seed mean ± std). Bars: best over $\alpha$ within each family. Dashed line and shaded band: uniform EWC ($\alpha=1$). All three monotone-decreasing schedules consistently beat uniform; the differences between schedule shapes are within or near the seed-to-seed std bands. Linear and step are not run at $c=20{,}000$ since the rank ordering is already clear at smaller $c$.}
\label{fig:sched-compare}
\end{figure}

\section{Full proof of Proposition~\ref{prop:fisher-gap}}
\label{app:fisher-gap-proof}

Let $A\in\R^{d\times d}$ be positive semidefinite with eigenvalues $\lambda_1\ge\dots\ge\lambda_d\ge 0$.

\paragraph{Left inequality.}
$\sum_p A_{pp}=\tr(A)=\sum_i\lambda_i$. The arithmetic mean of non-negative reals is at most the maximum, so $\tfrac{1}{d}\sum_p A_{pp} = \tfrac{1}{d}\sum_i\lambda_i \le \lambda_{\max}(A)$, with equality iff all $\lambda_i$ are equal, \textit{i.e.},\ $A=\sigma I$.

\paragraph{Right inequality.}
All eigenvalues being non-negative implies $\sum_i\lambda_i \ge \lambda_{\max}(A)$, hence $\tfrac{1}{d}\sum_p A_{pp} \ge \lambda_{\max}(A)/d$, which rearranges to $\lambda_{\max}(A)/(\tfrac{1}{d}\sum_p A_{pp}) \le d$. Equality holds iff exactly one eigenvalue is nonzero, \textit{i.e.},\ $\mathrm{rank}(A)\le 1$.

\paragraph{Intermediate values.}
Take $A_t=tI+(1-t)\mathbf{1}\mathbf{1}^\top$ with $t\in[0,1]$. Every diagonal entry equals $1$, so $\tfrac{1}{d}\sum_p (A_t)_{pp}=1$. The eigenvalues of $A_t$ are $t+(1-t)d$ (once, with eigenvector $\mathbf{1}/\sqrt{d}$) and $t$ (with multiplicity $d-1$, on the orthogonal complement). Thus $\lambda_{\max}(A_t)/(\tfrac{1}{d}\sum_p (A_t)_{pp}) = t+(1-t)d$, which is continuous in $t$ and sweeps $[1,d]$ as $t$ runs from $1$ down to $0$. \qed

\section{Direct test of the prescription on EWC}
\label{app:measured-ewc}

Theorem~\ref{thm:optimal} prescribes $\lambda_\ell\propto s_\ell$. The main-text experiments use a one-parameter geometric proxy. For completeness we also tested the literal prescription on EWC, using the per-layer $s_\ell$ measured on the ImageNet-pretrained checkpoint (Appendix~\ref{app:hessian}) directly as the schedule. Three seeds, four $c$ values, ResNet-18 and ResNet-50 on Split-CIFAR-100.

\begin{table}[h]
\centering
\footnotesize
\caption{Direct measured-$s$ schedule vs.\ uniform EWC on Split-CIFAR-100 (3 seeds, mean$\pm$std). Best $c$ per backbone in bold.}
\label{tab:measured-ewc}
\begin{tabular}{lcc}
\toprule
$c$ & uniform ($\alpha=1$) & measured-$s$ schedule \\
\midrule
\multicolumn{3}{l}{\emph{ResNet-50 / Split-CIFAR-100}} \\
50    & --- & $0.330\pm 0.117$ \\
200   & $\mathbf{0.453\pm 0.042}$ & $0.312\pm 0.058$ \\
1000  & $0.407\pm 0.036$ & $0.278\pm 0.054$ \\
5000  & $0.366\pm 0.013$ & $0.195\pm 0.027$ \\
\midrule
\multicolumn{3}{l}{\emph{ResNet-18 / Split-CIFAR-100}} \\
50    & --- & $0.510\pm 0.051$ \\
200   & $\mathbf{0.555\pm 0.009}$ & $0.498\pm 0.009$ \\
1000  & $0.515\pm 0.022$ & $0.409\pm 0.051$ \\
5000  & $0.415\pm 0.076$ & $0.249\pm 0.048$ \\
\bottomrule
\end{tabular}
\end{table}

The measured-$s$ schedule \emph{underperforms} uniform on both backbones, by 4--14\% absolute. This is a real negative result that the framework does not, by itself, predict away. We see two factors that plausibly contribute. The first is dynamic range, in that on ResNet-50 the measured ratio $s_\mathrm{layer1}/s_\mathrm{fc}$ is $134$, so the literal $\lambda_\ell\propto s_\ell$ schedule sets $\lambda_\mathrm{layer1}/\lambda_\mathrm{fc}\approx 134$, which over-protects \texttt{layer1} in absolute terms and starves the head of plasticity. The second is per-layer measurement noise, since each $s_\ell$ is estimated from a single calibration batch and the literal schedule applies the raw ratios with no smoothing across layers. Trajectory drift away from the pretraining checkpoint, where the ratios at $\theta^\star$ need not hold along the optimizer path, is a third candidate; the re-measurement experiment below isolates it. The one-parameter geometric proxy used in the main text trades exact-ratio matching for a regularising prior across $\alpha$, and that single-slope smoothing is what we believe carries the gain.

The framework's prediction is therefore narrower than ``apply $\lambda_\ell\propto s_\ell$ literally and you win.'' A correct reading is that the optimum within the scalar-per-layer family lies in the direction of $s_\ell$, and the geometric one-parameter approximation is a robust way to sweep that direction without committing to exact per-layer ratios that depend on noisy per-batch curvature estimates. The schedule-shape comparison in Appendix~\ref{app:sched-compare} reinforces this, showing that the broad family of monotone-decreasing schedules consistently beats uniform.

\paragraph{Interpolated schedules.}
The dynamic-range and measurement-noise factors suggest smoothing the literal schedule rather than abandoning it. We test the one-parameter family $\lambda_\ell \propto (s_\ell/s_{\mathrm{med}})^{\beta}$, where $s_{\mathrm{med}}$ is the median of the measured $s_\ell$: $\beta=0$ recovers uniform EWC, $\beta=1$ recovers the literal measured schedule above, and intermediate $\beta$ compresses the per-layer ratios toward one. We sweep $\beta\in\{0, 0.25, 0.5, 0.75, 1\}$ and $c\in\{50, 200, 1000, 5000\}$ with 3 seeds on Split-CIFAR-100. Figure~\ref{fig:reb-beta} reports avg-acc at the best $c$ per cell. ResNet-18 has an interior optimum at $\beta^{*}=0.5$ ($0.590\pm 0.014$ vs $0.557\pm 0.021$ for uniform), so a measured-$s$ schedule beats uniform on the very backbone where the one-parameter geometric prior gains nothing (Appendix~\ref{app:aewc-extra}). ResNet-50 has a boundary optimum at $\beta=0$ and decreases monotonically in $\beta$; its extreme dynamic range ($\kappa=195.3$, with the non-monotone \texttt{layer1} spike) defeats power scaling. Forgetting decreases monotonically in $\beta$ on both backbones, the stability-plasticity direction the framework predicts, and avg-acc varies smoothly in $\beta$ with no cliffs. The reading is consistent with the theory. Theorem~\ref{thm:optimal} fixes the direction of the schedule from the worst-case displacement. The average-case optimum $\lambda_\ell\propto\tr(H^{(\ell,\ell)})/d_\ell$ (Appendix~\ref{app:thm-optimal}) keeps the same ordering while compressing the dynamic range (Appendix~\ref{app:trace-vs-op}). The exponent $\beta<1$ plays the same compressing role and acts as shrinkage against per-batch measurement noise.

\begin{figure}[h]
\centering
\includegraphics[width=0.49\linewidth]{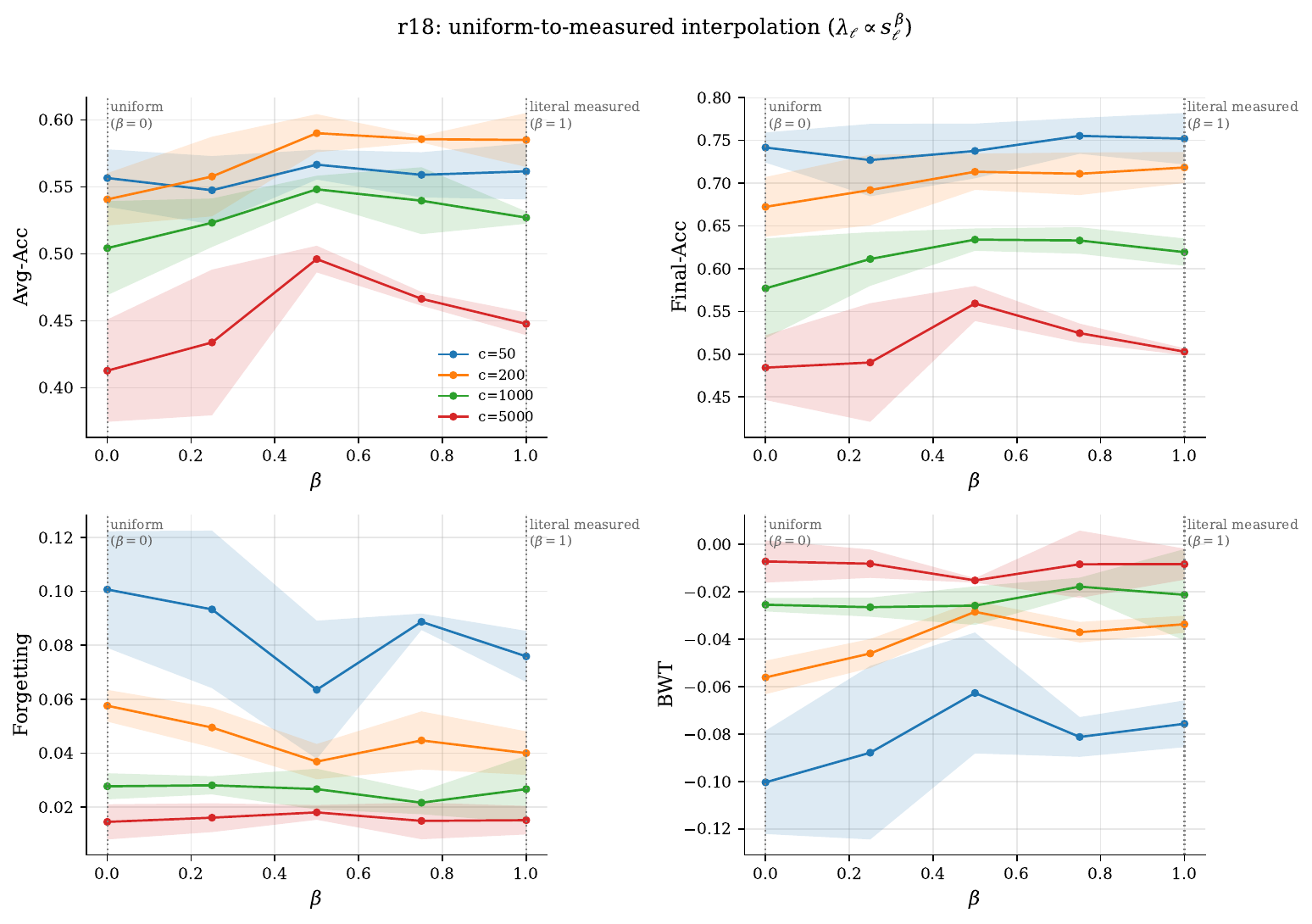}\hfill
\includegraphics[width=0.49\linewidth]{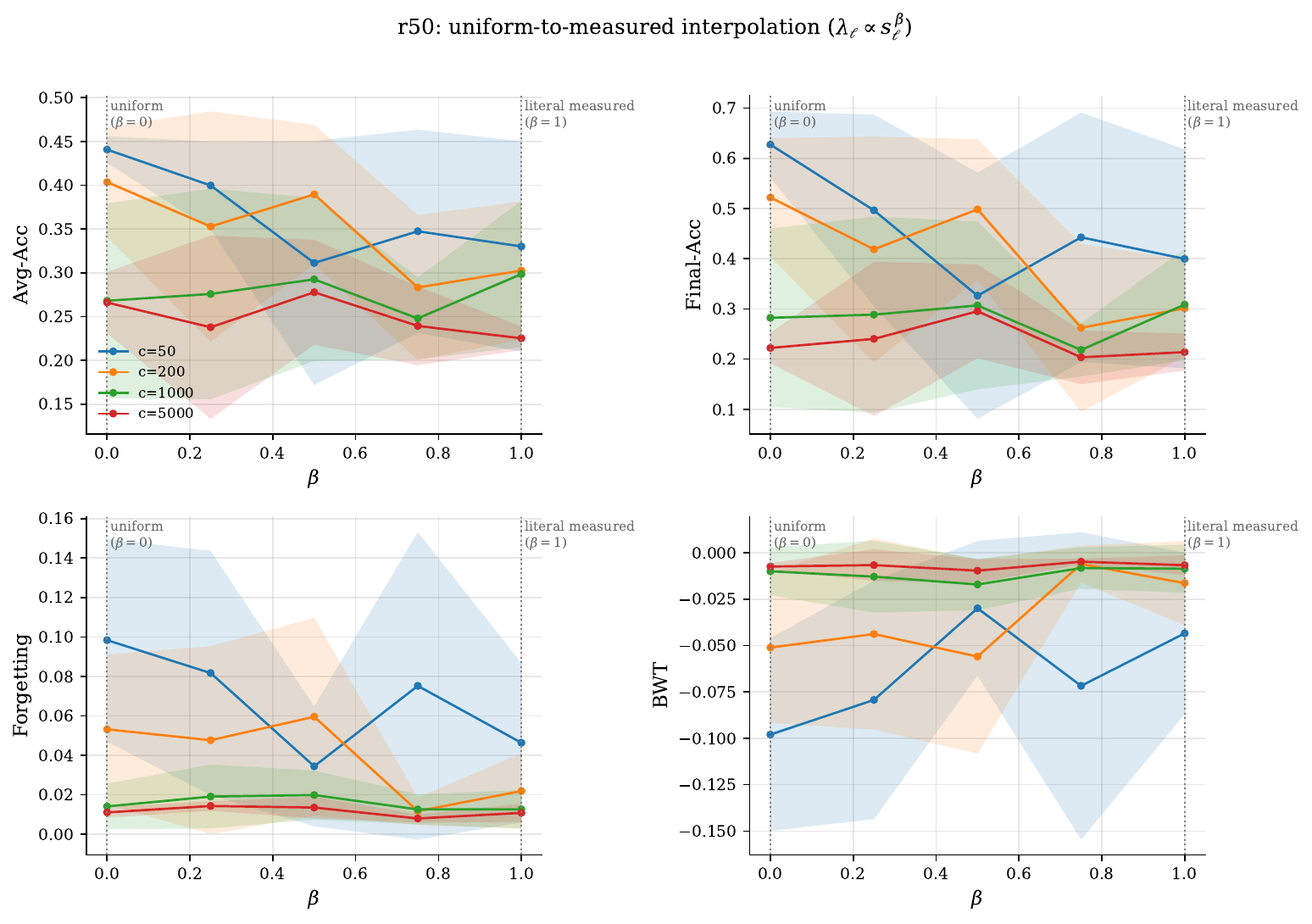}
\caption{The interpolated family $\lambda_\ell\propto(s_\ell/s_{\mathrm{med}})^{\beta}$ on Split-CIFAR-100 (left: ResNet-18, right: ResNet-50; 3 seeds, mean$\pm$std, one curve per $c$). At the best $c$ per cell, ResNet-18 peaks at $\beta^{*}=0.5$ ($0.590\pm0.014$ vs $0.557\pm0.021$ uniform); ResNet-50 is best at $\beta=0$ ($0.441\pm0.015$) and decreases monotonically to $0.330\pm0.120$ at $\beta=1$.}
\label{fig:reb-beta}
\end{figure}

\paragraph{Per-task re-measurement.}
The drift factor can be isolated directly. We re-measure $s_\ell$ at every task boundary and rebuild the schedule with the interpolated family above, in two arms: online $\beta=1$ (literal) and online $\beta=0.5$ (smoothed). Same benchmark, seeds, and $c$ grid; every measurement averages three calibration batches. Table~\ref{tab:reb-online} reports avg-acc at the best $c$ per arm. The online-smoothed arm is the strongest ResNet-18 arm in this comparison and halves forgetting relative to uniform ($0.030$ vs $0.062$). On ResNet-50, re-measurement repairs the static schedule's failure back to uniform parity ($0.462\pm 0.056$ vs $0.453\pm 0.034$). The static baseline on ResNet-18 also lands above the single-calibration-batch schedule of Table~\ref{tab:measured-ewc} ($0.575$ vs $0.510$), consistent with the measurement-noise factor above. The overhead is measured rather than estimated: $216$ HVPs per boundary ($12$ power iterations $\times$ $6$ blocks $\times$ $3$ calibration batches), $\approx$1.5\,s per boundary on ResNet-18 ($\sim$15\% of training wall-clock) and $\approx$3.6\,s on ResNet-50. Together with the $\beta$ sweep, this locates the causes of the negative result above: drift is real but secondary, and magnitude smoothing is primary.

\begin{table}[h]
\centering
\footnotesize
\caption{Per-task re-measured schedules on Split-CIFAR-100: avg-acc at the best $c$ per arm (3 seeds, mean$\pm$std). Best arm per backbone in bold.}
\label{tab:reb-online}
\begin{tabular}{lcc}
\toprule
Arm & ResNet-18 & ResNet-50 \\
\midrule
uniform ($\alpha=1$) & $0.555\pm 0.008$ & $0.453\pm 0.034$ \\
static measured & $0.575\pm 0.009$ & $0.330\pm 0.095$ \\
online $\beta=1$ & $0.570\pm 0.011$ & $\mathbf{0.462\pm 0.056}$ \\
online $\beta=0.5$ & $\mathbf{0.589\pm 0.008}$ & $0.460\pm 0.033$ \\
\bottomrule
\end{tabular}
\end{table}

\section{Forgetting swing vs measured $\kappa$}
\label{app:kappa-swing}

Theorem~\ref{thm:suboptimality} predicts that the regret of layer-uniform regularization grows with the layer condition number $\kappa$. We give a coarse quantitative test by plotting, for each EWC backbone, the \emph{forgetting swing}: the gap between forgetting at the most-deep-protective $\alpha=2^{-2}$ and at the least-deep-protective $\alpha=2^3$, taken as the maximum across $c$ values. This is the empirical analog of the regret bound: the more sensitive the deep layers are (high $\kappa$), the more forgetting one pays for under-protecting them.

\begin{figure}[h]
\centering
\includegraphics[width=0.7\linewidth]{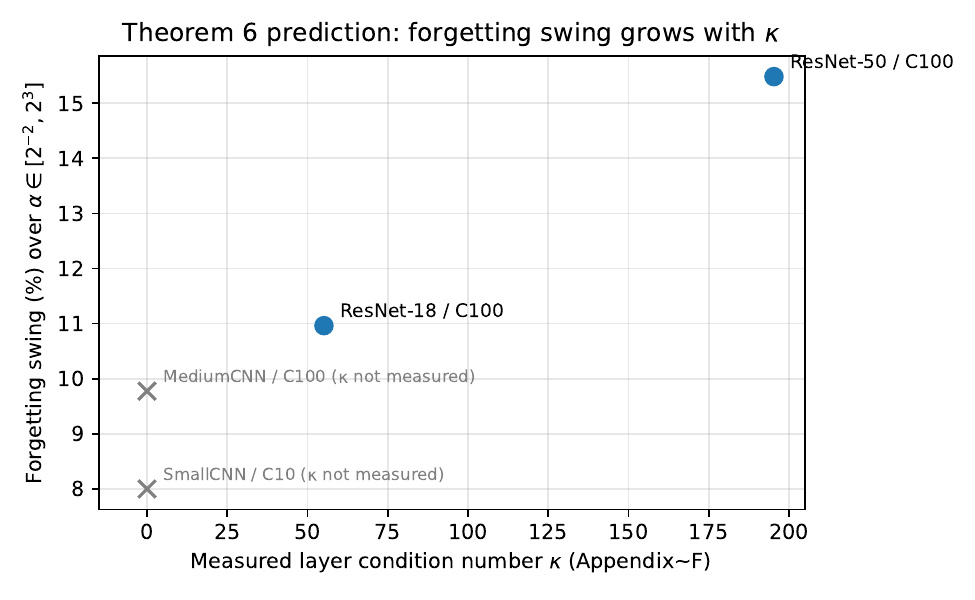}
\caption{Per-backbone forgetting swing on Split-CIFAR-10/CIFAR-100 (forg at $\alpha=2^3$ minus forg at $\alpha=2^{-2}$, max across $c$) vs measured $\kappa$. Backbones with measured $\kappa$ (ResNet-18, ResNet-50) are plotted at their true $\kappa$; SmallCNN/MediumCNN are trained from scratch and are placed on the left margin since their $s_\ell$ profile was not measured. The trend is monotone in $\kappa$ on the two pretrained backbones and consistent in magnitude with the bound's $(\kappa-1)/\kappa$ scaling.}
\label{fig:kappa-swing}
\end{figure}

\section{Trace proxy vs operator-norm sensitivity}
\label{app:trace-vs-op}

Theorem~\ref{thm:optimal} sets $\lambda_\ell\propto s_\ell$, the top Hessian eigenvalue, because forgetting is dominated by the worst displacement direction. If one instead averages forgetting over isotropic random displacement directions, the optimum becomes $\lambda_\ell\propto \tr(H^{(\ell,\ell)})/d_\ell$, the average eigenvalue per layer. The two summaries can be very different (Proposition~\ref{prop:fisher-gap}), and SGD's actual trajectories preferentially follow top eigendirections~\cite{jastrzebski2019sgd,zhu2019anisotropic} rather than isotropic ones, so $s_\ell$ is the right summary for the regret bound. The trace summary is nevertheless useful as a cheap proxy, costing one Hutchinson HVP per layer instead of $10$--$20$ power iterations. Figure~\ref{fig:trace-vs-op} plots both quantities side by side on ResNet-50 and ViT-B/16, each normalised by its own median for visual comparability.

\begin{figure}[h]
\centering
\includegraphics[width=\linewidth]{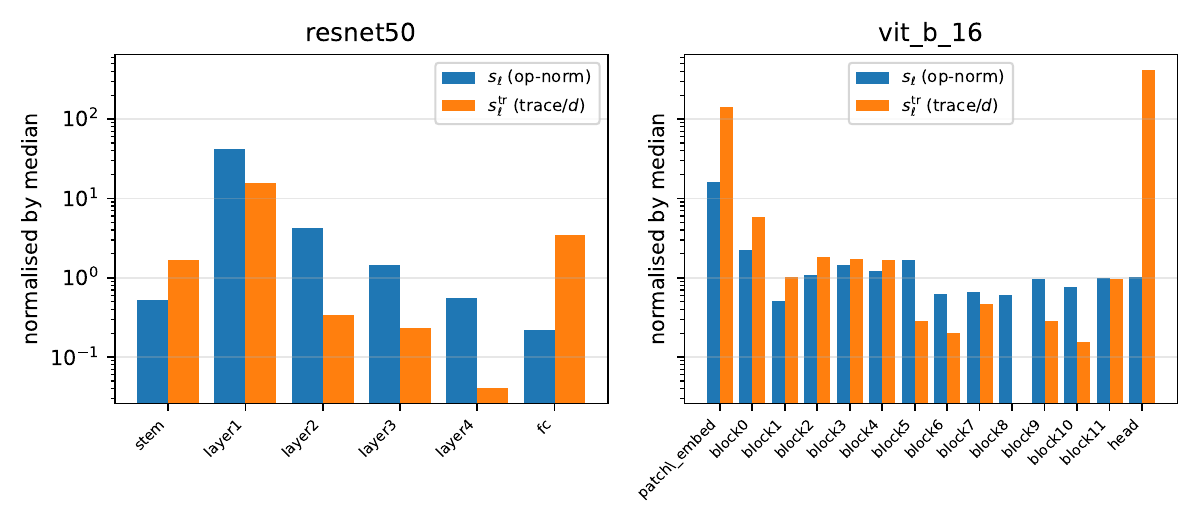}
\caption{Per-layer operator-norm $s_\ell$ (blue) and trace-divided-by-width $s_\ell^{\mathrm{tr}}=\tr(H^{(\ell,\ell)})/d_\ell$ (orange), each normalised by its median. Both signals identify the same dominant block (\texttt{layer1} on ResNet-50, \texttt{patch\_embed} on ViT-B/16), but the operator norm has a sharper peak. The trace proxy is a reasonable cheap surrogate for the qualitative ordering, but understates the dynamic range of $\kappa$.}
\label{fig:trace-vs-op}
\end{figure}

The takeaway: the trace proxy preserves the qualitative ordering and identifies the same most-sensitive block, but compresses the spread by a factor of several. For a practitioner choosing a depth-weighted schedule, the cheap trace estimate is a good first sweep direction; for the regret bound in Theorem~\ref{thm:suboptimality}, which depends on the operator-norm $\kappa$, the power-iteration estimate remains the right quantity.

\section{LIH on the BERT-base language backbone}
\label{app:nlp}

The LIH measurements in the main text and Appendix~\ref{app:hessian} are on six standard vision backbones (ResNet-18/34/50/101, ViT-B/16, ViT-L/16). The framework is stated for any layer-partitioned neural network, but a reasonable concern is whether the shallow-dominated $s_\ell$ pattern transfers beyond computer vision. We test this on BERT-base~\cite{devlin2019bert}: same per-layer power-iteration procedure as Appendix~\ref{app:hessian}, calibration batch of $32$ random token sequences (max length $64$, CLS/SEP boundary tokens preserved), $10$-class linear head on the CLS embedding, $15$ power-iteration steps, $3$ calibration seeds.

Figure~\ref{fig:bert-sell} reports the per-layer profile. The token-embedding block dominates with $s_\ell = 70.4 \pm 0.0$, the head sits at $20.8$, and the twelve transformer blocks span $s_\ell \in [7.7, 26.3]$. The measured layer condition number is $\kappa = 9.15$, modest compared to the $20$--$200$ range observed on vision backbones but still strictly above $1$. The qualitative pattern matches the LIH: a small number of layers (here, the embedding) dominate the loss curvature, with the rest of the architecture spread across roughly an order of magnitude. The framework's prescription therefore extends to NLP backbones in principle, with the specific schedule depending on the per-layer profile of the chosen model.

\begin{figure}[h]
\centering
\includegraphics[width=0.7\linewidth]{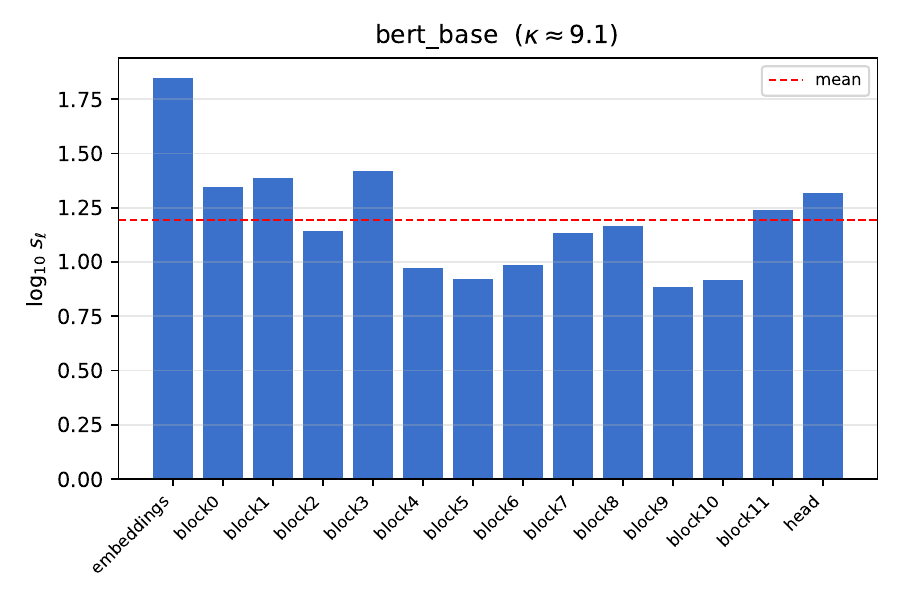}
\caption{Per-layer $\log_{10} s_\ell$ on BERT-base, evaluated on a random-token calibration batch with a $10$-class linear head on CLS. Embeddings dominate ($s_{\text{embed}} = 70$), with the twelve transformer blocks in $[8, 26]$ and the head at $s_{\text{head}} = 21$. Measured $\kappa = 9.15$.}
\label{fig:bert-sell}
\end{figure}

Two caveats. First, $\kappa$ on BERT is smaller than on vision backbones; the rule-of-thumb gain from a depth-weighted regularizer should be correspondingly smaller per Theorem~\ref{thm:suboptimality}. Second, our calibration batch is random tokens rather than a real text-classification benchmark, so the absolute magnitudes will shift on a downstream task. The per-layer ordering is what the framework consumes, and it is consistent with the vision-backbone pattern: a small number of high-curvature layers, here the token embedding, dominate the loss.

\paragraph{A text continual-learning benchmark.}
The measurement above is a curvature profile, not a task-performance result. To close the loop, we ran an end-to-end benchmark on the five-dataset text-classification sequence of Huang et al.~\cite{huang2021idbr}: AG News, Yelp Review Full, Amazon Review Full, Yahoo Answers, and DBpedia, in that fixed order. The backbone is BERT-base with a separate classification head per task; EWC regularizes the shared encoder parameters. We use the reduced setting of IDBR~\cite{huang2021idbr}, with $2000$ training and up to $1000$ test examples per class, $3$ seeds, and $c\in\{100, 1000, 10000\}$. Table~\ref{tab:bert-cl} reports sequential fine-tuning and both EWC arms at the strongest regularization strength.

\begin{table}[h]
\centering
\footnotesize
\caption{Five-dataset text continual learning with BERT-base (3 seeds, mean$\pm$std). SeqFT is sequential fine-tuning without regularization; both EWC arms are shown at $c=10000$.}
\label{tab:bert-cl}
\begin{tabular}{lcc}
\toprule
Method & Avg-acc & Forgetting \\
\midrule
SeqFT & $0.624\pm 0.010$ & $0.158\pm 0.011$ \\
Uniform EWC ($c=10000$) & $0.728\pm 0.002$ & $0.027\pm 0.006$ \\
Measured-$s$ EWC ($c=10000$) & $0.722\pm 0.005$ & $0.035\pm 0.004$ \\
\bottomrule
\end{tabular}
\end{table}

EWC itself transfers to language: uniform EWC gains $10$ percentage points of avg-acc over sequential fine-tuning and cuts forgetting from $0.158$ to $0.027$. The measured-$s$ schedule matches uniform within noise. The difference is between $-0.6$ and $+0.3$ percentage points at every $c$ and is never significant ($p\ge 0.09$). This null is the outcome Theorem~\ref{thm:suboptimality} predicts. The measured $\kappa=9.15$ on BERT-base is the smallest backbone-level value we observe, and the regret bound in Theorem~\ref{thm:suboptimality} is monotone non-decreasing in $\kappa$, so the predicted layer-adaptive gain is smallest exactly on this backbone. The benchmark therefore confirms the scope condition: layer-adaptive schedules pay off where the measured $\kappa$ is large, and the same measurement predicts in advance where they do not.

\section{Per-task $s_\ell$ variation and ordering stability}
\label{app:pertask}

The Layer-Importance Hypothesis (Figure~\ref{fig:lih}) is implicitly a claim about the architecture: $s_\ell$ varies dramatically across layers, and this variation should be a property of the network rather than of any particular calibration batch. We test this directly on ResNet-50 by drawing the calibration batch from each Split-CIFAR-100 task in turn (10-class subsets, $32$ samples per batch) and measuring per-layer $s_\ell$ at each. Figure~\ref{fig:pertask-sell} shows the resulting trajectories.

\begin{figure}[h]
\centering
\includegraphics[width=0.85\linewidth]{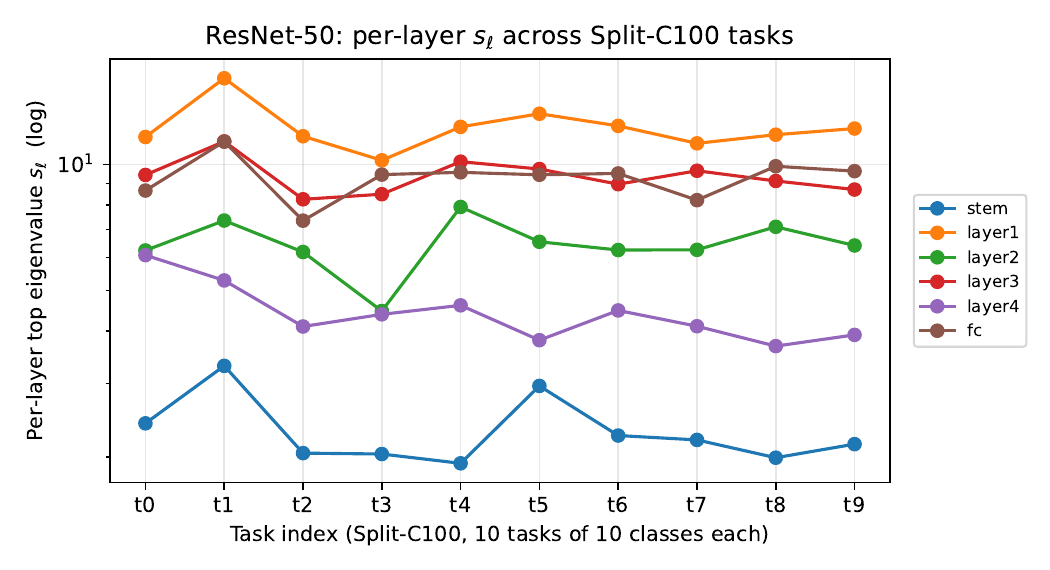}
\caption{Per-layer $s_\ell$ on ResNet-50 measured with calibration batches drawn from each of the $10$ Split-CIFAR-100 tasks separately. Each line is one architectural block; the $x$-axis is task index. The per-block ordering (\texttt{layer1} on top, then \texttt{layer3} and \texttt{fc}, then \texttt{layer2} and \texttt{layer4}, with \texttt{stem} at the bottom) is preserved across all $10$ tasks. Absolute magnitudes vary by a factor of $\sim 2$ within each block, but the rank is invariant.}
\label{fig:pertask-sell}
\end{figure}

Two observations. First, the per-layer rank is stable: \texttt{layer1} dominates on every task ($s_\ell = 12.3 \pm 1.5$ across tasks), with \texttt{layer3} and \texttt{fc} consistently second-tier and \texttt{stem} consistently lowest. The shallow-conv-dominated pattern from Figure~\ref{fig:lih} therefore does not depend on a specific calibration batch and reads through to every task's data slice. Second, absolute $s_\ell$ values vary within a factor of $\sim 2$ across tasks, and the per-task $\kappa$ values cluster in $4.5$--$6.4$. This is much smaller than the $\kappa = 195$ reported on the same backbone with a 100-class head in Appendix~\ref{app:hessian}, because here each task has only $10$ classes and the cross-entropy loss has correspondingly less spread. The framework's prescription is invariant to the absolute scale of $\kappa$ (it only uses the per-layer ratios), so this scale shift does not affect the rule of thumb; it does say that quoting a single $\kappa$ for a backbone hides real data dependence in the absolute magnitudes.

\begin{table}[t]
\centering
\footnotesize
\setlength{\tabcolsep}{5pt}
\caption{Per-experiment compute summary. EWC and SLCA depth-weighted sweeps use 3 seeds where indicated. ``Wallclock per job'' is the typical end-to-end time on a single allocation; ``Total'' aggregates across the full grid.}
\label{tab:compute}
\begin{tabular}{lrrr}
\toprule
Experiment & \# jobs & Wallclock/job & GPU-h (approx.) \\
\midrule
EWC sweep, SmallCNN/CIFAR-10 (3 seeds, 6$\times$4 grid) & 72 & 5--10\,min & 10 \\
EWC sweep, MediumCNN/CIFAR-100 (3 seeds, 6$\times$4 grid) & 72 & 10--20\,min & 20 \\
EWC sweep, ResNet-18/CIFAR-100 (3 seeds, 6$\times$4 grid) & 72 & 20--40\,min & 40 \\
EWC sweep, ResNet-50/CIFAR-100 (3 seeds, 6$\times$4 grid) & 72 & 30--60\,min & 60 \\
EWC measured-$s$, ResNet-18/CIFAR-100 (3 seeds, 4 $c$, Appendix~\ref{app:measured-ewc}) & 12 & 20--40\,min & 7 \\
EWC measured-$s$, ResNet-50/CIFAR-100 (3 seeds, 4 $c$, Appendix~\ref{app:measured-ewc}) & 12 & 30--60\,min & 10 \\
EWC schedule-shape, MediumCNN/CIFAR-100 (linear+step, Appendix~\ref{app:sched-compare}) & 54 & 10--20\,min & 15 \\
EWC schedule-shape, SmallCNN/CIFAR-10 (linear+step, Appendix~\ref{app:sched-compare}) & 54 & 5--10\,min & 7 \\
EWC schedule-shape, ResNet-18/CIFAR-100 (linear+step, Appendix~\ref{app:sched-compare}) & 54 & 20--40\,min & 30 \\
EWC 20-task Split-CIFAR-100 (3 seeds, 6$\times$4 grid) & 72 & 20--40\,min & 35 \\
TUNA, ImageNet-R B0-Inc20 & $\sim$24 & 60--120\,min & 36 \\
TUNA, CIFAR-100 B0-Inc10 (Appendix~\ref{app:tuna}) & $\sim$24 & 60--120\,min & 36 \\
SLCA depth-weighted, MoCoV3 (15 configs, 90 ep/task) & 15 & 240--360\,min & 75 \\
SLCA depth-weighted, ImageNet-21k (15 configs, 20 ep/task) & 15 & 25--40\,min & 8 \\
Hessian power iteration, 6 backbones (Appendix~\ref{app:hessian}) & 6 & 30--60\,min & 4 \\
Hessian extras: per-task $s_\ell$, empirical Fisher, granularity, BERT-base NLP & 4 & 20--60\,min & 3 \\
\midrule
\textbf{Total reported runs} & & & $\sim$397 \\
\bottomrule
\end{tabular}
\end{table}

\paragraph{Within-task and across-task trajectories.}
The measurements above vary the calibration task at a fixed trained checkpoint. A complementary question is how $s_\ell$ evolves along the training trajectory itself. We measure $s_\ell$ at $11$ checkpoints during task 1 of Split-CIFAR-100 ($3$ seeds) on pretrained ResNet-50 and on from-scratch ResNet-18 and SmallCNN, and summarize rank stability as the Spearman correlation between the profile at step $t$ and the final within-task profile. Figure~\ref{fig:traj-sell} shows the trajectories. On pretrained ResNet-50 the correlation stays at or above $0.70$ at every checkpoint: the ranking measured at $\theta^\star$ does not go stale within a task, only the magnitudes move. The from-scratch runs are rank-unstable before roughly step $200$ and stabilize thereafter. Mid-training, the top-curvature estimate on ResNet-18 and ResNet-50 can be negative; the trajectory measurements use a signed Rayleigh-quotient estimator, and we report magnitudes with a sign flag.

\begin{figure}[h]
\centering
\includegraphics[width=\linewidth]{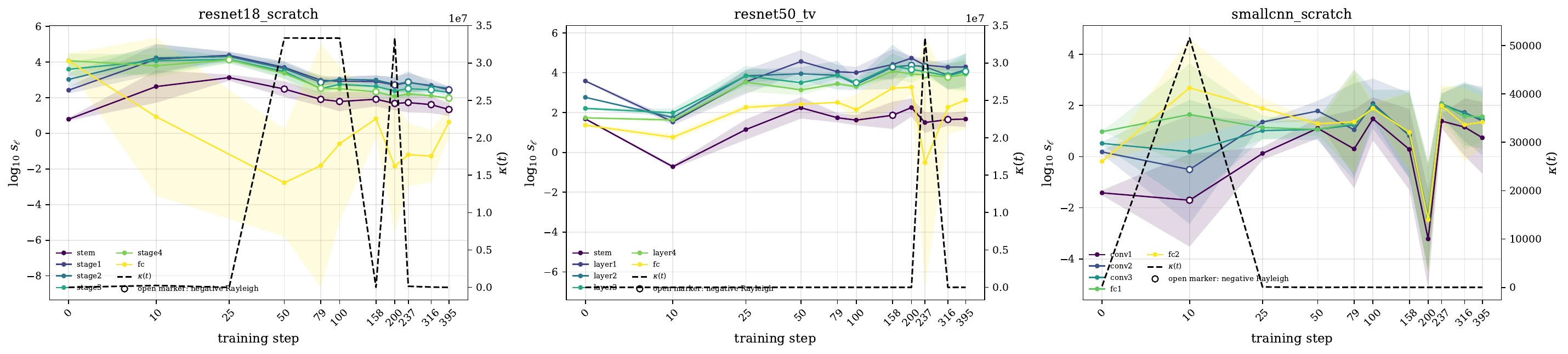}
\caption{Per-layer $\log_{10}|s_\ell|$ at 11 checkpoints during task 1 of Split-CIFAR-100 (one panel per backbone; mean over 3 seeds with seed band; open markers flag negative signed-Rayleigh curvature; dashed line: $\kappa(t)$).}
\label{fig:traj-sell}
\end{figure}

Across task boundaries, we track a from-scratch ResNet-50 (CIFAR variant; blocks \texttt{stage1}--\texttt{stage3} and \texttt{fc}) trained with EWC ($c=200$) on Split-CIFAR-100, snapshotting the profile after tasks 1, 2, 5, and 10. At initialization the profile is fc-heavy ($s_\ell = 32{,}600$ on \texttt{fc} vs $11{,}500$ on \texttt{stage3}). After task 1 it has reordered to stage3-heavy ($21.2$ vs $1.9$). From task 2 onward the ranking is stable while the magnitudes grow: \texttt{stage3} passes through $478$, $1{,}768$, and $3{,}201$, and \texttt{fc} through $20$, $81$, and $126$, after tasks 2, 5, and 10 respectively. The two views agree. Profiles anchored at a trained $\theta^\star$ are rank-stable, both within a task and across tasks from task 2 on, whereas profiles taken at initialization or early in from-scratch training do not predict the trained ordering. This supports re-measuring the schedule at task boundaries rather than fixing it at initialization.

\section{Comparison with experience replay}
\label{app:replay}

The main text treats replay as an orthogonal family (Section~\ref{sec:related}). Here we quantify the comparison directly. We run experience replay (ER) on Split-CIFAR-100 with the same task split, backbones, and training protocol as the EWC experiments (Section~\ref{sec:empirical}): class-balanced index replay over past-task data, memory budgets $M\in\{500, 2000, 6200\}$ stored examples, replay batch size equal to the current-task batch, 3 seeds. The largest budget is chosen so that the buffer roughly matches the extra storage EWC itself carries on MediumCNN, anchor parameters plus Fisher values, \textit{i.e.},\ $2\times$ the parameter count at 4 bytes each.

\begin{table}[h]
\centering
\footnotesize
\caption{Average accuracy on Split-CIFAR-100 (3 seeds, mean$\pm$std): ER at three memory budgets, the EWC arms at their best configuration, and the composition of the two families. Best $c$ is $200$ throughout; the best geometric $\alpha$ is $2$ on MediumCNN and ResNet-50 and $4$ on ResNet-18.}
\label{tab:reb-replay}
\begin{tabular}{lccc}
\toprule
Method & MediumCNN & ResNet-18 & ResNet-50 \\
\midrule
ER ($M=500$)  & $0.441\pm 0.012$ & $0.658\pm 0.009$ & $0.620\pm 0.027$ \\
ER ($M=2000$) & $0.443\pm 0.014$ & $0.725\pm 0.012$ & $0.684\pm 0.031$ \\
ER ($M=6200$) & $0.493\pm 0.002$ & $0.775\pm 0.012$ & $0.720\pm 0.011$ \\
EWC uniform (best $c$) & $0.572\pm 0.004$ & $0.555\pm 0.009$ & $0.453\pm 0.042$ \\
EWC geometric (best $\alpha$, $c$) & $0.598\pm 0.023$ & $0.542\pm 0.038$ & $0.431\pm 0.023$ \\
ER+EWC ($M=500$, $c=200$, $\alpha=2$) & $0.574\pm 0.005$ & $0.593\pm 0.044$ & $0.505\pm 0.017$ \\
\bottomrule
\end{tabular}
\end{table}

Three observations. First, ER dominates both EWC arms on the ResNets at every budget, consistent with prior reports that replay is a strong baseline on deep networks~\cite{buzzega2020dark,prabhu2020gdumb}. Second, the ordering reverses on the small from-scratch CNN: every EWC arm beats ER at every budget, and the geometric schedule is the best arm overall. Third, the two families compose. The ER+EWC row combines a small buffer ($M=500$) with the geometric schedule at $\alpha=2$ and $c=200$; on MediumCNN this beats ER alone at every budget. Replay also carries a cost the table does not show: the buffer stores raw past-task data, which is unavailable under privacy or data-retention constraints. Layer-adaptive regularization applies unchanged in that regime and, per the last row, remains useful when storage is allowed.

\section{Broader impacts}
\label{app:impacts}

The paper's contribution is theoretical and diagnostic. The most direct positive impact is efficiency: a layer-adaptive regularizer that recovers the same continual-learning quality with a smaller forgetting budget reduces the compute needed to keep deployed models current with new data, an effect that compounds with the total number of model updates a system performs over its lifetime. The associated diagnostic ($\kappa$ measurement) is a single-number test for whether layer-adaptive regularization is worth using on a given backbone, which spares practitioners from running a full sweep when uniform regularization is already adequate.

The paper also has a dual-use facet that we want to acknowledge directly. The per-layer top Hessian eigenvalue $s_\ell$ that our framework prescribes for protective regularization is the \emph{same} per-layer sensitivity quantity that the adversarial weight-attack literature uses to identify which layers are most worth attacking with bit-flip, weight-perturbation, or trojan-implant procedures~\cite{rakin2019bfa,yao2020mbda,chen2021proflip}. Our work does not introduce a new attack capability (the bit-flip literature already operates on the same quantity), but reframing $s_\ell$ as a CL hyperparameter could in principle make the per-layer sensitivity profile a more visible target for adversaries planning weight-space attacks on deployed models. We mitigate this by working only with publicly available, non-safety-critical benchmarks (Split-CIFAR-10/100, ImageNet-R) and by releasing no new model checkpoints. We recommend that practitioners who measure $s_\ell$ on production models treat the resulting per-layer profile with the same care they would apply to any other model-internal artefact that informs an attack surface.

A second consideration is more diffuse. Improved continual learning makes long-running personalized systems (recommenders, assistants, on-device adapters) easier to keep current. This is a property of better CL in general rather than a specific contribution of this paper, but it interacts with privacy in the usual way: a system that retains and integrates information well across tasks also retains and integrates information well across users and time, which sharpens the existing tension between personalization quality and data minimization. Any deployment of CL methods on user data should pair the stability gains we discuss with explicit data-retention and consent policies.

\section{Compute resources}
\label{app:compute}

All experiments ran on an internal HPC cluster with NVIDIA GPUs (B200, H200, H100, A100-80GB, A100-40GB, RTXA6000) inside an NVIDIA PyTorch 25.01 enroot container, scheduled by SLURM. Per-job resources requested: 1 GPU, 4--8 CPUs, 32--64\,GB RAM, 3--10\,h walltime depending on workload. Approximate wallclock and total compute per experiment block, including the appendix-only sweeps:

The full research effort (preliminary sweeps, debugging runs, abandoned configurations not reported here) consumed roughly $2$--$3\times$ the total in Table~\ref{tab:compute}, \textit{i.e.},\ on the order of $800$--$1200$ GPU-hours. A later round of experiments, including the replay comparison in Appendix~\ref{app:replay}, added roughly $300$ further GPU-jobs of 0.1--2\,h each.

\FloatBarrier
\clearpage
\section*{NeurIPS Paper Checklist}

\begin{enumerate}

\item {\bf Claims}
    \item[] Question: Do the main claims made in the abstract and introduction accurately reflect the paper's contributions and scope?
    \item[] Answer: \answerYes{}
    \item[] Justification: The abstract and Section~\ref{sec:intro} state the four main results (Theorems~\ref{thm:decomp}, \ref{thm:suboptimality}, \ref{thm:optimal} and Proposition~\ref{prop:fisher-gap}). They also state how two existing layer-adaptive CL methods (TUNA and SLCA) and our own depth-weighted EWC variant fit the scalar-per-layer family. Each claim is formalized in Section~\ref{sec:methodology} and tested in Section~\ref{sec:empirical}.
    \item[] Guidelines:
    \begin{itemize}
        \item The answer \answerNA{} means that the abstract and introduction do not include the claims made in the paper.
        \item The abstract and/or introduction should clearly state the claims made, including the contributions made in the paper and important assumptions and limitations. A \answerNo{} or \answerNA{} answer to this question will not be perceived well by the reviewers. 
        \item The claims made should match theoretical and experimental results, and reflect how much the results can be expected to generalize to other settings. 
        \item It is fine to include aspirational goals as motivation as long as it is clear that these goals are not attained by the paper. 
    \end{itemize}

\item {\bf Limitations}
    \item[] Question: Does the paper discuss the limitations of the work performed by the authors?
    \item[] Answer: \answerYes{}
    \item[] Justification: Section~\ref{sec:discussion} explicitly lists five caveats: local validity of the second-order model (trust region), block-diagonal Hessian approximation, $F=H$ identity at exponential-family optima, functional (not algebraic) SLCA reduction, and locality of $s_\ell$ measurement.
    \item[] Guidelines:
    \begin{itemize}
        \item The answer \answerNA{} means that the paper has no limitation while the answer \answerNo{} means that the paper has limitations, but those are not discussed in the paper. 
        \item The authors are encouraged to create a separate ``Limitations'' section in their paper.
        \item The paper should point out any strong assumptions and how robust the results are to violations of these assumptions (e.g., independence assumptions, noiseless settings, model well-specification, asymptotic approximations only holding locally). The authors should reflect on how these assumptions might be violated in practice and what the implications would be.
        \item The authors should reflect on the scope of the claims made, e.g., if the approach was only tested on a few datasets or with a few runs. In general, empirical results often depend on implicit assumptions, which should be articulated.
        \item The authors should reflect on the factors that influence the performance of the approach. For example, a facial recognition algorithm may perform poorly when image resolution is low or images are taken in low lighting. Or a speech-to-text system might not be used reliably to provide closed captions for online lectures because it fails to handle technical jargon.
        \item The authors should discuss the computational efficiency of the proposed algorithms and how they scale with dataset size.
        \item If applicable, the authors should discuss possible limitations of their approach to address problems of privacy and fairness.
        \item While the authors might fear that complete honesty about limitations might be used by reviewers as grounds for rejection, a worse outcome might be that reviewers discover limitations that aren't acknowledged in the paper. The authors should use their best judgment and recognize that individual actions in favor of transparency play an important role in developing norms that preserve the integrity of the community. Reviewers will be specifically instructed to not penalize honesty concerning limitations.
    \end{itemize}

\item {\bf Theory assumptions and proofs}
    \item[] Question: For each theoretical result, does the paper provide the full set of assumptions and a complete (and correct) proof?
    \item[] Answer: \answerYes{}
    \item[] Justification: All theorems, propositions, corollaries, and definitions are numbered and cross-referenced. Assumption~\ref{ass:block-diag} is stated before every result that uses it. Theorem~\ref{thm:decomp} and Proposition~\ref{prop:fisher-gap} are proved in full in the main text; Theorems~\ref{thm:suboptimality} and~\ref{thm:optimal} have main-text proof sketches with full proofs deferred to Appendices~B and~C.
    \item[] Guidelines:
    \begin{itemize}
        \item The answer \answerNA{} means that the paper does not include theoretical results. 
        \item All the theorems, formulas, and proofs in the paper should be numbered and cross-referenced.
        \item All assumptions should be clearly stated or referenced in the statement of any theorems.
        \item The proofs can either appear in the main paper or the supplemental material, but if they appear in the supplemental material, the authors are encouraged to provide a short proof sketch to provide intuition. 
        \item Inversely, any informal proof provided in the core of the paper should be complemented by formal proofs provided in appendix or supplemental material.
        \item Theorems and Lemmas that the proof relies upon should be properly referenced. 
    \end{itemize}

    \item {\bf Experimental result reproducibility}
    \item[] Question: Does the paper fully disclose all the information needed to reproduce the main experimental results of the paper to the extent that it affects the main claims and/or conclusions of the paper (regardless of whether the code and data are provided or not)?
    \item[] Answer: \answerYes{}
    \item[] Justification: Section~\ref{sec:empirical} reports the full grid for each experiment: depth-weighted EWC sweeps over $\alpha\in\{0.25,0.5,1,2,4,8\}$ and $c\in\{200,1000,5000,20000\}$ on SmallCNN (5 layers) on Split-C10, MediumCNN (6 layers) on Split-C100, and ResNet-18 on Split-C100, all over 3 seeds with 2 epochs per task; TUNA orthogonality sweeps over the same $(c,\alpha)$ parametrization on ViT-B/16 with C100 B0-Inc10 and ImNet-R B0-Inc20; SLCA depth-weighted sweeps over $c\in\{0.5,1,2\}$ and $\alpha\in\{0.85,0.92,1,1.08,1.18\}$ on ViT-B/16 + MoCoV3 and ViT-B/16 + ImNet-21k, both on Split-C100. The remaining hyperparameters (optimizer, batch size, Fisher-sample count, milestones) are listed in the companion repository.
    \item[] Guidelines:
    \begin{itemize}
        \item The answer \answerNA{} means that the paper does not include experiments.
        \item If the paper includes experiments, a \answerNo{} answer to this question will not be perceived well by the reviewers: Making the paper reproducible is important, regardless of whether the code and data are provided or not.
        \item If the contribution is a dataset and\slash or model, the authors should describe the steps taken to make their results reproducible or verifiable. 
        \item Depending on the contribution, reproducibility can be accomplished in various ways. For example, if the contribution is a novel architecture, describing the architecture fully might suffice, or if the contribution is a specific model and empirical evaluation, it may be necessary to either make it possible for others to replicate the model with the same dataset, or provide access to the model. In general. releasing code and data is often one good way to accomplish this, but reproducibility can also be provided via detailed instructions for how to replicate the results, access to a hosted model (e.g., in the case of a large language model), releasing of a model checkpoint, or other means that are appropriate to the research performed.
        \item While NeurIPS does not require releasing code, the conference does require all submissions to provide some reasonable avenue for reproducibility, which may depend on the nature of the contribution. For example
        \begin{enumerate}
            \item If the contribution is primarily a new algorithm, the paper should make it clear how to reproduce that algorithm.
            \item If the contribution is primarily a new model architecture, the paper should describe the architecture clearly and fully.
            \item If the contribution is a new model (e.g., a large language model), then there should either be a way to access this model for reproducing the results or a way to reproduce the model (e.g., with an open-source dataset or instructions for how to construct the dataset).
            \item We recognize that reproducibility may be tricky in some cases, in which case authors are welcome to describe the particular way they provide for reproducibility. In the case of closed-source models, it may be that access to the model is limited in some way (e.g., to registered users), but it should be possible for other researchers to have some path to reproducing or verifying the results.
        \end{enumerate}
    \end{itemize}

\item {\bf Open access to data and code}
    \item[] Question: Does the paper provide open access to the data and code, with sufficient instructions to faithfully reproduce the main experimental results, as described in supplemental material?
    \item[] Answer: \answerNo{}
    \item[] Justification: The main contribution is theoretical, but every empirical figure in the paper (the per-layer Hessian measurement of Figure~\ref{fig:lih}, the multi-seed depth-weighted EWC sweeps, the TUNA orthogonality sweeps, and the two SLCA depth-weighted sweeps) is produced by an anonymous companion codebase. We will release this codebase together with the camera-ready, including the launch scripts and per-task JSON outputs that generate every plot. The submission text gives the full $(c,\alpha)$ grids, backbones, datasets, and seeds needed to reproduce the runs from public sources alone.
    \item[] Guidelines:
    \begin{itemize}
        \item The answer \answerNA{} means that paper does not include experiments requiring code.
        \item Please see the NeurIPS code and data submission guidelines (\url{https://neurips.cc/public/guides/CodeSubmissionPolicy}) for more details.
        \item While we encourage the release of code and data, we understand that this might not be possible, so \answerNo{} is an acceptable answer. Papers cannot be rejected simply for not including code, unless this is central to the contribution (e.g., for a new open-source benchmark).
        \item The instructions should contain the exact command and environment needed to run to reproduce the results. See the NeurIPS code and data submission guidelines (\url{https://neurips.cc/public/guides/CodeSubmissionPolicy}) for more details.
        \item The authors should provide instructions on data access and preparation, including how to access the raw data, preprocessed data, intermediate data, and generated data, etc.
        \item The authors should provide scripts to reproduce all experimental results for the new proposed method and baselines. If only a subset of experiments are reproducible, they should state which ones are omitted from the script and why.
        \item At submission time, to preserve anonymity, the authors should release anonymized versions (if applicable).
        \item Providing as much information as possible in supplemental material (appended to the paper) is recommended, but including URLs to data and code is permitted.
    \end{itemize}

\item {\bf Experimental setting/details}
    \item[] Question: Does the paper specify all the training and test details (e.g., data splits, hyperparameters, how they were chosen, type of optimizer) necessary to understand the results?
    \item[] Answer: \answerYes{}
    \item[] Justification: Section~\ref{sec:empirical} reports architecture (SmallCNN, 5 layers; MediumCNN, 6 layers), dataset (Split-C10, 5 tasks; Split-C100, 10 tasks), epochs per task, and the full $(\alpha, c)$ grid. Full hyperparameter specifications (optimizer, batch size, Fisher-sample size) are in the companion repository.
    \item[] Guidelines:
    \begin{itemize}
        \item The answer \answerNA{} means that the paper does not include experiments.
        \item The experimental setting should be presented in the core of the paper to a level of detail that is necessary to appreciate the results and make sense of them.
        \item The full details can be provided either with the code, in appendix, or as supplemental material.
    \end{itemize}

\item {\bf Experiment statistical significance}
    \item[] Question: Does the paper report error bars suitably and correctly defined or other appropriate information about the statistical significance of the experiments?
    \item[] Answer: \answerYes{}
    \item[] Justification: The depth-weighted EWC sweeps report mean curves with shaded $\pm 1$ std bands across 3 seeds (216 runs total). The TUNA and SLCA sweeps are single-seed grid sweeps following the protocol of the original methods. Each of the two SLCA $(c,\alpha)$ grids spans 15 configs, and the trend is read across the grid rather than from per-cell error bars.
    \item[] Guidelines:
    \begin{itemize}
        \item The answer \answerNA{} means that the paper does not include experiments.
        \item The authors should answer \answerYes{} if the results are accompanied by error bars, confidence intervals, or statistical significance tests, at least for the experiments that support the main claims of the paper.
        \item The factors of variability that the error bars are capturing should be clearly stated (for example, train/test split, initialization, random drawing of some parameter, or overall run with given experimental conditions).
        \item The method for calculating the error bars should be explained (closed form formula, call to a library function, bootstrap, etc.)
        \item The assumptions made should be given (e.g., Normally distributed errors).
        \item It should be clear whether the error bar is the standard deviation or the standard error of the mean.
        \item It is OK to report 1-sigma error bars, but one should state it. The authors should preferably report a 2-sigma error bar than state that they have a 96\% CI, if the hypothesis of Normality of errors is not verified.
        \item For asymmetric distributions, the authors should be careful not to show in tables or figures symmetric error bars that would yield results that are out of range (e.g., negative error rates).
        \item If error bars are reported in tables or plots, the authors should explain in the text how they were calculated and reference the corresponding figures or tables in the text.
    \end{itemize}

\item {\bf Experiments compute resources}
    \item[] Question: For each experiment, does the paper provide sufficient information on the computer resources (type of compute workers, memory, time of execution) needed to reproduce the experiments?
    \item[] Answer: \answerYes{}
    \item[] Justification: Appendix~\ref{app:compute} (Table~\ref{tab:compute}) reports per-experiment GPU type, memory, walltime per job, and total GPU-hours, including a full-research-effort estimate that covers preliminary and abandoned runs.
    \item[] Guidelines:
    \begin{itemize}
        \item The answer \answerNA{} means that the paper does not include experiments.
        \item The paper should indicate the type of compute workers CPU or GPU, internal cluster, or cloud provider, including relevant memory and storage.
        \item The paper should provide the amount of compute required for each of the individual experimental runs as well as estimate the total compute. 
        \item The paper should disclose whether the full research project required more compute than the experiments reported in the paper (e.g., preliminary or failed experiments that didn't make it into the paper). 
    \end{itemize}
    
\item {\bf Code of ethics}
    \item[] Question: Does the research conducted in the paper conform, in every respect, with the NeurIPS Code of Ethics \url{https://neurips.cc/public/EthicsGuidelines}?
    \item[] Answer: \answerYes{}
    \item[] Justification: The paper is theoretical and the supporting experiments use only public continual-learning benchmarks (Split-C10/100, ImNet-R/A, ObjectNet, CUB-200, Cars-196). No human subjects, private data, or deployment risks are involved.
    \item[] Guidelines:
    \begin{itemize}
        \item The answer \answerNA{} means that the authors have not reviewed the NeurIPS Code of Ethics.
        \item If the authors answer \answerNo, they should explain the special circumstances that require a deviation from the Code of Ethics.
        \item The authors should make sure to preserve anonymity (e.g., if there is a special consideration due to laws or regulations in their jurisdiction).
    \end{itemize}

\item {\bf Broader impacts}
    \item[] Question: Does the paper discuss both potential positive societal impacts and negative societal impacts of the work performed?
    \item[] Answer: \answerYes{}
    \item[] Justification: Appendix~\ref{app:impacts} discusses positive impacts (compute efficiency for keeping deployed models current; a single-number diagnostic that spares unnecessary sweeps), the dual-use overlap between the per-layer sensitivity $s_\ell$ used here for regularization and the same quantity used by the adversarial weight-attack literature, and the indirect privacy interaction between better continual learning and long-running personalized systems.
    \item[] Guidelines:
    \begin{itemize}
        \item The answer \answerNA{} means that there is no societal impact of the work performed.
        \item If the authors answer \answerNA{} or \answerNo, they should explain why their work has no societal impact or why the paper does not address societal impact.
        \item Examples of negative societal impacts include potential malicious or unintended uses (e.g., disinformation, generating fake profiles, surveillance), fairness considerations (e.g., deployment of technologies that could make decisions that unfairly impact specific groups), privacy considerations, and security considerations.
        \item The conference expects that many papers will be foundational research and not tied to particular applications, let alone deployments. However, if there is a direct path to any negative applications, the authors should point it out. For example, it is legitimate to point out that an improvement in the quality of generative models could be used to generate Deepfakes for disinformation. On the other hand, it is not needed to point out that a generic algorithm for optimizing neural networks could enable people to train models that generate Deepfakes faster.
        \item The authors should consider possible harms that could arise when the technology is being used as intended and functioning correctly, harms that could arise when the technology is being used as intended but gives incorrect results, and harms following from (intentional or unintentional) misuse of the technology.
        \item If there are negative societal impacts, the authors could also discuss possible mitigation strategies (e.g., gated release of models, providing defenses in addition to attacks, mechanisms for monitoring misuse, mechanisms to monitor how a system learns from feedback over time, improving the efficiency and accessibility of ML).
    \end{itemize}
    
\item {\bf Safeguards}
    \item[] Question: Does the paper describe safeguards that have been put in place for responsible release of data or models that have a high risk for misuse (e.g., pre-trained language models, image generators, or scraped datasets)?
    \item[] Answer: \answerNA{}
    \item[] Justification: The paper releases no new models or datasets; it analyzes existing continual-learning methods and uses public benchmarks only.
    \item[] Guidelines:
    \begin{itemize}
        \item The answer \answerNA{} means that the paper poses no such risks.
        \item Released models that have a high risk for misuse or dual-use should be released with necessary safeguards to allow for controlled use of the model, for example by requiring that users adhere to usage guidelines or restrictions to access the model or implementing safety filters. 
        \item Datasets that have been scraped from the Internet could pose safety risks. The authors should describe how they avoided releasing unsafe images.
        \item We recognize that providing effective safeguards is challenging, and many papers do not require this, but we encourage authors to take this into account and make a best faith effort.
    \end{itemize}

\item {\bf Licenses for existing assets}
    \item[] Question: Are the creators or original owners of assets (e.g., code, data, models), used in the paper, properly credited and are the license and terms of use explicitly mentioned and properly respected?
    \item[] Answer: \answerYes{}
    \item[] Justification: All prior methods discussed (EWC, SI, MAS, TUNA, SLCA, K-FAC) and datasets used by the companion experiments (Split-C10/100, ImNet-R/A, ObjectNet, CUB-200, Cars-196) are cited with their original publications in Section~\ref{sec:related} and the References. The companion repository documents dataset licenses for the empirical runs.
    \item[] Guidelines:
    \begin{itemize}
        \item The answer \answerNA{} means that the paper does not use existing assets.
        \item The authors should cite the original paper that produced the code package or dataset.
        \item The authors should state which version of the asset is used and, if possible, include a URL.
        \item The name of the license (e.g., CC-BY 4.0) should be included for each asset.
        \item For scraped data from a particular source (e.g., website), the copyright and terms of service of that source should be provided.
        \item If assets are released, the license, copyright information, and terms of use in the package should be provided. For popular datasets, \url{paperswithcode.com/datasets} has curated licenses for some datasets. Their licensing guide can help determine the license of a dataset.
        \item For existing datasets that are re-packaged, both the original license and the license of the derived asset (if it has changed) should be provided.
        \item If this information is not available online, the authors are encouraged to reach out to the asset's creators.
    \end{itemize}

\item {\bf New assets}
    \item[] Question: Are new assets introduced in the paper well documented and is the documentation provided alongside the assets?
    \item[] Answer: \answerNA{}
    \item[] Justification: The paper does not release new datasets, models, or trained checkpoints with this submission. Companion experimental code will be released separately with the camera-ready version.
    \item[] Guidelines:
    \begin{itemize}
        \item The answer \answerNA{} means that the paper does not release new assets.
        \item Researchers should communicate the details of the dataset\slash code\slash model as part of their submissions via structured templates. This includes details about training, license, limitations, etc. 
        \item The paper should discuss whether and how consent was obtained from people whose asset is used.
        \item At submission time, remember to anonymize your assets (if applicable). You can either create an anonymized URL or include an anonymized zip file.
    \end{itemize}

\item {\bf Crowdsourcing and research with human subjects}
    \item[] Question: For crowdsourcing experiments and research with human subjects, does the paper include the full text of instructions given to participants and screenshots, if applicable, as well as details about compensation (if any)?
    \item[] Answer: \answerNA{}
    \item[] Justification: The paper does not involve crowdsourcing or research with human subjects.
    \item[] Guidelines:
    \begin{itemize}
        \item The answer \answerNA{} means that the paper does not involve crowdsourcing nor research with human subjects.
        \item Including this information in the supplemental material is fine, but if the main contribution of the paper involves human subjects, then as much detail as possible should be included in the main paper. 
        \item According to the NeurIPS Code of Ethics, workers involved in data collection, curation, or other labor should be paid at least the minimum wage in the country of the data collector. 
    \end{itemize}

\item {\bf Institutional review board (IRB) approvals or equivalent for research with human subjects}
    \item[] Question: Does the paper describe potential risks incurred by study participants, whether such risks were disclosed to the subjects, and whether Institutional Review Board (IRB) approvals (or an equivalent approval/review based on the requirements of your country or institution) were obtained?
    \item[] Answer: \answerNA{}
    \item[] Justification: The paper does not involve human subjects or crowdsourcing, so IRB approval is not applicable.
    \item[] Guidelines:
    \begin{itemize}
        \item The answer \answerNA{} means that the paper does not involve crowdsourcing nor research with human subjects.
        \item Depending on the country in which research is conducted, IRB approval (or equivalent) may be required for any human subjects research. If you obtained IRB approval, you should clearly state this in the paper. 
        \item We recognize that the procedures for this may vary significantly between institutions and locations, and we expect authors to adhere to the NeurIPS Code of Ethics and the guidelines for their institution. 
        \item For initial submissions, do not include any information that would break anonymity (if applicable), such as the institution conducting the review.
    \end{itemize}

\item {\bf Declaration of LLM usage}
    \item[] Question: Does the paper describe the usage of LLMs if it is an important, original, or non-standard component of the core methods in this research? Note that if the LLM is used only for writing, editing, or formatting purposes and does \emph{not} impact the core methodology, scientific rigor, or originality of the research, declaration is not required.
    \item[] Answer: \answerNA{}
    \item[] Justification: LLMs are not part of the methodology or any core component of this research; any use of LLM tools was limited to writing and formatting, which per NeurIPS policy does not require declaration.
    \item[] Guidelines:
    \begin{itemize}
        \item The answer \answerNA{} means that the core method development in this research does not involve LLMs as any important, original, or non-standard components.
        \item Please refer to our LLM policy in the NeurIPS handbook for what should or should not be described.
    \end{itemize}

\end{enumerate}

\end{document}